\documentclass{article}
\usepackage[utf8]{inputenc}
\usepackage[T1]{fontenc}
\usepackage{fullpage}
\usepackage{amsmath,amssymb,amsthm,mathtools,bm,mathrsfs,dsfont}
\usepackage{graphicx,booktabs,array}
\usepackage{caption,placeins}
\usepackage[dvipsnames]{xcolor}
\usepackage{xspace}
\usepackage{algorithm,algpseudocode}
\usepackage{tikz}
\usetikzlibrary{arrows.meta,positioning}
\usepackage{microtype}
\usepackage[round,authoryear]{natbib}
\usepackage{xurl}
\usepackage[colorlinks,linkcolor=magenta,citecolor=blue,urlcolor=blue]{hyperref}
\usepackage{cleveref}
\hypersetup{pdftitle={GAV: A scientific experimentation environment for PDE agents},pdfauthor={Muheng Li and Wenlong Mou}}
\newcommand{\E}{\mathbb{E}}

\DeclareMathOperator*{\argmin}{arg\,min}

\makeatletter
\newcommand{\numberthis}{%
  \refstepcounter{equation}%
  \global\let\paper@numberthis@creflabel\cref@currentlabel
  \tag{\theequation}%
  \g@addto@macro\df@tag{%
    \let\cref@currentlabel\paper@numberthis@creflabel}}
\makeatother
\renewcommand{\leq}{\leqslant}
\renewcommand{\geq}{\geqslant}

\theoremstyle{plain}
\newtheorem{theorem}{Theorem}
\newtheorem{proposition}{Proposition}

\newtheorem{lemma}{Lemma}

\theoremstyle{definition}

\newtheorem{assumption}{Assumption}

\theoremstyle{remark}

\theoremstyle{plain}
\crefname{informaltheorem}{theorem}{theorems}
\Crefname{informaltheorem}{Theorem}{Theorems}

\begin{document}

\begin{center}
{\bfseries\LARGE GAV: A scientific experimentation\\environment for PDE agents\par}
\vspace{0.2in}
{\large Muheng Li \qquad Wenlong Mou\par}
\vspace{0.2in}
{\large Department of Statistical Sciences, University of Toronto\par}
{\large Vector Institute\par}
\end{center}

\begin{abstract}
Reliable feedback is essential for coding agents to develop programs for partial differential equations (PDEs). Yet evaluating these programs on new PDE problems often requires expensive reference computations, while fixed datasets restrict the range of experiments. We introduce Generator--Agent--Verifier (GAV), a scientific experimentation environment that makes new PDE problems verifiable by construction. The Generator learns a distribution of continuous reference solutions using data from numerical simulations or experimental measurements to represent a target physical regime. Using the method of manufactured solutions, it works backward from the generated reference solutions to construct matching PDE problems without a separate numerical solve. The Agent designs experiments, investigates failures, and iteratively improves its program. The Verifier evaluates candidate programs against the reference solutions and provides numerical feedback on accuracy and computational cost. This feedback can also guide the generation of new problems that probe the current program's weaknesses. We provide a theoretical analysis of evaluation reliability and guided development. Experiments on Navier--Stokes and Vlasov--Poisson-type equations across different data sources and regimes demonstrate the effectiveness of GAV. In particular, on the numerical Navier--Stokes task, GAV reduces the mean test error of three existing PDE agents by more than an order of magnitude relative to their original setups. Code is available at \url{https://github.com/PotatoJnny/PDE-GAV}.

\end{abstract}

\section{Introduction}
\label{sec:introduction}

Large language model agents increasingly contribute to scientific research by writing, executing, and revising programs.
Recent progress spans mathematical discovery and algorithm design~\citep{romeraparedes2024funsearch,novikov2025alphaevolve}, as well as machine-learning experimentation and scientific software development~\citep{yamada2025aiscientistv2,aygun2026era}.
Yet reliably completing multi-step scientific tasks remains challenging~\citep{bragg2026astabench}.
By testing candidate programs, an agent can identify failures, compare alternatives, and assess whether a revision improves accuracy or efficiency~\citep{aygun2026era,li2025codepde}.
Repeated cycles of proposing, evaluating, and revising programs also offer a way to turn additional test-time computation into better programs~\citep{novikov2025alphaevolve,ye2026simpletes}.
Realizing this potential requires a scientific experimentation environment in which agents can design and run informative experiments and obtain reliable feedback to guide their next steps.

We focus on developing programs for partial differential equation (PDE) tasks, including numerical solution and reduced models for predicting specified macroscopic quantities. In these tasks, an agent not only needs to produce a program that executes successfully, but also to ensure that the program produces accurate results. Reliable numerical feedback is therefore needed to measure progress~\citep{dong2026autopde}.
For large-scale computational problems arising from real-world applications, an effective experimentation environment for PDE agents must address the following needs:
\begin{itemize}
    \item \textbf{Asymmetric computational cost:} Reference-based evaluation compares a candidate program's output with a reference solution, often computed by a high-resolution numerical solver~\citep{li2021fno}. Computing these reference solutions can be prohibitively expensive when an agent evaluates many new problems. The environment must therefore provide reliable feedback without expensive reference computations for every new problem.
    \item \textbf{Data-driven physical regimes:} The effectiveness of a numerical method or a reduced model depends on its operating regime. Different regimes may favor different methods. In real-world applications, the precise physical regime is often unknown a priori, and can be characterized as the distribution of plausible solutions. Therefore, the environment must provide a way to learn and represent the physical regime from simulation or measurement data, and use this knowledge to guide the agent's experimentation and program development.
    \item \textbf{Interactive feedback:} In experimentation environments, agents can iteratively propose experiments, receive feedback, and refine their programs. This interactive loop allows agents to explore the problem space, identify failure modes, and improve their methods. A fixed dataset alone is insufficient; the environment must provide new problems and reliable numerical feedback~as~experiments~require.
\end{itemize}

\begin{samepage}
Recent work has developed PDE agents that generate and refine solvers using numerical feedback~\citep{li2025codepde,fazliani2025pdesharp,dong2026autopde}.
However, their experimental environments do not yet meet all three needs above.
We introduce the Generator--Agent--Verifier (GAV) environment (Figure~\ref{fig:workbench}) to address this gap.
The Generator learns a distribution of reference solutions from simulation or measurement data, then constructs matching PDE problems using the method of manufactured solutions~\citep{salari2000manufactured}.
Each problem comes with an exact reference solution without a separate numerical solve.
The Agent can request new problems and obtain problem inputs for existing problems as its experiments require.
The Verifier evaluates the Agent's program outputs using the corresponding reference solutions under the task's performance objective and returns numerical feedback. This feedback can also guide the Generator toward new problems that probe the current program's weaknesses.
The Agent uses this feedback to revise its program and choose its next experiment.
Our work enhances existing PDE agents by providing the experimentation environment in which they develop their programs.

\end{samepage}

\begin{figure}[t]
\centering

\includegraphics[width=\linewidth]{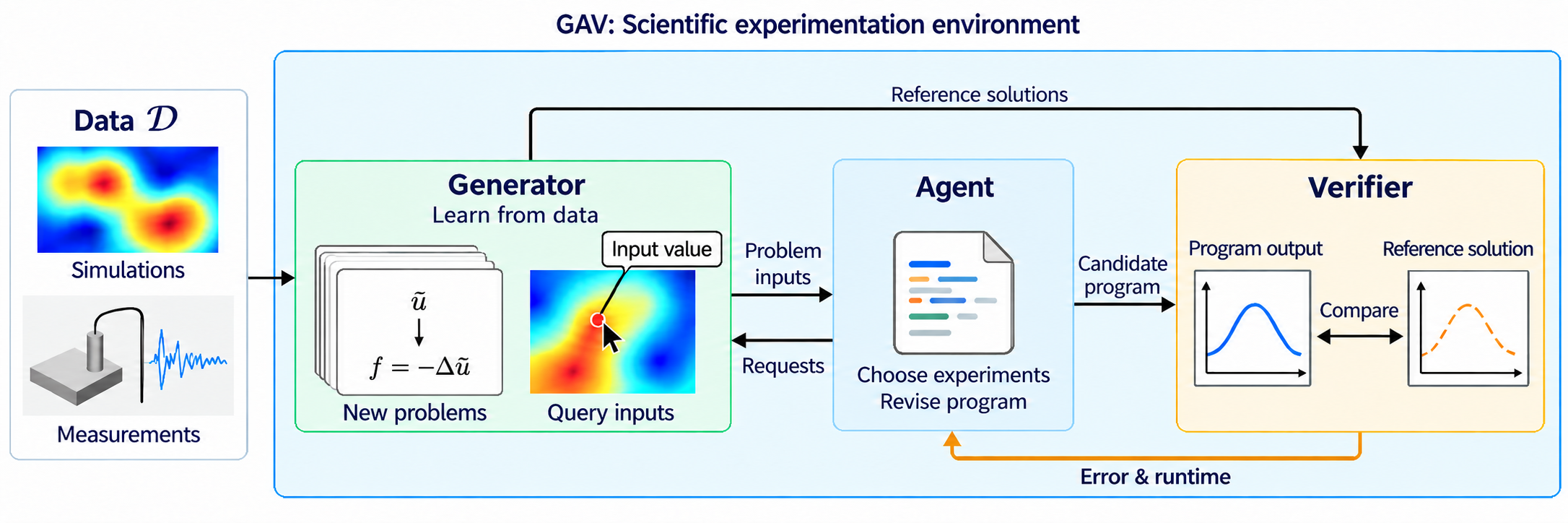}

\caption{\textbf{GAV: a scientific experimentation environment for PDE agents.} The Generator constructs PDE problems and reference solutions from scientific data. The Agent requests problems and input information, tests programs, and uses the Verifier's feedback on error and runtime to guide further experiments and revisions. The diagram illustrates a numerical solution task.}
\label{fig:workbench}

\end{figure}

Specifically, our contributions are as follows:
\begin{itemize}
    \item We construct an experimentation environment from scientific data that supports problem generation, continuous input queries, and evaluation against exact reference solutions, enabling existing PDE agents to conduct experiments throughout program development.
    \item We establish theoretical guarantees for evaluation reliability and guided development. Under stated assumptions, we bound the error in evaluating a fixed program and prove convergence of idealized guided development in a convex setting.
    \item We evaluate GAV on Navier--Stokes and Vlasov--Poisson-type equations using numerical and experimental data. On the numerical Navier--Stokes task, GAV reduces the mean test error of CodePDE, AutoNumerics, and PDE-SHARP by more than an order of magnitude relative to their original experimental setups. Ablations and development case studies examine how the environment's components affect program quality and the Agent's experimental decisions.
\end{itemize}

\section{Problem setup}
\label{sec:task}

\paragraph{Inputs.}
The environment and program-development process are specified by four inputs:

\begin{enumerate}
\item \textbf{PDE specification.}
This specifies a PDE family through its equation form, spatial domain $\Omega$ including any internal obstacles, and types of initial and boundary conditions \citep{li2025codepde,dong2026autopde}.
Time-dependent problems also specify a time interval.
Within this family, a \emph{PDE problem} is represented by $c=(\lambda,f,g)$: the physical parameters $\lambda$, source term $f$, and initial and boundary data $g$.
Its unknown solution $u$ satisfies
\begin{equation}
\mathcal N_\lambda[u] = f, \qquad \mathcal B[u] = g.
\label{eq:pde_problem}
\end{equation}
Here $\mathcal N_\lambda$ is the PDE operator with parameters $\lambda$, and $\mathcal B$ encodes the initial and boundary conditions.
For example, Navier--Stokes equations describe fluid velocity and pressure, with viscosity as a physical parameter and external forcing as a source term.
The Vlasov--Poisson (VP) equations describe a particle distribution over position and velocity coupled to an electric field; the Vlasov--Poisson--Fokker--Planck (VPFP) equations additionally include collisions.

\item \textbf{Data $\mathcal D$.}
We are given discrete solution observations $\mathcal D$ from numerical simulations or experimental measurements, together with their sampling coordinates and any available problem inputs.
When complete problem--solution pairs are available, we write $\mathcal D=\{(c_i,\widehat u_i)\}_{i=1}^{n}$, where $\widehat u_i$ contains the sampled solution values for problem $c_i$.
The data may also provide partial observations from which the Generator constructs continuous reference solutions (Appendix~\ref{app:construction}).

\item \textbf{Program-development task.}
The task specifies the program's inputs, required outputs, and performance objective.
The Agent chooses the computational method and internal representation used to produce those outputs.
We use two tasks as examples:

\emph{Numerical solution.}
A numerical solver computes an approximation to $u$ from problem inputs \citep{li2025codepde}.
For example, the task may require minimizing solution error subject to a runtime limit per problem.

\emph{Reduced models for macroscopic prediction.}
Here a reduced model is a program that predicts specified macroscopic quantities, such as density, electric field, particle flux, pressure, and heat flux in kinetic equations, from the problem inputs.
The reduction concerns the required outputs.
No particular reduced state or macroscopic evolution equation is prescribed.
The Agent aims to minimize prediction error for these quantities, subject to any runtime limit specified by the task.

\item \textbf{Agent setting.}
This specifies the Agent, its development strategy, and its development budget (time and computational resources).
We treat this setting as given \mbox{when studying the environment.}

\end{enumerate}

\paragraph{Output.}
The output of program development is an executable program for the specified task.
The Agent develops this program by using the environment to run experiments and obtain numerical feedback, aiming to meet the task requirements within its development budget.

Our goal is to build an environment that supports this process across different PDE specifications, types of data, tasks, and Agent settings.
We evaluate the environment through the performance of the final program.

\section{GAV: A scientific experimentation environment}
\label{sec:method}

GAV supplies the Agent specified in Section~\ref{sec:task} with problems and numerical feedback for developing its program.
We first describe how the Generator constructs problems with reference solutions (Section~\ref{sec:generation}) and how the Verifier evaluates candidate programs (Section~\ref{sec:selection}).
We then explain how the Agent uses these capabilities to choose further experiments and improve its program (Section~\ref{sec:interaction}), completing the interaction in Figure~\ref{fig:workbench}.

\subsection{Generating PDE problems with reference solutions}
\label{sec:generation}

The Generator aims to construct new PDE problems in the physical regime represented by $\mathcal D$.
We learn a generative model of continuous reference solutions from the available discrete observations.
Sampling from this model yields a reference solution $\widetilde u$ that can be evaluated at any point in the prescribed domain and differentiated as required by the PDE operator.
Representation, training, and sampling details are given in Appendix~\ref{app:construction}.

For each generated reference solution, we construct a corresponding PDE problem using the method of manufactured solutions (MMS)~\citep{salari2000manufactured}.
For prescribed physical parameters~$\lambda$,~we~set
\begin{equation}
f = \mathcal N_\lambda[\widetilde u], \qquad g = \mathcal B[\widetilde u].
\label{eq:manufactured_case}
\end{equation}
This defines a problem $c=(\lambda,f,g)$ for which $\widetilde u$ satisfies Equation~\eqref{eq:pde_problem}.
Each problem therefore comes with a reference solution without a separate numerical solve.
We provide this reference solution to the Verifier for evaluating program outputs.
Repeating this construction supplies new problems for further experiments.
The continuous representation of $\widetilde u$ and the construction in Equation~\eqref{eq:manufactured_case} give continuous representations of the source term $f$ and initial and boundary data $g$.
The Agent's program can therefore obtain values of $f$ and $g$ at the locations needed by its numerical method, without being restricted to the resolution of the original data in $\mathcal D$.

\subsection{Numerical evaluation and feedback}
\label{sec:selection}

The Verifier evaluates candidate programs under the task's performance objective and returns numerical feedback to the Agent.
Because the Generator constructs each problem from a continuous reference solution, the Verifier has an exact reference solution for every generated problem.
It computes the task's required outputs from this reference and compares them with the program's predictions at the evaluation points.
These outputs may be solution values, as in the NS velocity task, or derived quantities, as in the kinetic macroscopic-prediction tasks.
The program may compute additional quantities internally to produce these outputs.

During development, reference solutions remain private to the Verifier; the Agent receives measured errors and runtime.
The same evaluation procedure provides feedback for program development and supports program selection and final evaluation.

\paragraph{Reliability of evaluation.}
We formalize the Verifier's evaluation procedure and provide theoretical guarantees on the accuracy of its loss estimates.
For the paired-data setting, we assume that the problems $c_i$ in $\mathcal D$ are drawn from a target distribution $P$.
The Generator trained on $\mathcal D$ produces new problems from a learned problem distribution $\widehat P$, each paired with a reference solution.

For this analysis, we consider a fixed test set of $M$ problems $\widetilde c_1,\ldots,\widetilde c_M$, sampled independently from $\widehat P$.
For a candidate program $A$ chosen without using this test set, let $\ell_A(c)\geq0$ denote its task loss on a problem $c$, computed by comparing the program's outputs with the corresponding quantities obtained from the reference solution.
We use the empirical mean loss $\widehat\mu_M$ computed on this test set to estimate the expected loss $\mu_P$ under the target distribution $P$:
\[
 \widehat\mu_M=\frac{1}{M}\sum_{j=1}^{M}\ell_A(\widetilde c_j),
 \qquad
 \mu_P=\E_{c\sim P}[\ell_A(c)].
\]
We now present the assumptions and the resulting evaluation guarantee.

\begin{assumption}
\label{ass:evaluation}
(i) $D_{\mathrm{KL}}(P\|\widehat P)\leq\varepsilon$ for some $\varepsilon\geq0$.
(ii)~Let $\mu_{\widehat P}=\E_{c\sim\widehat P}[\ell_A(c)]\in(0,\infty)$.
The normalized loss $\ell_A/\mu_{\widehat P}$ is sub-exponential under $\widehat P$ with parameters $\nu,\kappa\geq0$, i.e.,
\begin{equation}
 \log\E_{\widehat P}\exp\!\left(s\left[\frac{\ell_A}{\mu_{\widehat P}}-1\right]\right)
 \leq \frac{\nu^2 s^2}{2},
 \qquad |s|<1/\kappa.
 \label{eq:normalized-tail}
\end{equation}
When $\kappa=0$, the condition applies to every real $s$.
\end{assumption}

\begin{proposition}
\label{prop:relative-evaluation}
Under \Cref{ass:evaluation}, for sufficiently small $\varepsilon$ and any $\delta\in(0,1)$, we have with probability at least $1-\delta$,
\begin{equation}
 \frac{|\widehat\mu_M-\mu_P|}{\mu_P}
 \leq \frac{a_\varepsilon+b_M}{1-a_\varepsilon},
 \label{eq:relative-evaluation}
\end{equation}
where
\[
 a_\varepsilon=\nu\sqrt{2\varepsilon}+\kappa\varepsilon,
 \qquad
 b_M=\nu\sqrt{\frac{2\log(2/\delta)}{M}}+
       \frac{\kappa\log(2/\delta)}{M}.
\]
Here, ``sufficiently small'' means that $a_\varepsilon<1$.
\end{proposition}
\noindent See~\Cref{app:evaluation-theory} for the proof.

\Cref{ass:evaluation}(i) requires the learned problem distribution to be close to the target distribution in KL divergence.
Under regularity assumptions, flow-matching theory bounds the forward KL divergence in terms of the vector-field approximation error~\citep[Theorem~1]{wildberger2023flowmatching}.
\Cref{ass:evaluation}(ii) requires that the loss does not vary too much across generated problems relative to its mean (a standard condition satisfied, for example, by bounded losses).

In \Cref{prop:relative-evaluation}, the bound on the relative error in estimating $\mu_P$ becomes smaller as $\varepsilon$ decreases and $M$ increases.
\Cref{prop:relative-evaluation} thus supports our Generator--Verifier design: when the learned problem distribution approximates the target distribution and enough independent problems are evaluated, the Verifier can reliably evaluate a fixed program under the target distribution.

\subsection{Agent-directed experimentation}
\label{sec:interaction}

Together, the Generator and Verifier provide an experimentation environment in which the Agent can \emph{test on new problems}, \emph{obtain problem inputs at chosen locations and resolutions}, and \emph{receive error feedback computed from exact reference solutions} to improve its program.
The Agent can use this feedback to decide which experiments to conduct next.

During development, the Agent may need to investigate specific weaknesses or check for overfitting to the learned problem distribution.
We therefore provide two optional extensions to the Generator.

Given a program, \emph{guided generation} seeks out problems on which that program is likely to make larger errors.
We fit a predictor to the Verifier's error feedback and use its gradient to guide sampling, adapting the return guidance of \citet{janner2022planning}.
Implementation details are given~in~Appendix~\ref{app:guided-generation}.

\emph{Perturbation} aims to reduce overfitting to the Generator's learned distribution.
We modify a generated reference solution while preserving the constraints required by the PDE, then use MMS to construct the corresponding problem.
Implementation details are given in Appendix~\ref{app:perturbation}.

Appendix~\ref{app:guided-case} examines problem difficulty under guided generation.
Appendix~\ref{app:perturbation-case} illustrates how perturbation feedback redirects the Agent's development of an NS solver.

\paragraph{Guided development.}
Guided generation seeks problems with large program errors while staying close to the learned problem distribution.
Consider a family of programs $A_\theta$, with parameters $\theta\in\Theta\subset\mathbb R^d$ and task loss $\ell_\theta(c)=\ell_{A_\theta}(c)$.
For a fixed finite KL radius $\rho\geq0$, this gives the~worst-case~objective
\begin{equation}
 F(\theta)=\sup_{R:\,D_{\mathrm{KL}}(R\|\widehat P)\leq\rho}\E_{c\sim R}[\ell_\theta(c)].
 \label{eq:robust-objective}
\end{equation}
For the current program $A_{\theta_t}$, replacing the KL constraint by a penalty gives the objective $\E_R\ell_{\theta_t}-\beta_t^{-1}D_{\mathrm{KL}}(R\|\widehat P)$.
For bounded losses and $\beta_t>0$, its maximizer is the exponential tilt
\begin{equation}
 \frac{dR_t}{d\widehat P}(c)
 =\frac{\exp\{\beta_t\ell_{\theta_t}(c)\}}
 {\E_{\widehat P}\exp\{\beta_t\ell_{\theta_t}\}}.
 \label{eq:exact-guidance-tilt}
\end{equation}
When $\beta_t$ is chosen so that $D_{\mathrm{KL}}(R_t\|\widehat P)=\rho$, this law also solves the constrained problem in~\eqref{eq:robust-objective}.
The reformulation provides a sampling target for guided generation~\citep{feng2025guidance}.

We analyze an idealized setting in which each guidance step returns this exact constrained maximizer, with a finite calibrated $\beta_t>0$ when $\rho>0$.
For $\rho=0$, take $R_t=\widehat P$ and $\beta_t=0$.
Let $f_t(\theta)=\E_{c\sim R_t}[\ell_\theta(c)]$ denote the expected loss on the guided problems; exact maximization gives $f_t(\theta_t)=F(\theta_t)$.
Starting from $\theta_0\in\Theta$, the Agent revises its program with a fixed stabilization parameter $\tau>0$:
\begin{equation}
 \theta_{t+1}\in\argmin_{\theta\in\Theta}
 \left\{f_t(\theta)+\frac\tau2\|\theta-\theta_t\|^2\right\}.
 \label{eq:stabilized-update}
\end{equation}

\begin{assumption}
\label{ass:convex-guidance}
(i)~The set $\Theta$ is nonempty, compact, and convex, and $F$ is convex on $\Theta$.

(ii)~The losses are uniformly bounded under $\widehat P$ and $G$-Lipschitz continuous in $\theta$, for some finite $G\geq0$, on a common set of problems of full $\widehat P$-measure.
\end{assumption}

\begin{theorem}
\label{thm:guided-convergence}
Under \Cref{ass:convex-guidance}, consider the idealized guided updates in~\eqref{eq:stabilized-update}, with $0\leq\beta_t\leq\bar\beta<\infty$ for every $t\geq0$ and $\tau>\bar\beta G^2$.
Then $F(\theta_{t+1})\leq F(\theta_t)$ for all $t\geq0$.
Let $F^\star=\min_{\theta\in\Theta}F(\theta)$ denote the minimum objective value over this program family.
For every integer $T\geq1$, the difference between $F(\theta_T)$ and $F^\star$ satisfies
\begin{equation}
 F(\theta_T)-F^\star
 \leq\frac{2\tau\,\operatorname{diam}(\Theta)^2}{T+1},
 \label{eq:convex-guided-rate}
\end{equation}
where $\operatorname{diam}(\Theta)=\sup_{\theta,\theta'\in\Theta}\|\theta-\theta'\|$.
\end{theorem}
\noindent See~\Cref{app:convergence} for the proof.

The bound on tilting strength, together with Lipschitz losses, gives $F(\theta)-f_t(\theta)\leq\bar\beta G^2\|\theta-\theta_t\|^2/2$ by \Cref{lem:guidance-model}.
The stabilization therefore controls the change in worst expected loss as the Agent improves its program on the current guided problems.

\section{Experiments}
\label{sec:experiments}

In this section, we evaluate the effectiveness of GAV as an experimentation environment for PDE agents.
We first assess the accuracy of programs developed with GAV for different PDEs and scientific tasks.
We then examine whether access to GAV improves the programs developed by existing PDE agents.
Finally, we analyze how individual components of the environment affect~program~performance.

\subsection{Experimental setup}
\label{sec:protocol}

Following the setup in Section~\ref{sec:task}, we specify the PDEs, tasks, data sources, and agent settings used in our experiments, and describe how we evaluate the resulting programs.

\paragraph{PDEs and tasks.}
We study two-dimensional incompressible Navier--Stokes (NS) equations and kinetic equations.
For NS, the task is to predict velocity at requested space--time locations from the initial velocity, boundary data, and forcing on the specified geometry (Appendix~\ref{app:ns-equations}).
The kinetic studies comprise collisionless VP in one spatial and one velocity dimension, and VPFP in two spatial and two velocity dimensions under weak and strong collisions (Appendix~\ref{app:vpfp-equations}).
The program predicts density, electric field, particle flux, central pressure, and central heat flux from the initial distribution, source, and physical parameters.
The Agent chooses the numerical method and internal representation.

\paragraph{Data sources.}
The NS studies use numerical simulations in a square domain with two fixed obstacles and particle image velocimetry (PIV) measurements of a cylinder wake~\citep{tu2014spectral}.
For collisionless VP, local electron velocity profiles from ARTEMIS observations constrain a learned continuous reference model~\citep{mcfadden2008themis,angelopoulos2023themisbesa}.
The observed profiles constrain the velocity distribution, while the continuous model supplies its spatial and temporal structure.
The VPFP studies use separate numerical datasets for weak and strong collisions.
We develop a separate program in each study.
Appendices~\ref{app:pdes} and~\ref{app:environment} describe the data and Generators.

\paragraph{Agent setting.}
For the experiments in Sections~\ref{sec:main-comparison} and~\ref{sec:environment-analysis}, two agents develop programs in parallel over multiple \emph{rounds} (Figure~\ref{fig:agent-loop}).
Both agents start each study from the same initial program.
At the start of each round, both agents receive the current program and a cumulative log of their experiments and selection results.
One explores alternative approaches, while the other improves the current approach.
Each agent conducts experiments and submits a candidate program within its allotted time.
At the end of the round, the candidates and the current program are evaluated on the same problems; the selected program becomes the starting point for the next round.
The coding model, number of rounds, and time budget are parameters of this shared development setting.
We use GPT 5.6 Sol (Sol) and GPT-6 Astra (Astra), with both agents in a run using the same model.
Table~\ref{tab:development-settings} gives the configurations used in each study.

\begin{figure}[t]
\centering
\captionsetup{skip=3pt}
\includegraphics[width=\linewidth,trim=17bp 40bp 6bp 47bp,clip]{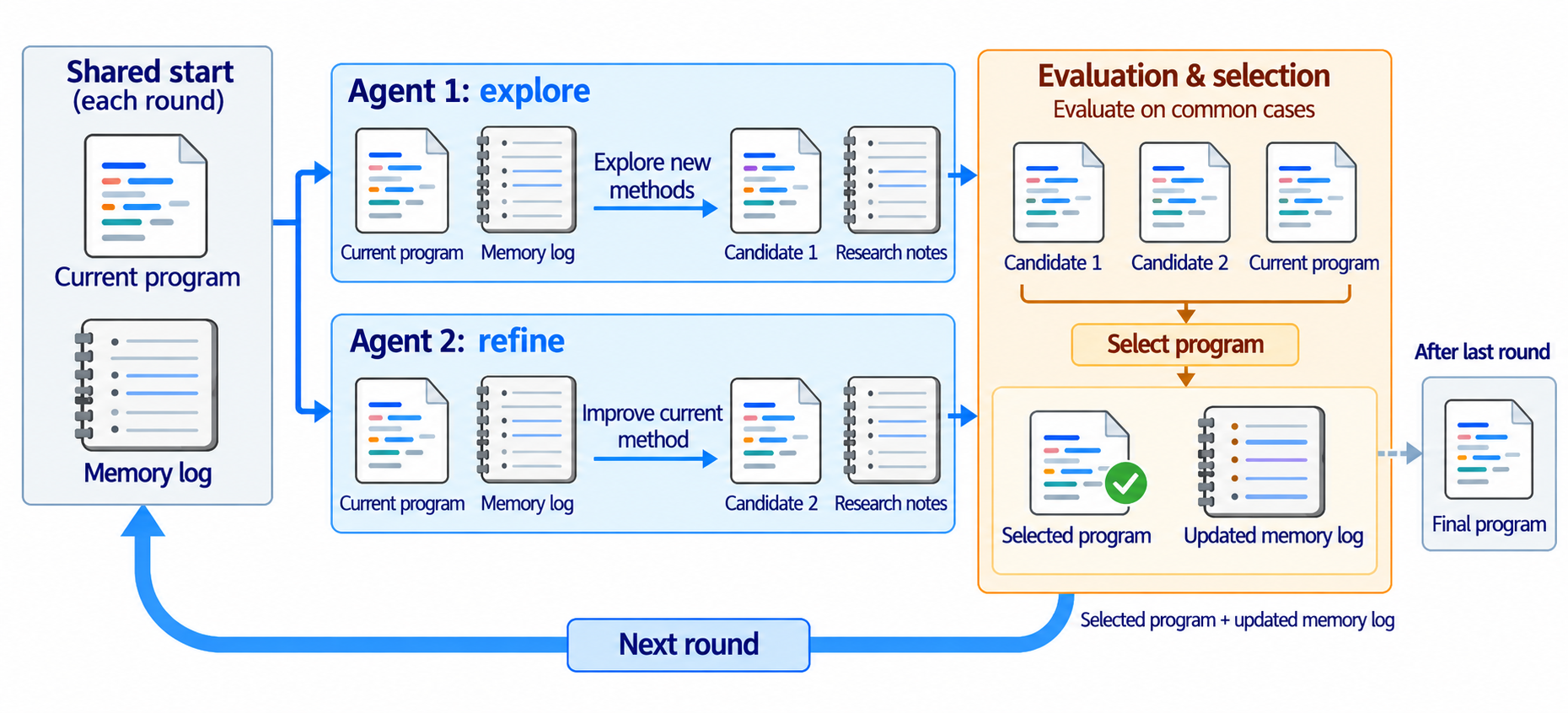}
\caption{\textbf{Multi-turn setting.} The selected program and updated memory log pass to the next round; research notes from both agents are retained regardless of the selection.}
\label{fig:agent-loop}

\end{figure}

\paragraph{Evaluation.}
For the development curves, we evaluate the programs retained after each round on a fixed test set for each task.
Within a task, the test problems and evaluation points are shared across rounds and coding models.
The final program is the one retained after the last round prescribed by the development configuration.
The Verifier compares each program's required outputs with the corresponding reference quantities, using the following task-specific errors.

For NS, we average the relative $\ell_2$ errors of the two velocity components for each problem:
\begin{equation}
\mathrm{Error}_i
=\frac{1}{2}\sum_{k=1}^{2}
\frac{\|u^{\mathrm{pred}}_{i,k}-\widetilde u_{i,k}\|_2}
     {\|\widetilde u_{i,k}\|_2}.
\label{eq:ns-velocity-error}
\end{equation}
Here $u^{\mathrm{pred}}_{i,k}$ and $\widetilde u_{i,k}$ denote the predicted and reference values of velocity component $k$ at the evaluation points for problem $i$.
Appendix~\ref{app:selection-evaluation} specifies the test sets and aggregation.

For kinetic equations, we use a common macroscopic transport score.
For each field $A\in\{\rho,\mathbf E,\mathbf j,\mathbf P,\mathbf q\}$, define
\begin{equation}
e_A=\left(
\frac{\langle\|A^{\mathrm{pred}}-\widetilde A\|^2\rangle}
     {\langle\|\widetilde A\|^2\rangle+a_A^2}
\right)^{1/2},
\qquad
\mathrm{Error}_{\mathrm{kin}}
=\frac{e_\rho+e_{\mathbf E}}{8}
 +\frac{e_{\mathbf j}+e_{\mathbf P}+e_{\mathbf q}}{4}.
\label{eq:kinetic-error}
\end{equation}
Here $\langle\cdot\rangle$ averages over the evaluation points of one problem, and $a_A>0$ is a fixed physical normalization scale calibrated before development.
We use the Euclidean norm for vectors and the Frobenius norm for the pressure tensor.
The reference fields are computed from the manufactured distribution.
We report the arithmetic mean of the per-problem scores.
Appendix~\ref{app:selection-evaluation} gives the calibration, selection, and test protocols.

\subsection{Program development across PDE tasks}
\label{sec:main-comparison}

We study program development with \emph{GAV-Full}, with guided generation and perturbation enabled.
We use the shared development setting in Section~\ref{sec:protocol}, configured for each task as in Table~\ref{tab:development-settings}.
The development curves compare retained programs on generated manufactured problems; separate tests examine accuracy beyond the learned Generators.

\paragraph{NS solvers.}
GAV-Full enables agents to develop accurate solvers for both numerical and PIV-derived NS problems.
On numerical NS, the final program achieves a mean test error of $0.1046\%$, more than an order of magnitude lower than COMSOL's $1.6352\%$ (Table~\ref{tab:ns-solver-comparison}).
To check that this accuracy does not come from overfitting to the learned Generator, we also test the solver on problems constructed directly from numerical solutions not used during development, without sampling the Generator.
Its mean error there is $0.1256\%$, close to the $0.1046\%$ on generated problems (Appendix~\ref{app:interpolation-test}).
Figure~\ref{fig:development-rounds} shows the decline in test error over development for both data sources; Appendix~\ref{app:program-evolution} details the program revisions.

\begin{table}[htbp]
\centering
\caption{\textbf{Numerical NS solver comparison.} Mean relative velocity error (\%) on the same test set; lower is better. Appendix~\ref{app:numerical-baselines} describes the baselines.}
\label{tab:ns-solver-comparison}
\begin{tabular}{@{}lr@{}}
\toprule
Method & Mean error (\%) \\
\midrule
COMSOL~\citep{comsol2025multiphysics} & 1.6352 \\
FEniCSx~\citep{baratta2023dolfinx} & 0.3510 \\
Nektar++~\citep{moxey2020nektar} & 1.9042 \\
GAV-Full (ours) & \textbf{0.1046} \\
\bottomrule
\end{tabular}
\end{table}

\begin{figure}[htbp]
\centering
\includegraphics[width=0.86\linewidth]{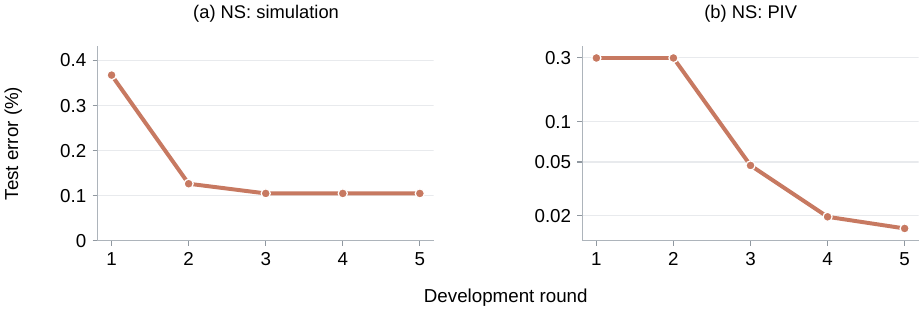}
\caption{\textbf{NS program development with GAV-Full.} Mean relative velocity error (\%) on the numerical and PIV-derived NS test sets over five development rounds. Lower is better; the PIV panel uses a log scale.}
\label{fig:development-rounds}
\end{figure}

\paragraph{Vlasov--Poisson reduced-order models.}
Figure~\ref{fig:kinetic-development} and Table~\ref{tab:kinetic-results} show lower final-round mean error than first-round mean error in all six runs, with reductions from $25.4\%$ to $90.1\%$.
Strong-collision Astra's mean error increases in round 5 and decreases again in round 6; weak-collision Astra reaches its lowest test mean in round 4.
Collisionless VP with Astra has nearly constant mean error after round 3.
Appendix~\ref{app:kinetic-program-evolution} illustrates how agents use feedback to revise their numerical methods and input queries in the Astra runs.

Independent tests constructed without the learned Generators use interpolated VPFP references and VP problems from previously unused ARTEMIS profiles (Appendix~\ref{app:kinetic-independent}).
All six final programs improve over their initial programs, and the R1-to-R6 mean-error reduction transfers in five runs; ARTEMIS Astra's independent mean increases by a factor of $5.69$.

\begin{figure}[htbp]
\centering
\includegraphics[width=\linewidth]{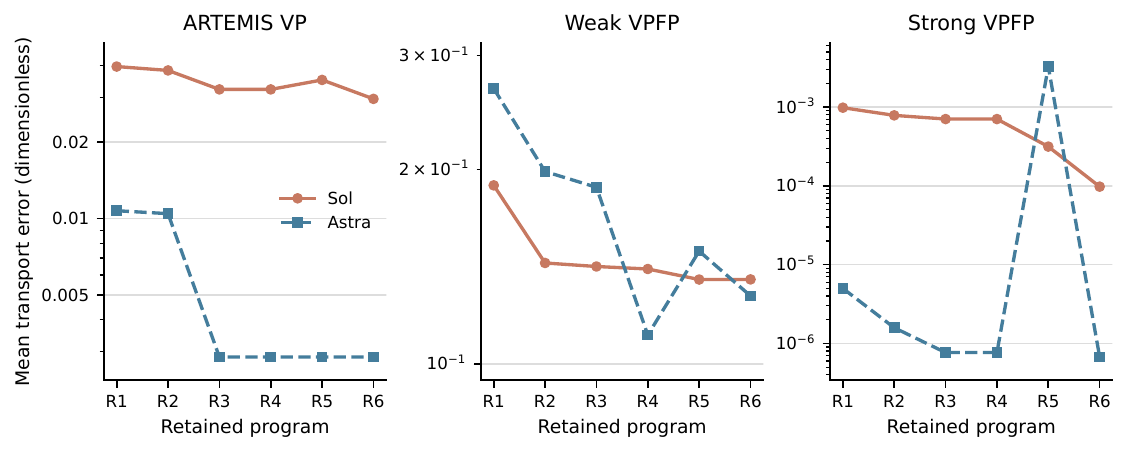}
\caption{\textbf{Kinetic program development with GAV-Full.} Mean transport score from Equation~\eqref{eq:kinetic-error} for programs retained after rounds 1--6. Each task uses the same test set and evaluation queries across both models and all rounds. Scores are dimensionless; all three panels use logarithmic axes and lower is better. R6 is the final program. Strong-collision Astra retains its round-3 program in round 4.}
\label{fig:kinetic-development}
\end{figure}

\begin{table}[htbp]
\centering
\caption{\textbf{First-round and final kinetic programs.} Mean dimensionless transport score on each task's test set, shared by Sol and Astra. Reduction is $100(1-\mathrm{error}_{\mathrm{R6}}/\mathrm{error}_{\mathrm{R1}})$, relative to the first retained program R1. The final program is R6 in all six runs.}
\label{tab:kinetic-results}
\begin{tabular}{@{}llrrr@{}}
\toprule
Task & Model & R1 error & R6 error & Reduction (\%) \\
\midrule
ARTEMIS VP & Sol & 0.03961 & 0.02956 & 25.4 \\
ARTEMIS VP & Astra & 0.01074 & 0.002854 & 73.4 \\
\addlinespace
Weak-collision VPFP & Sol & 0.189 & 0.1352 & 28.5 \\
Weak-collision VPFP & Astra & 0.2668 & 0.1274 & 52.2 \\
\addlinespace
Strong-collision VPFP & Sol & $9.831\times10^{-4}$ & $9.781\times10^{-5}$ & 90.1 \\
Strong-collision VPFP & Astra & $4.999\times10^{-6}$ & $6.719\times10^{-7}$ & 86.6 \\
\bottomrule
\end{tabular}
\end{table}

\paragraph{Solution fields.}
We visualize predicted fields and their errors on fixed examples in Figure~\ref{fig:piv-solution-error} and Appendix~\ref{app:solution-visualizations}.
Figure~\ref{fig:piv-solution-error} compares the first-round (R1) and final (R5) programs on one predetermined PIV-derived manufactured problem.
The final program substantially reduces the spatial velocity errors at the five displayed times, with the full-window vector relative $\ell_2$ error decreasing from $0.0154\%$ to $0.00172\%$.
The VPFP examples show macroscopic fields and errors on fixed generated problems.

\begin{figure}[htbp]
\centering
\captionsetup{skip=3pt}
\includegraphics[width=0.90\linewidth,trim=0bp 2bp 0bp 7bp,clip]{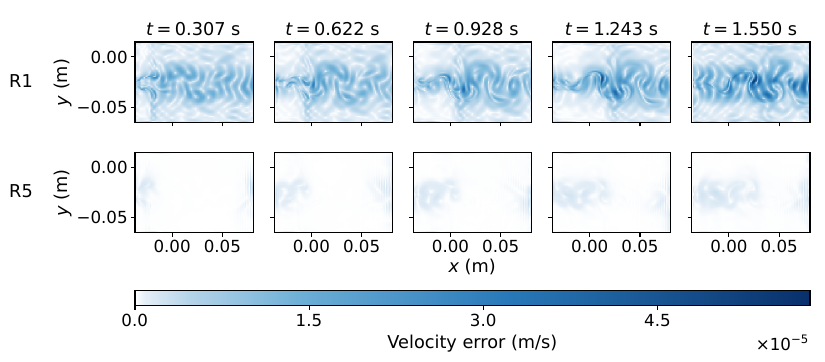}
\caption{\textbf{Velocity error on a PIV-derived manufactured problem.} Rounds 1 (top) and 5 (bottom) use the same generated reference at five times. Color shows $\|\mathbf u^{\mathrm{pred}}-\widetilde{\mathbf u}\|_2$ (m/s) on one full-window linear scale: white at zero, darker blue for larger errors.}
\label{fig:piv-solution-error}

\end{figure}
\FloatBarrier

\subsection{GAV with existing PDE agents}
\label{sec:pde-agent-workflows}

GAV is an experimentation environment designed to support PDE agents with different models and development strategies.
We next evaluate whether GAV improves the performance of \mbox{existing} PDE agents.
We adapt CodePDE~\citep{li2025codepde}, AutoNumerics~\citep{du2026autonumerics}, and PDE-SHARP~\citep{fazliani2025pdesharp} to the NS task with numerical simulation data and follow each method's original development strategy.
Within each method, we compare its original setup (Original) with two configurations of our environment.
Original uses a fixed collection of problems with inputs sampled on prescribed grids.
GAV additionally allows agents to request new problems and query inputs at chosen coordinates.
GAV-Full also enables guided generation and perturbation.
We evaluate all final programs on the same held-out test set.
Appendix~\ref{app:pde-agent-adaptations} describes how each method uses the environment.

GAV and GAV-Full achieve lower test error than Original for all three PDE agents (Figure~\ref{fig:pde-agent-gains}).
GAV reduces mean test error by more than an order of magnitude.
Both variants also use less total agent active time than Original (Appendix~\ref{app:pde-agent-adaptations}, Table~\ref{tab:pde-agent-development-time}).

\begin{figure}[htb]
\centering
\captionsetup{skip=3pt}
\includegraphics[width=0.90\linewidth,trim=0bp 2bp 0bp 1.5bp,clip]{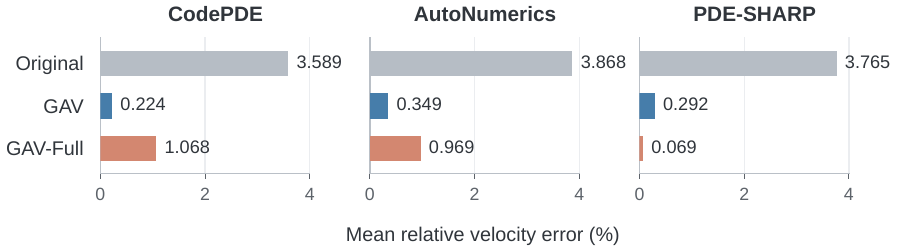}
\caption{Mean relative velocity error (\%) on the numerical NS test set; lower is better.}
\label{fig:pde-agent-gains}

\end{figure}

\subsection{Ablation studies}
\label{sec:environment-analysis}

Three key features support experimentation in GAV: learning from data to represent the target physical regime, MMS construction to provide exact references without a separate numerical solve, and continuous input queries at locations and resolutions chosen by the Agent.
We examine how each feature contributes to program development.

\paragraph{Learning from data.}
We replace the learned reference distribution with a Gaussian-process prior while keeping MMS and continuous input queries.
This tests the role of learning the reference distribution from the target scientific data.

\paragraph{MMS construction.}
MMS constructs problem inputs from an exact reference solution, avoiding a separate numerical solve.
We keep the learned problem distribution unchanged and numerically solve each generated problem to obtain an approximate reference for development feedback.
This requires additional computation time and introduces numerical error.

\paragraph{Continuous input queries.}
Without continuous input queries, agents receive initial velocity and forcing values only on a fixed $64\times64$ spatial grid during development.
They cannot directly query the input fields at locations or spatial resolutions of their choice, but may interpolate the provided grid values.
We keep the learned Generator and MMS.

We compare GAV with the three ablation variants on the NS task using the multi-turn setting in Section~\ref{sec:protocol}.
All runs start from the same initial program and use the shared development setting with the ablation configuration in Table~\ref{tab:development-settings}.
Table~\ref{tab:environment-analysis} reports errors on the same numerical NS test set used in Section~\ref{sec:pde-agent-workflows}.
Appendix~\ref{app:environment-comparisons} gives the implementation details.

\begin{table}[htbp]
\centering
\captionsetup{position=top,skip=3pt}
\caption{Mean relative velocity error (\%) after three rounds on the same NS test set. \mbox{Error / GAV} is the ratio to GAV; lower is better.}
\label{tab:environment-analysis}

\begin{tabular}{@{}lrr@{}}
\toprule
Environment & Mean error (\%) & Error / GAV \\
\midrule
GAV & \textbf{0.441} & 1.0$\times$ \\
w/o learning from data & 1.468 & 3.3$\times$ \\
w/o MMS construction & 0.988 & 2.2$\times$ \\
w/o continuous input queries & 5.211 & 11.8$\times$ \\
\bottomrule
\end{tabular}

\end{table}

Compared with GAV, the conditions without learning from data, MMS construction, or continuous input queries have $3.3$, $2.2$, and $11.8$ times the mean error, respectively.
These results support GAV's core combination of learning reference solutions from data, constructing matching problems through MMS, and providing continuous input queries.

\section{Related work}
\label{sec:related}

\paragraph{PDE agents.}
Recent work uses language model agents to automate PDE solver development through numerical reasoning, code generation, and iterative refinement~\citep{li2025codepde,fazliani2025pdesharp,dong2026autopde,du2026autonumerics}.
Related systems automate simulation workflows using established scientific software, including OpenFOAM and finite-element tools~\citep{chen2024metaopenfoam,yue2025foamagent,deotale2026allfem}.
Within solver development, \citet{li2025codepde} study inference-time search and refinement, while \citet{fazliani2025pdesharp} use mathematical analysis and solver synthesis to organize candidate generation and evaluation.
\citet{dong2026autopde} explicitly represent the solver strategy and revise it using numerical evidence.
Our work enhances the capabilities of existing PDE agents through the design of an experimentation environment.

\paragraph{Environments for scientific agents.}
Recent work has developed environments and benchmarks to support and evaluate agents on scientific tasks~\citep{huang2024mlagentbench,chen2025scienceagentbench,nathani2025mlgym,bragg2026astabench}.
Within this direction, \citet{gandhi2025boxinggym} and \citet{kabra2026llmautoscilab} study how agents choose experiments and use the resulting observations to revise scientific models.
For scientific code generation, \citet{zeng2026lcp} connect testing to program improvement by jointly refining test cases and candidate programs.
GAV advances this direction by turning scientific data into a continuing source of PDE experiments with reliable numerical feedback, without running a reference solver for every new problem.

\section{Discussion}
\label{sec:discussion}

We introduced GAV, a scientific experimentation environment for PDE agents.
It turns scientific data into new PDE problems with reliable numerical feedback, supporting agents as they test ideas, investigate failures, and refine their programs.
We establish theoretical guarantees for evaluation and guided development under stated conditions, and demonstrate substantial gains for existing PDE agents.
These results highlight environment design as a way to enhance agents' ability to solve~scientific~problems.

Future work will extend this approach beyond PDEs by coupling data-informed problem generators with task-specific verifiers, enabling agents to test ideas and improve solutions across scientific domains.

\section*{AI use statement}
As this work studies AI agents for PDE tasks, we use \texttt{gpt-5.6-sol} and \texttt{gpt-6-astra} as agent models in the experiments specified in Appendix~\ref{app:agent-settings}.
We also used AI tools to polish the manuscript and prepare scientific figures.

\bibliography{references}

\begin{thebibliography}{35}
\providecommand{\natexlab}[1]{#1}
\providecommand{\url}[1]{\texttt{#1}}
\expandafter\ifx\csname urlstyle\endcsname\relax
  \providecommand{\doi}[1]{doi: #1}\else
  \providecommand{\doi}{doi: \begingroup \urlstyle{rm}\Url}\fi

\bibitem[Angelopoulos et~al.(2023)Angelopoulos, Carlson, and
  McFadden]{angelopoulos2023themisbesa}
Vassilis Angelopoulos, Charles~W. Carlson, and James~P. McFadden.
\newblock {THEMIS-B}: {ESA} electron/ion energy fluxes and moments, 2023.
\newblock URL
  \url{https://spase-metadata.org/NASA/NumericalData/THEMIS/B/ESA/PT3S}.
\newblock Data set; metadata accessed 2026-10-19.

\bibitem[Ayg{\"u}n et~al.(2026)Ayg{\"u}n, Belyaeva, Comanici, Coram, Cui,
  Garrison, Johnston, Kast, McLean, Norgaard, Shamsi, Smalling, Thompson,
  Venugopalan, Williams, He, Martinson, Plomecka, Wei, Zhou, Zhu, Abraham,
  Brand, Bulanova, Cardille, Co, Ellsworth, Joseph, Kane, Krueger, Kartiwa,
  Liebling, Lueckmann, Raccuglia, Wang, Chou, Manyika, Matias, Platt, Dorfman,
  Mourad, and Brenner]{aygun2026era}
Eser Ayg{\"u}n, Anastasiya Belyaeva, Gheorghe Comanici, Marc Coram, Hao Cui,
  Jake Garrison, Renee Johnston, Anton Kast, Cory~Y. McLean, Peter Norgaard,
  Zahra Shamsi, David Smalling, James Thompson, Subhashini Venugopalan,
  Brian~P. Williams, Chujun He, Sarah Martinson, Martyna Plomecka, Lai Wei,
  Yuchen Zhou, Qian-Ze Zhu, Matthew Abraham, Erica Brand, Anna Bulanova,
  Jeffrey~A. Cardille, Chris Co, Scott Ellsworth, Grace Joseph, Malcolm Kane,
  Ryan Krueger, Johan Kartiwa, Dan Liebling, Jan-Matthis Lueckmann, Paul
  Raccuglia, Xuefei~Julie Wang, Katherine Chou, James Manyika, Yossi Matias,
  John~C. Platt, Lizzie Dorfman, Shibl Mourad, and Michael~P. Brenner.
\newblock An {AI} system to help scientists write expert-level empirical
  software.
\newblock \emph{Nature}, 654:\penalty0 909--916, 2026.
\newblock \doi{10.1038/s41586-026-10658-6}.
\newblock URL \url{https://www.nature.com/articles/s41586-026-10658-6}.

\bibitem[Baldasso et~al.(2022)Baldasso, Pereira, and Reis]{baldasso2022large}
Rangel Baldasso, Alan Pereira, and Guilherme Reis.
\newblock Large deviations for interacting diffusions with path-dependent
  {McKean--Vlasov} limit.
\newblock \emph{The Annals of Applied Probability}, 32\penalty0 (1):\penalty0
  665--695, 2022.
\newblock \doi{10.1214/21-AAP1692}.
\newblock URL \url{https://arxiv.org/abs/1912.00440v2}.

\bibitem[Baratta et~al.(2023)Baratta, Dean, Dokken, Habera, Hale, Richardson,
  Rognes, Scroggs, Sime, and Wells]{baratta2023dolfinx}
Igor~A. Baratta, Joseph~P. Dean, J{\o}rgen~S. Dokken, Michal Habera, Jack~S.
  Hale, Chris~N. Richardson, Marie~E. Rognes, Matthew~W. Scroggs, Nathan Sime,
  and Garth~N. Wells.
\newblock {DOLFINx}: The next generation {FEniCS} problem solving environment,
  2023.
\newblock URL \url{https://doi.org/10.5281/zenodo.10447666}.

\bibitem[Bragg et~al.(2026)Bragg, D'Arcy, Balepur, Bareket, Dalvi, Feldman,
  Haddad, Hwang, Jansen, Kishore, Majumder, Naik, Rahamimov, Richardson, Singh,
  Surana, Tiktinsky, Vasu, Wiener, Anastasiades, Candra, Dunkelberger, Emery,
  Evans, Hamada, Huff, Kinney, Latzke, Lochner, Lozano-Aguilera, Nguyen, Rao,
  Tanaka, Vlahos, Clark, Downey, Goldberg, Sabharwal, and
  Weld]{bragg2026astabench}
Jonathan Bragg, Mike D'Arcy, Nishant Balepur, Dan Bareket, Bhavana Dalvi,
  Sergey Feldman, Dany Haddad, Jena~D. Hwang, Peter Jansen, Varsha Kishore,
  Bodhisattwa~Prasad Majumder, Aakanksha Naik, Sigal Rahamimov, Kyle
  Richardson, Amanpreet Singh, Harshit Surana, Aryeh Tiktinsky, Rosni Vasu, Guy
  Wiener, Chloe Anastasiades, Stefan Candra, Jason Dunkelberger, Dan Emery, Rob
  Evans, Malachi Hamada, Regan Huff, Rodney Kinney, Matt Latzke, Jaron Lochner,
  Ruben Lozano-Aguilera, Cecile Nguyen, Smita Rao, Amber Tanaka, Brooke Vlahos,
  Peter Clark, Doug Downey, Yoav Goldberg, Ashish Sabharwal, and Daniel~S.
  Weld.
\newblock {AstaBench}: Rigorous benchmarking of {AI} agents with a scientific
  research suite.
\newblock In \emph{International Conference on Learning Representations}, 2026.
\newblock URL \url{https://arxiv.org/abs/2510.21652v2}.

\bibitem[Chen et~al.(2024)Chen, Zhu, Zhou, and Ren]{chen2024metaopenfoam}
Yuxuan Chen, Xu~Zhu, Hua Zhou, and Zhuyin Ren.
\newblock {MetaOpenFOAM}: An {LLM}-based multi-agent framework for {CFD}.
\newblock arXiv preprint arXiv:2407.21320, 2024.
\newblock URL \url{https://arxiv.org/abs/2407.21320v1}.

\bibitem[Chen et~al.(2025)Chen, Chen, Ning, Zhang, Wang, Yu, Li, Liao, Wei, Lu,
  Dey, Xue, Baker, Burns, Adu-Ampratwum, Huang, Ning, Gao, Su, and
  Sun]{chen2025scienceagentbench}
Ziru Chen, Shijie Chen, Yuting Ning, Qianheng Zhang, Boshi Wang, Botao Yu,
  Yifei Li, Zeyi Liao, Chen Wei, Zitong Lu, Vishal Dey, Mingyi Xue, Frazier~N.
  Baker, Benjamin Burns, Daniel Adu-Ampratwum, Xuhui Huang, Xia Ning, Song Gao,
  Yu~Su, and Huan Sun.
\newblock {ScienceAgentBench}: Toward rigorous assessment of language agents
  for data-driven scientific discovery.
\newblock In \emph{International Conference on Learning Representations}, 2025.
\newblock URL
  \url{https://proceedings.iclr.cc/paper_files/paper/2025/hash/f12b4df26344f3be803c06b555252efe-Abstract-Conference.html}.

\bibitem[{COMSOL AB}(2025)]{comsol2025multiphysics}
{COMSOL AB}.
\newblock {COMSOL Multiphysics} v. 6.4.
\newblock Stockholm, Sweden, 2025.
\newblock URL \url{https://www.comsol.com}.

\bibitem[Deotale et~al.(2026)Deotale, Srinivasan, Golestanian, Tian, Zhang,
  Vlachos, and Gomez]{deotale2026allfem}
Rushikesh Deotale, Adithya Srinivasan, Mahmoud Golestanian, Yuan Tian, Tianyi
  Zhang, Pavlos Vlachos, and Hector Gomez.
\newblock {ALL-FEM}: Agentic large language models fine-tuned for finite
  element methods.
\newblock \emph{Computer Methods in Applied Mechanics and Engineering},
  457:\penalty0 118985, 2026.
\newblock \doi{10.1016/j.cma.2026.118985}.
\newblock URL \url{https://doi.org/10.1016/j.cma.2026.118985}.

\bibitem[Dong et~al.(2026)Dong, Zhang, Wang, Hao, Zhuang, Liu, Wang, Liu, and
  Jin]{dong2026autopde}
Huanshuo Dong, Keyao Zhang, Hong Wang, Zhezheng Hao, Zhiwei Zhuang, Ziyan Liu,
  Jiacong Wang, Gengyuan Liu, and Xin Jin.
\newblock {AutoPDE}: Reliable agentic {PDE} solving via explicitly represented
  solver strategies.
\newblock arXiv preprint arXiv:2606.10752, 2026.
\newblock URL \url{https://arxiv.org/abs/2606.10752v1}.

\bibitem[Du et~al.(2026)Du, Sun, and Yang]{du2026autonumerics}
Jianda Du, Youran Sun, and Haizhao Yang.
\newblock {AutoNumerics}: An autonomous, {PDE}-agnostic multi-agent pipeline
  for scientific computing.
\newblock arXiv preprint arXiv:2602.17607, 2026.
\newblock URL \url{https://arxiv.org/abs/2602.17607v1}.

\bibitem[Fazliani and Udell(2025)]{fazliani2025pdesharp}
Shaghayegh Fazliani and Madeleine Udell.
\newblock {PDE-SHARP}: {PDE} solver hybrids through analysis and refinement
  passes.
\newblock arXiv preprint arXiv:2511.00183, 2025.
\newblock URL \url{https://arxiv.org/abs/2511.00183v2}.

\bibitem[Feng et~al.(2025)Feng, Yu, Deng, Hu, and Wu]{feng2025guidance}
Ruiqi Feng, Chenglei Yu, Wenhao Deng, Peiyan Hu, and Tailin Wu.
\newblock On the guidance of flow matching.
\newblock In \emph{Proceedings of the 42nd International Conference on Machine
  Learning}, volume 267 of \emph{Proceedings of Machine Learning Research},
  pages 16993--17029. PMLR, 2025.
\newblock URL \url{https://proceedings.mlr.press/v267/feng25s.html}.

\bibitem[Gandhi et~al.(2025)Gandhi, Li, Goodyear, Bhatia, Li, Bhaskar, Zaman,
  and Goodman]{gandhi2025boxinggym}
Kanishk Gandhi, Michael~Y. Li, Lyle Goodyear, Agam Bhatia, Louise Li, Aditi
  Bhaskar, Mohammed Zaman, and Noah~D. Goodman.
\newblock {BoxingGym}: Benchmarking progress in automated experimental design
  and model discovery.
\newblock In \emph{NeurIPS Workshop on Scaling Environments for Agents}, 2025.
\newblock URL \url{https://arxiv.org/abs/2501.01540v2}.

\bibitem[Huang et~al.(2024)Huang, Vora, Liang, and
  Leskovec]{huang2024mlagentbench}
Qian Huang, Jian Vora, Percy Liang, and Jure Leskovec.
\newblock {MLAgentBench}: Evaluating language agents on machine learning
  experimentation.
\newblock In \emph{Proceedings of the 41st International Conference on Machine
  Learning}, volume 235, pages 20271--20309. PMLR, 2024.
\newblock URL \url{https://proceedings.mlr.press/v235/huang24y.html}.

\bibitem[Janner et~al.(2022)Janner, Du, Tenenbaum, and
  Levine]{janner2022planning}
Michael Janner, Yilun Du, Joshua Tenenbaum, and Sergey Levine.
\newblock Planning with diffusion for flexible behavior synthesis.
\newblock In \emph{Proceedings of the 39th International Conference on Machine
  Learning}, volume 162 of \emph{Proceedings of Machine Learning Research},
  pages 9902--9915. PMLR, 2022.
\newblock URL \url{https://proceedings.mlr.press/v162/janner22a.html}.

\bibitem[Kabra et~al.(2026)Kabra, Abhyankar, Desai, Iyer, and
  Reddy]{kabra2026llmautoscilab}
Sanchit Kabra, Nikhil Abhyankar, Saaketh Desai, Prasad~P. Iyer, and Chandan~K.
  Reddy.
\newblock {LLM-AutoSciLab}: Closed-loop scientific discovery via active
  experimentation with {LLMs}.
\newblock arXiv preprint arXiv:2605.24043, 2026.
\newblock URL \url{https://arxiv.org/abs/2605.24043v1}.

\bibitem[Karras et~al.(2022)Karras, Aittala, Aila, and
  Laine]{karras2022elucidating}
Tero Karras, Miika Aittala, Timo Aila, and Samuli Laine.
\newblock Elucidating the design space of diffusion-based generative models.
\newblock \emph{Advances in neural information processing systems},
  35:\penalty0 26565--26577, 2022.

\bibitem[Li et~al.(2025)Li, Marwah, Shen, Sun, Risteski, Yang, and
  Talwalkar]{li2025codepde}
Shanda Li, Tanya Marwah, Junhong Shen, Weiwei Sun, Andrej Risteski, Yiming
  Yang, and Ameet Talwalkar.
\newblock Codepde: An inference framework for llm-driven pde solver generation.
\newblock \emph{arXiv preprint arXiv:2505.08783}, 2025.

\bibitem[Li et~al.(2021)Li, Kovachki, Azizzadenesheli, Liu, Bhattacharya,
  Stuart, and Anandkumar]{li2021fno}
Zongyi Li, Nikola Kovachki, Kamyar Azizzadenesheli, Burigede Liu, Kaushik
  Bhattacharya, Andrew Stuart, and Anima Anandkumar.
\newblock Fourier neural operator for parametric partial differential
  equations.
\newblock In \emph{International Conference on Learning Representations}, 2021.
\newblock URL \url{https://openreview.net/forum?id=c8P9NQVtmnO}.

\bibitem[Liu et~al.(2023)Liu, Gong, and Liu]{liu2023rectified}
Xingchao Liu, Chengyue Gong, and Qiang Liu.
\newblock Flow straight and fast: Learning to generate and transfer data with
  rectified flow.
\newblock In \emph{International Conference on Learning Representations}, 2023.
\newblock URL \url{https://openreview.net/forum?id=XVjTT1nw5z}.

\bibitem[McFadden et~al.(2008)McFadden, Carlson, Larson, Ludlam, Abiad,
  Elliott, Turin, Marckwordt, and Angelopoulos]{mcfadden2008themis}
J.~P. McFadden, C.~W. Carlson, D.~Larson, M.~Ludlam, R.~Abiad, B.~Elliott,
  P.~Turin, M.~Marckwordt, and V.~Angelopoulos.
\newblock The {THEMIS ESA} plasma instrument and in-flight calibration.
\newblock \emph{Space Science Reviews}, 141\penalty0 (1--4):\penalty0 277--302,
  2008.
\newblock \doi{10.1007/s11214-008-9440-2}.
\newblock URL \url{https://doi.org/10.1007/s11214-008-9440-2}.

\bibitem[Moxey et~al.(2020)Moxey, Cantwell, Bao, Cassinelli, Castiglioni, Chun,
  Juda, Kazemi, Lackhove, Marcon, Mengaldo, Serson, Turner, Xu, Peir\'o, Kirby,
  and Sherwin]{moxey2020nektar}
D.~Moxey, C.~D. Cantwell, Y.~Bao, A.~Cassinelli, G.~Castiglioni, S.~Chun,
  E.~Juda, E.~Kazemi, K.~Lackhove, J.~Marcon, G.~Mengaldo, D.~Serson,
  M.~Turner, H.~Xu, J.~Peir\'o, R.~M. Kirby, and S.~J. Sherwin.
\newblock {Nektar++}: Enhancing the capability and application of high-fidelity
  spectral/$hp$ element methods.
\newblock \emph{Computer Physics Communications}, 249:\penalty0 107110, 2020.
\newblock \doi{10.1016/j.cpc.2019.107110}.

\bibitem[Nathani et~al.(2025)Nathani, Madaan, Roberts, Bashlykov, Menon, Moens,
  Budhiraja, Magka, Vorotilov, Chaurasia, Hupkes, Cabral, Shavrina, Foerster,
  Bachrach, Wang, and Raileanu]{nathani2025mlgym}
Deepak Nathani, Lovish Madaan, Nicholas Roberts, Nikolay Bashlykov, Ajay Menon,
  Vincent Moens, Amar Budhiraja, Despoina Magka, Vladislav Vorotilov, Gaurav
  Chaurasia, Dieuwke Hupkes, Ricardo~Silveira Cabral, Tatiana Shavrina, Jakob
  Foerster, Yoram Bachrach, William~Yang Wang, and Roberta Raileanu.
\newblock {MLGym}: A new framework and benchmark for advancing {AI} research
  agents.
\newblock arXiv preprint arXiv:2502.14499, 2025.
\newblock URL \url{https://arxiv.org/abs/2502.14499v1}.

\bibitem[Novikov et~al.(2025)Novikov, V{\~u}, Eisenberger, Dupont, Huang,
  Wagner, Shirobokov, Kozlovskii, Ruiz, Mehrabian, Kumar, See, Chaudhuri,
  Holland, Davies, Nowozin, Kohli, and Balog]{novikov2025alphaevolve}
Alexander Novikov, Ng{\^a}n V{\~u}, Marvin Eisenberger, Emilien Dupont, Po-Sen
  Huang, Adam~Zsolt Wagner, Sergey Shirobokov, Borislav Kozlovskii, Francisco
  J.~R. Ruiz, Abbas Mehrabian, M.~Pawan Kumar, Abigail See, Swarat Chaudhuri,
  George Holland, Alex Davies, Sebastian Nowozin, Pushmeet Kohli, and Matej
  Balog.
\newblock {AlphaEvolve}: A coding agent for scientific and algorithmic
  discovery.
\newblock arXiv preprint arXiv:2506.13131, 2025.
\newblock URL \url{https://arxiv.org/abs/2506.13131v1}.

\bibitem[Raissi et~al.(2019)Raissi, Perdikaris, and
  Karniadakis]{raissi2019physics}
Maziar Raissi, Paris Perdikaris, and George~E Karniadakis.
\newblock Physics-informed neural networks: A deep learning framework for
  solving forward and inverse problems involving nonlinear partial differential
  equations.
\newblock \emph{Journal of Computational Physics}, 378:\penalty0 686--707,
  2019.

\bibitem[Romera-Paredes et~al.(2024)Romera-Paredes, Barekatain, Novikov, Balog,
  Kumar, Dupont, Ruiz, Ellenberg, Wang, Fawzi, Kohli, and
  Fawzi]{romeraparedes2024funsearch}
Bernardino Romera-Paredes, Mohammadamin Barekatain, Alexander Novikov, Matej
  Balog, M.~Pawan Kumar, Emilien Dupont, Francisco J.~R. Ruiz, Jordan~S.
  Ellenberg, Pengming Wang, Omar Fawzi, Pushmeet Kohli, and Alhussein Fawzi.
\newblock Mathematical discoveries from program search with large language
  models.
\newblock \emph{Nature}, 625:\penalty0 468--475, 2024.
\newblock \doi{10.1038/s41586-023-06924-6}.
\newblock URL \url{https://www.nature.com/articles/s41586-023-06924-6}.

\bibitem[Salari and Knupp(2000)]{salari2000manufactured}
Kambiz Salari and Patrick Knupp.
\newblock Code verification by the method of manufactured solutions.
\newblock Technical Report SAND2000-1444, Sandia National Laboratories, 2000.
\newblock URL \url{https://digital.library.unt.edu/ark:/67531/metadc702130/}.

\bibitem[Trefethen(2000)]{trefethen2000spectral}
Lloyd~N. Trefethen.
\newblock \emph{Spectral Methods in {MATLAB}}.
\newblock Society for Industrial and Applied Mathematics, 2000.
\newblock \doi{10.1137/1.9780898719598}.
\newblock URL \url{https://doi.org/10.1137/1.9780898719598}.

\bibitem[Tu et~al.(2014)Tu, Rowley, Kutz, and Shang]{tu2014spectral}
Jonathan~H. Tu, Clarence~W. Rowley, J.~Nathan Kutz, and Jessica~K. Shang.
\newblock Spectral analysis of fluid flows using sub-nyquist-rate {PIV} data.
\newblock \emph{Experiments in Fluids}, 55\penalty0 (9):\penalty0 1805, 2014.
\newblock \doi{10.1007/s00348-014-1805-6}.
\newblock URL \url{https://doi.org/10.1007/s00348-014-1805-6}.

\bibitem[Wildberger et~al.(2023)Wildberger, Dax, Buchholz, Green, Macke, and
  Sch{\"o}lkopf]{wildberger2023flowmatching}
Jonas Wildberger, Maximilian Dax, Simon Buchholz, Stephen Green, Jakob~H.
  Macke, and Bernhard Sch{\"o}lkopf.
\newblock Flow matching for scalable simulation-based inference.
\newblock In \emph{Advances in Neural Information Processing Systems},
  volume~36, pages 16837--16864, 2023.
\newblock \doi{10.52202/075280-0737}.
\newblock URL
  \url{https://proceedings.neurips.cc/paper_files/paper/2023/file/3663ae53ec078860bb0b9c6606e092a0-Paper-Conference.pdf}.

\bibitem[Yamada et~al.(2025)Yamada, Lange, Lu, Hu, Lu, Foerster, Clune, and
  Ha]{yamada2025aiscientistv2}
Yutaro Yamada, Robert~Tjarko Lange, Cong Lu, Shengran Hu, Chris Lu, Jakob
  Foerster, Jeff Clune, and David Ha.
\newblock The {AI Scientist-v2}: Workshop-level automated scientific discovery
  via agentic tree search.
\newblock arXiv preprint arXiv:2504.08066, 2025.
\newblock URL \url{https://arxiv.org/abs/2504.08066v1}.

\bibitem[Ye et~al.(2026)Ye, Lin, Tang, Luo, Thapa, Yang, Su, Yang, Liu, Li, Li,
  Sun, Gao, Ding, He, Zhang, Sun, Wang, Zhong, Shen, Li, Lu, Cui, He, Ma, Li,
  Baoyin, Choi, Ermon, Chu, Li, Xu, and Zou]{ye2026simpletes}
Haotian Ye, Haowei Lin, Jingyi Tang, Yizhen Luo, Rahul Thapa, Caiyin Yang,
  Chang Su, Rui Yang, Ruihua Liu, Rundao Li, Zeyu Li, Pengwei Sun, Chong Gao,
  Dachao Ding, Guangrong He, Miaolei Zhang, Lina Sun, Wenyang Wang, Yuchen
  Zhong, Zhuohao Shen, Puheng Li, Pan Lu, Bianxiao Cui, Di~He, Jianzhu Ma,
  Junfeng Li, Hexi Baoyin, Yejin Choi, Stefano Ermon, Xiaowen Chu, Tongyang Li,
  Yuzhi Xu, and James Zou.
\newblock Structured scaling of {AI} discovery across diverse scientific
  domains.
\newblock arXiv preprint arXiv:2604.19341, 2026.
\newblock URL \url{https://arxiv.org/abs/2604.19341v2}.

\bibitem[Yue et~al.(2025)Yue, Somasekharan, Zhang, Cao, and
  Pan]{yue2025foamagent}
Ling Yue, Nithin Somasekharan, Tingwen Zhang, Yadi Cao, and Shaowu Pan.
\newblock {Foam-Agent}: An end-to-end composable multi-agent framework for
  automating {CFD} simulation in {OpenFOAM}.
\newblock arXiv preprint arXiv:2509.18178, 2025.
\newblock URL \url{https://arxiv.org/abs/2509.18178v1}.

\bibitem[Zeng et~al.(2026)Zeng, Zhang, Li, Qi, and Chen]{zeng2026lcp}
Zihang Zeng, Jiaquan Zhang, Pengze Li, Yuan Qi, and Xi~Chen.
\newblock {AI}-for-science low-code platform with bayesian adversarial
  multi-agent framework.
\newblock In \emph{International Conference on Learning Representations}, 2026.
\newblock URL \url{https://openreview.net/forum?id=Cug26Y0RlT}.

\end{thebibliography}
\bibliographystyle{plainnat}

\clearpage
\appendix

\section{PDEs and data sources}
\label{app:pdes}

\subsection{Navier--Stokes equations}
\label{app:ns-equations}

On a two-dimensional fluid domain $\Omega$, the velocity $\mathbf u(t,\mathbf x)$ and pressure $p(t,\mathbf x)$ satisfy
\begin{equation}
\partial_t \mathbf{u} + (\mathbf{u}\cdot\nabla)\mathbf{u}
= -\nabla p + \nu\Delta\mathbf{u} + \mathbf{f},
\qquad \nabla\cdot\mathbf{u}=0.
\label{eq:ns}
\end{equation}
Here $\nu>0$ is the kinematic viscosity, $\mathbf f(t,\mathbf x)$ is the forcing, and the divergence constraint enforces incompressibility.
Each problem specifies the domain, viscosity, initial velocity $\mathbf u(0,\mathbf x)=\mathbf u_0(\mathbf x)$, and boundary conditions.

\paragraph{Numerical data.}
We obtain discrete solution data by numerically solving $100{,}000$ problems in a unit square containing a fixed kite-shaped obstacle and a triangular baffle (Figure~\ref{fig:ns-data-sources}a).
The outer walls and both obstacles impose the no-slip condition $\mathbf u=0$, and the viscosity is $\nu=0.01$.
Initial velocities and forcing vary across problems, with ten families of forcing that include shear, rotating vortices, and localized pulses.
We integrate the equations with a finite-difference projection method, choosing time steps from advective and diffusive stability limits.
For each problem, we sample velocity and pressure on a $64\times64$ cell-centered grid at $32$ uniformly spaced times in $[0,0.12]$.

We then use MMS to construct matching forcing from the sampled velocity and pressure, approximating derivatives by finite differences.
This gives a dataset of $100{,}000$ problems, each paired with the discrete solution data used to construct it.

\begin{figure}[htbp]
\centering
\includegraphics[width=\linewidth]{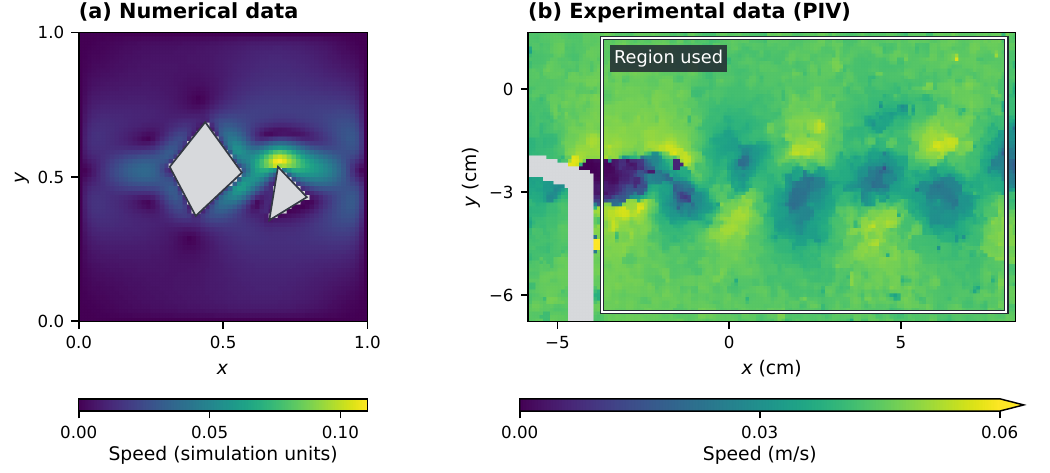}
\caption{\textbf{Navier--Stokes data sources.}
Examples of discrete solution data used to construct the environment.
Color shows fluid speed at a single time.
(a) Numerical simulation in a square domain with two solid obstacles (gray).
(b) PIV measurements of flow behind a cylinder.
The white rectangle marks the measurement region used to construct problems; gray areas indicate missing measurements.
Each panel uses its own color scale; PIV speeds above $0.06$\,m/s use the top color.}
\label{fig:ns-data-sources}
\end{figure}

\paragraph{Experimental data.}
We use planar particle image velocimetry (PIV) measurements of flow past a cylinder in a water channel at Reynolds number $413$, described in \citet{tu2014spectral}.
The dataset contains $8{,}000$ velocity fields sampled at $20$\,Hz, each with two velocity components on a $135\times80$ grid.
For problems constructed from these measurements, we use a kinematic viscosity of $\nu\approx1.0\times10^{-6}$\,m$^2$/s.

To construct problems from these measurements, we first retain every second grid point along each spatial direction.
We then select the largest rectangle whose grid points all have valid measurements at every time.
The selected region contains $57\times39$ measurement locations and no solid obstacles (Figure~\ref{fig:ns-data-sources}b).
For each problem, velocity measurements at $32$ consecutive sampling times provide the discrete solution data over a $1.55$\,s interval.

To form the paired problem inputs in Section~\ref{sec:task}, we take initial and time-dependent boundary velocities from the measured velocity in each observation window.
We construct the forcing from the measured velocity, using zero pressure as a construction convention.

\FloatBarrier

\subsection{Kinetic equations}
\label{app:vpfp-equations}

The kinetic equations describe a particle distribution $F(t,\mathbf x,\mathbf v)\geq0$ coupled to a self-consistent electric field.
We use collisionless VP and collisional VPFP problems.
In dimensionless form, the system is
\begin{align}
\partial_t F + \mathbf v\cdot\nabla_{\mathbf x}F
+ \mathbf E\cdot\nabla_{\mathbf v}F
&= \nu_{\mathrm c}\nabla_{\mathbf v}\cdot
\bigl(\mathbf v F + T_{\mathrm v}\nabla_{\mathbf v}F\bigr) + S,
\label{eq:vpfp}\\
-\Delta_{\mathbf x}\phi &= \rho-\rho_{\mathrm b},
\qquad \mathbf E=-\nabla_{\mathbf x}\phi,
\qquad \rho=\int_{\mathbb R^d}F\,\mathrm d\mathbf v.
\label{eq:vpfp-poisson}
\end{align}
Here $d$ is the spatial and velocity dimension, $\rho$ is the particle density, $\rho_{\mathrm b}=1$ is the uniform neutralizing background, and $S$ is a source term.
The electric field is written as an acceleration, with the charge convention absorbed into its definition.
For VPFP, $\nu_{\mathrm c}>0$ is the collision rate and $T_{\mathrm v}>0$ is the temperature parameter of the linear Fokker--Planck operator, which combines velocity drag and diffusion.
Setting $\nu_{\mathrm c}=0$ gives the collisionless VP equation.

Both VPFP studies use $d=2$, $\mathbf x\in[0,2\pi)^2$, $\mathbf v\in\mathbb R^2$, and $T_{\mathrm v}\in[0.8,1.2]$.
The collision rate ranges over $[0.02,0.08]$ in the weak-collision study and $[10,30]$ in the strong-collision study.
The ARTEMIS-informed VP study uses $d=1$, $x\in[0,2\pi)$, and $v\in\mathbb R$.
All three use periodic spatial boundaries, a zero-mean electric field, and $t\in[0,1]$.
Reference distributions and their moments are defined over the full velocity space; input queries use finite velocity windows.

\paragraph{Program inputs and outputs.}
The program receives the equation parameters, initial distribution $F_0$, and source $S$.
It predicts the electric field and the following velocity moments:
\begin{align}
\rho&=\int F\,\mathrm d\mathbf v,
&\mathbf j&=\int\mathbf v F\,\mathrm d\mathbf v,
&\mathbf u&=\mathbf j/\rho,\nonumber\\
\mathbf P&=\int(\mathbf v-\mathbf u)(\mathbf v-\mathbf u)^\top F\,\mathrm d\mathbf v,
&\mathbf q&=\frac12\int|\mathbf v-\mathbf u|^2(\mathbf v-\mathbf u)F\,\mathrm d\mathbf v.
&&\label{eq:vpfp-macros}
\end{align}
Here $\mathbf j$ is particle flux, $\mathbf P$ is central pressure, and $\mathbf q$ is central heat flux, with unit particle mass.
VPFP programs return ten scalar fields: $\rho$, two components each of $\mathbf E$ and $\mathbf j$, the three independent entries of $\mathbf P$, and two components of $\mathbf q$.
VP programs return $M_k=\int v^kF\,\mathrm dv$ for $k=0,1,2,3$ and $E$.
The Verifier converts these raw moments to $\rho=M_0$, $j=M_1$, $P=M_2-M_1^2/M_0$, and $q=(M_3-3uM_2+2M_0u^3)/2$, using $u=M_1/M_0$ from each answer separately.

\paragraph{Strong-collision numerical data.}
We draw the collision rate and temperature independently and uniformly from the strong-collision ranges above.
The initial density is a normalized exponential Fourier field with four random modes, wave numbers up to four, and Gaussian coefficients of standard deviation $0.1$ in log density.
The conditional velocity distribution is a spatially varying mixture of four to ten anisotropic Gaussian components, with the component count sampled uniformly.
Their base centers are uniform on $[-2,2]^2$, their base standard deviations are log-uniform on $[0.65,1.8]$, and their correlations are uniform on $[-0.6,0.6]$.
Periodic Fourier fields modulate their weights, centers, and widths.
The simulation source has twelve signed Gaussian--Hermite wave packets, with velocity polynomial degree at most three and a common amplitude scale sampled uniformly from $[5\times10^{-4},5\times10^{-3}]$.
The numerical solver uses Fourier differentiation in space, conservative finite-volume velocity fluxes with a Chang--Cooper discretization of the collision term, and fourth-order Runge--Kutta time integration.
Each trajectory is sampled at $33$ uniformly spaced times on a $16\times16$ spatial grid and a $21\times21$ velocity grid covering $[-8,8]^2$.
The internal time step is capped at $10^{-3}$ and further restricted by the stability estimate with CFL factor $0.18$.
The numerical trajectories are projected into the continuous representation described in Appendix~\ref{app:vpfp-generator}.
The projected distribution then determines matching source terms through MMS.
Generator fitting and evaluation use separate trajectory splits, specified in that appendix.

\paragraph{Weak-collision numerical data.}
We use a separate numerical corpus in the weak-collision range, with coarse numerical solves followed by continuous projection.
It records the same $33$ times, $16\times16$ spatial grid, and $21\times21$ velocity grid on $[-8,8]^2$.
The Generator models macroscopic velocity and stress together with microscopic distribution coefficients, as described in Appendix~\ref{app:weak-generator}.

\paragraph{ARTEMIS observations.}
We use electron observations from the Electrostatic Analyzer (ESA) on ARTEMIS P1, also known as THEMIS-B~\citep{mcfadden2008themis,angelopoulos2023themisbesa}.
The inputs are calibrated Level-2 full electron energy-flux distributions, together with magnetic-field, spacecraft-potential, and moment data from the corresponding \texttt{THB\_L2\_ESA} products.
The two observation windows are 00:00--01:00 UTC on 23 October 2011 and 04:00--05:00 UTC on 29 July 2016.
After quality filtering and spacecraft-potential correction, angular integration gives a signed velocity marginal along the local magnetic field, represented using $256$ velocity bins.
These local profiles constrain the Generator's distribution.
The continuous model supplies periodic space--time reference fields, and MMS supplies their matching source terms.

\section{Environment construction}
\label{app:environment}

\subsection{Reference solutions and problem generation}
\label{app:construction}

We construct a separate Generator for each data source.
Each Generator follows the same construction: fit a continuous representation to the available data, define a sampling distribution over that representation, and use MMS to obtain problem inputs from sampled reference solutions.
The representation and sampling model depend on the data and PDE; the sampled reference determines the public inputs and the Verifier's reference quantities through Equation~\eqref{eq:manufactured_case}.
We describe the data splits, continuous representation, learned distribution, and problem construction for each study below.

\subsubsection{NS Generator}
\label{app:ns-generator}

\paragraph{Data splits.}
For the numerical data in Appendix~\ref{app:ns-equations}, we use $80{,}000$ problems to fit the Generator.
We draw validation data from a separate set of $10{,}000$ problems to choose its settings.
The remaining $10{,}000$ problems are excluded from Generator fitting.

For PIV, velocity measurements at the first $3{,}600$ sampling times provide training data, and those at the next $800$ provide validation data for choosing the Generator's settings.
Each problem uses measurements at $32$ consecutive times; shifting the start by four sampling intervals yields $893$ training problems and $193$ validation problems.
The training and validation sets use disjoint time intervals, although problems within each set share measurements.
Later measurements are excluded from Generator fitting and selection.
Table~\ref{tab:ns-generator-settings} gives the Generator settings for the two NS studies.

\paragraph{Continuous representation.}
We represent the velocity through a streamfunction $\psi$, with $\widetilde{\mathbf u}=(\partial_y\psi,-\partial_x\psi)$, so that incompressibility holds by construction~\citep{raissi2019physics}.
We use a spectral representation~\citep{trefethen2000spectral}: tensor products of spatial cosine functions and temporal Chebyshev polynomials for both the streamfunction and pressure.
The spatial basis includes the constant function.
We evaluate the spatial basis in coordinates normalized to $[0,1]$, while taking velocity derivatives with respect to physical coordinates.
For problems generated from the numerical NS data, the domain extends to the outermost grid centers, with coordinate bounds $[1/128,127/128]^2$.

For the numerical data, we multiply the streamfunction by a boundary mask that enforces zero velocity on the outer walls and obstacle surfaces.
The mask rises from zero at a solid boundary to one in the fluid interior, with a transition width of $0.05$ in normalized coordinates.
Both the mask and its first spatial derivatives vanish on each solid boundary, so differentiating the masked streamfunction gives zero velocity there.
For PIV, the observation rectangle lies within the fluid and allows inflow and outflow.
We extend the cosine basis interval to $[-0.25,1.25]$ along each normalized spatial axis to allow nonzero boundary velocities.
The observation rectangle remains $[0,1]^2$, and its boundary velocities are taken from the reference solution and may vary with time.

\begin{samepage}
For problem $i$, let $\widehat{\mathbf u}_i(t_j,\mathbf x_m)$ denote the observed velocity at time $t_j$ and location $\mathbf x_m$.
We omit spatial basis functions whose velocity contribution at the observation points is zero or numerically negligible.
At each sampled time, we fit the remaining spatial coefficients by
\[
a_j=\arg\min_a\left\{
\sum_m\left\|\widetilde{\mathbf u}_a(\mathbf x_m)-\widehat{\mathbf u}_i(t_j,\mathbf x_m)\right\|_2^2
+\eta\|a\|_2^2\right\}.
\]
Here $\widetilde{\mathbf u}_a$ is the velocity obtained from the streamfunction with spatial coefficients $a$; the sum runs over observed fluid locations.
The penalty $\eta\|a\|_2^2$ discourages large coefficients.
For each spatial basis function used in this fit, we evaluate its velocity contribution with unit coefficient and sum the squared components over the observation points.
In both studies, we set $\eta$ to $10^{-6}$ times the median of these sums, so that the penalty scales with the velocity basis.
We then fit the coefficients $a_j$ across sampled times with Chebyshev polynomials by regularized least squares, with the penalty computed from the temporal basis in the same way.
This gives a streamfunction that is continuous in both space and time.
For the numerical data, pressure is fitted after removing its spatial mean at each time.
For PIV, pressure is set to zero, following the convention in Appendix~\ref{app:ns-equations}.
\end{samepage}

\paragraph{Learning the distribution.}
Let $w_i$ be the vector of $K$ fitted space--time streamfunction coefficients for training problem $i$, with $N$ training problems.
We subtract the training mean and divide by a shared root-mean-square scale:
\[
\bar w=\frac{1}{N}\sum_{i=1}^{N}w_i,
\qquad
\sigma=\left(\frac{1}{NK}\sum_{i=1}^{N}\|w_i-\bar w\|_2^2\right)^{1/2},
\qquad
\widetilde w_i=\frac{w_i-\bar w}{\sigma}.
\]
Here $\bar w$ averages each coefficient across problems, while the scalar $\sigma$ pools all $K$ coefficients and all $N$ problems.
For the numerical data, we normalize pressure separately using its own mean and scale; PIV pressure remains zero.
This expresses each solution relative to the training mean and balances the two fields while preserving amplitude differences across problems.

We concatenate the scaled field coefficients and apply principal component analysis (PCA), retaining the directions with the largest variation across training problems.
We choose the smallest tested number of components whose median validation velocity reconstruction error is within $5\%$ of the best value.
Finally, we divide each retained PCA coordinate by its training standard deviation to obtain a whitened vector $z$.
We learn the distribution of these vectors with rectified flow~\citep{liu2023rectified}, training a flow velocity $V$ by minimizing
\begin{equation}
\mathbb E\!\left[\left\|V(z_s,s)-(z-\epsilon)\right\|_2^2\right],
\qquad z_s=(1-s)\epsilon+sz,
\label{eq:appendix-flow-training}
\end{equation}
Here $z$ is sampled from the whitened training vectors, $\epsilon$ is a standard Gaussian vector of the same dimension, and $s\sim\operatorname{Unif}[0,1]$; these draws are independent.
The flow time $s$ is distinct from physical time $t$.
At generation, we integrate the learned flow from Gaussian noise at $s=0$ to $s=1$ using $256$ Heun steps.
We undo the whitening, reconstruct the coefficients in the retained PCA basis, and restore the field scales and training mean.
The resulting coefficients define a continuous reference solution.

The flow model is a residual multilayer perceptron with eight blocks of width $768$, SiLU activations, and layer normalization, conditioned on flow time.
We use AdamW with a peak learning rate of $2\times10^{-4}$, $1{,}000$ warmup steps, and cosine decay to one tenth of the peak rate.
The Generator is fitted before program development and remains fixed during development; guided generation modifies sampling as described in Appendix~\ref{app:guided-generation}.

\paragraph{Problem construction.}
We substitute the generated velocity and pressure into Equation~\eqref{eq:ns} to obtain the matching forcing, following Equation~\eqref{eq:manufactured_case}.
The reference velocity determines the initial and boundary data and supplies the Verifier's evaluation values.

\begin{table}[t]
\centering
\caption{\textbf{NS Generator settings.}}
\label{tab:ns-generator-settings}
\begin{tabular}{@{}lcc@{}}
\toprule
Setting & Numerical data & PIV data \\
\midrule
Cosine basis functions per spatial axis & \multicolumn{2}{c}{$17$} \\
Chebyshev polynomial degree in time & \multicolumn{2}{c}{$12$} \\
Retained principal components & $256$ & $32$ \\
\midrule
Batch size & $512$ & $128$ \\
Training steps & $30{,}000$ & $5{,}000$ \\
Checkpoint used (training step) & $15{,}000$ & $5{,}000$ \\
\bottomrule
\end{tabular}
\end{table}

\FloatBarrier

\subsubsection{Shared construction for kinetic problems}
\label{app:kinetic-construction}

For both VPFP regimes and collisionless VP, the Generator supplies a continuous reference distribution $\widetilde F$.
We obtain its density by integration over velocity and solve the periodic Poisson equation for the zero-mean electric field $\widetilde{\mathbf E}$.
We define the source by substituting that same distribution into Equation~\eqref{eq:vpfp}:
\begin{equation}
S=\partial_t\widetilde F+\mathbf v\cdot\nabla_{\mathbf x}\widetilde F
+\widetilde{\mathbf E}\cdot\nabla_{\mathbf v}\widetilde F
-\nu_{\mathrm c}\nabla_{\mathbf v}\cdot
\bigl(\mathbf v\widetilde F+T_{\mathrm v}\nabla_{\mathbf v}\widetilde F\bigr).
\label{eq:vpfp-manufactured-source}
\end{equation}
Collisionless VP uses $\nu_{\mathrm c}=0$.
Initial data are the reference distribution at $t=0$, and spatial periodicity is built into each representation.
The problem inputs and the Verifier's reference quantities therefore come from the same continuous distribution.
The three Generators differ in how they represent and sample this distribution.

\subsubsection{Strong-collision VPFP Generator}
\label{app:vpfp-generator}

\paragraph{Data splits.}
We use $2{,}048$ numerical trajectories for Generator fitting and $128$ to monitor validation denoising loss.
Three further sets of $128$ trajectories each are held out for calibration and reference-distribution diagnostics.
The splits contain distinct equation seeds, and coefficient normalization uses only the training trajectories.

\paragraph{Continuous representation.}
We write $\widetilde F=\widetilde\rho\,\widetilde g$, where $\widetilde g$ is a normalized conditional velocity density.
For $\widetilde\rho$, we fix the constant Fourier mode to one and rescale the nonconstant coefficients so that the magnitude of the density fluctuation stays below one.
This preserves positivity and spatial mean density one.
For $\widetilde g$, we square an expansion in $120$ tensor-product Hermite functions of total degree at most $14$ and divide the result by its integral over $\mathbb R^2$.
Orthonormality makes this integral the squared norm of the Hermite coefficient vector.
We mix the normalized density with a Maxwellian of weight $10^{-5}$ to keep it strictly positive.
We fit the Hermite amplitude to the square root of the discrete conditional velocity density by weighted least squares.
Spatial Fourier expansions then represent the fitted coefficient fields, with mode orders six for density and four for the Hermite coefficients.
Piecewise cubic Hermite interpolation through $33$ control times supplies continuous time dependence, with slopes computed from neighboring control values.

\paragraph{Learning the distribution.}
We train a conditional diffusion model with EDM preconditioning~\citep{karras2022elucidating}.
The model is conditioned on the collision rate, temperature parameter, and complete initial density and Hermite coefficient fields.
It denoises coefficient values at the $32$ future control times; the initial field remains fixed during sampling.
We normalize coefficient groups using the training set and whiten their future temporal coordinates conditional on the initial values.
The denoiser has four residual space--time convolutional blocks of width $256$.
We use AdamW with learning rate $2\times10^{-4}$, zero weight decay, and batch size four, and retain the fixed checkpoint after $40{,}000$ updates.
The loss weights the density- and Hermite-coefficient denoising terms equally.
Training noise levels satisfy $\log\sigma\sim\mathcal N(-0.6,1.8^2)$, clipped to $\sigma\in[0.002,80]$, with EDM data scale one and gradient-norm clipping at one.
At generation, we sample the same initial-distribution law used for the numerical data and project its initial field into the Fourier--Hermite representation.
We generate the future coefficients with $64$ Heun steps, using the EDM noise schedule with exponent seven and no added sampling noise, then invert the normalization and construct the interpolation slopes.
After sampling, the generated field determines the source.

\paragraph{Problem construction.}
We apply the construction in Appendix~\ref{app:kinetic-construction}, solving the periodic Poisson equation in Fourier space.
The Verifier computes strong-collision reference moments from the Hermite representation over $\mathbb R^2$.

\subsubsection{Weak-collision VPFP Generator}
\label{app:weak-generator}

\paragraph{Data splits.}
We use $512$ numerical trajectories for fitting and $128$ for validation, with three further sets of $128$ reserved for calibration and reference-distribution diagnostics.

\paragraph{Continuous representation.}
The representation uses $33$ time controls and a Fourier--Hermite base with $120$ Hermite channels.
The moving-frame decoder first standardizes the underlying velocity density and then imposes the generated macroscopic mean and covariance.
The resulting continuous distribution is nonnegative and retains the microscopic degrees of freedom.

\paragraph{Learning the distribution.}
A joint macro diffusion model represents velocity and stress coordinates, with separate outputs for persistent and time-varying global stress.
A conditional micro diffusion model represents the full kinetic coefficients in temporal-increment coordinates.
The moving-frame decoder combines the outputs of these two models.
Both denoisers have four blocks of width $256$ and use EDM preconditioning~\citep{karras2022elucidating}.
We train them with AdamW, batch size four, zero weight decay, and learning rates $2\times10^{-4}$ for the macro model and $10^{-5}$ for the micro model.
Micro-model training also includes an auxiliary source penalty evaluated on generated states from a frozen replay bank.
We fix the checkpoints at $5{,}000$ and $20{,}000$ updates, respectively, and use validation for diagnostics.
Both models use $64$ denoising steps at generation.

\paragraph{Problem construction.}
We derive the initial distribution, electric field, and source from the resulting continuous distribution using the shared kinetic construction in Appendix~\ref{app:kinetic-construction}.

\subsubsection{ARTEMIS-informed VP Generator}
\label{app:vp-generator}

\paragraph{Data splits.}
The profile data contain $899$ fitting profiles and $432$ held-out profiles, organized into $225$ and $108$ observation blocks.
The split is internal to the two observation windows in Appendix~\ref{app:vpfp-equations}.
Within these observation windows, six conditioning groups distinguish event, geometry, and quality categories.

\paragraph{Continuous representation.}
We encode the square roots of the $256$-bin probability masses using a $16$-component PCA fitted only on the fitting split, and retain relative log density as an additional coordinate.
A learned continuous lift maps these profile coordinates to a nonnegative reference distribution.
It has $123{,}414$ parameters, hidden width $128$, two residual modulation blocks, and $32$ density Fourier modes.
Squared, normalized profile amplitudes and smooth logistic bin atoms represent the velocity density; learned modulations provide continuous spatial and temporal dependence.

\paragraph{Learning the distribution.}
Within each conditioning group, a finite-sample prior samples piecewise-linear interpolations between adjacent nodes constructed from the fitting profiles and ordered by standardized log density.
We train the continuous lift with AdamW, learning rate $3\times10^{-4}$, weight decay $10^{-4}$, and batches of $31$ profiles.
Training stops after $56$ epochs and restores the checkpoint from epoch $46$, selected using a fixed held-out criterion combining variation, profile reconstruction, and source diagnostics.
The model and profile prior remain fixed during program development.

\paragraph{Problem construction.}
We apply the shared kinetic construction in Appendix~\ref{app:kinetic-construction} with $\nu_{\mathrm c}=0$.
The Verifier computes the base distribution's full-line moments analytically and uses deterministic quadrature corrections for velocity-map perturbations.
This construction produces new verifiable VP problems informed by the observed profile distribution.

\FloatBarrier

\subsection{Guided generation}
\label{app:guided-generation}

Guided generation fits predictors to program-error feedback and uses their gradients to modify sampling while leaving the base Generator fixed.
After sampling, each Generator decodes the continuous reference and applies the same MMS construction used for unguided problems.
These learned guidance rules pursue the loss-tilted target in Equation~\eqref{eq:exact-guidance-tilt}, which gives more weight to problems with larger program errors.
\Cref{thm:guided-convergence} analyzes the idealized step that samples exactly from this target.

\paragraph{NS.}
For the rectified-flow Generator in Appendix~\ref{app:ns-generator}, we approximate this step by adapting the return guidance of \citet{janner2022planning}.
Their method learns a return predictor and uses its gradient to guide sampling; here we predict program error and add its gradient to the flow velocity.

Given a program and a set of unguided, unperturbed problems, we compute its error on each problem as the mean of the two velocity components' relative $\ell_2$ errors, as in Equation~\eqref{eq:ns-velocity-error}.
For feedback problem $i$, let $z_i$ be its whitened PCA vector.
We standardize the errors by subtracting their mean and dividing by their standard deviation across the feedback problems, and denote the resulting target by $e_i$.
We train a scalar predictor $h(z_s,s)$ by minimizing
\[
\mathbb E_{i,\epsilon,s}\!\left[\bigl(h(z_s,s)-e_i\bigr)^2\right],
\qquad z_s=(1-s)\epsilon+sz_i.
\]
Here $i$ is sampled uniformly from the feedback problems, independently of the Gaussian noise $\epsilon$ and uniform flow time $s$ in Equation~\eqref{eq:appendix-flow-training}.
This gives a differentiable prediction of standardized program error at intermediate flow states.

During generation, we evaluate the following velocity at the current whitened PCA vector $z$ and flow time $s$:
\[
V_{\mathrm{guided}}(z,s)
=V(z,s)+\gamma(z,s)\nabla_z h(z,s).
\]
The nonnegative factor $\gamma(z,s)$ controls the norm of the added guidance relative to the original flow velocity $V(z,s)$; Table~\ref{tab:ns-extension-settings} gives the target ratio and its upper bound.
The gradient is taken with respect to the whitened PCA coordinates and points toward increasing predicted error.
We convert the final PCA vector back to a continuous reference solution as in Appendix~\ref{app:ns-generator}, then use MMS (Equation~\eqref{eq:manufactured_case}) to construct the matching problem inputs.

For the numerical NS study in Section~\ref{sec:main-comparison}, guided generation during development starts in round two.
Before rounds two through five, we evaluate the current program on the first $32$ successfully acquired unguided, unperturbed development problems and fit the predictor to these errors.
For PIV, we use $64$ unguided, unperturbed feedback problems: a new set for the initial program before round one, and the preceding selection set's unguided problems for later rounds.
Both agents share this fixed predictor during the round.
For our CodePDE, AutoNumerics, and PDE-SHARP adaptations (Section~\ref{sec:pde-agent-workflows}), GAV-Full uses a common predictor fitted before development and held fixed.

\paragraph{Kinetic feedback.}
Before each round, we fit guidance using the current program's transport errors on a fixed prefix of unguided, unperturbed feedback problems.
The first guidance model in every run uses $80$ problems, split into $64$ fitting and $16$ held-out problems.
Later rounds use $32$ problems, split into $26$ and six, except that the second-round guidance models in the weak- and strong-collision Sol runs retain the $80$-problem split.
The split is fixed before noisy replicas are formed, and both agents share the fitted guidance model throughout the round.

\paragraph{Strong-collision VPFP.}
The predictor has two residual space--time convolutional blocks of width $32$ and takes the noisy future coefficients, noise level, physical parameters, and complete initial field as inputs.
We train it for $2{,}000$ AdamW updates with learning rate $2\times10^{-4}$, weight decay $0.01$, batch size four, and gradient-norm limit one.
Validation is diagnostic; we use the final checkpoint and share it between both agents for the round.

For a valid feedback transport error $e_i$, define the bounded target $r_i=\log(1+e_i)/(1+\log(1+e_i))$ and standardize it using the valid fitting targets' mean and population standard deviation to obtain $\widehat r_i$.
The scalar predictor $h$ is trained by squared-error regression of $\exp(h)$ onto weights proportional to $\exp(0.5\widehat r_i)$, normalized to mean one over the fitting problems.
Invalid evaluations have regression target zero; an auxiliary validity-classification loss has weight $0.1$ and does not enter the sampling score.
For noisy future coefficients $z_\sigma$, let $c$ denote the physical parameters and fixed initial field.
Guided sampling replaces the EDM denoiser by
\[
D_{\mathrm G}(z_\sigma,\sigma,c)
=D(z_\sigma,\sigma,c)+\lambda_\sigma\sigma^2
\operatorname{clip}_1\!\left(\nabla_{z_\sigma}h(z_\sigma,\sigma,c)\right),
\]
where $\operatorname{clip}_1$ limits the gradient's Euclidean norm to one.
Guidance acts only for $\sigma\in[0.002,2]$ and leaves the initial field fixed.
The gain targets a ratio of $0.1$ between the RMS guidance drift and the base drift $(z_\sigma-D)/\sigma$, capped at $0.25$, with $\lambda_\sigma\leq10^6$.
Numerically vanishing gradients or base drift disable the update.
We retain the base model's $64$-step Heun sampler.

\paragraph{Weak-collision VPFP.}
We use separate guidance heads for the macro and micro stages of the Generator.
Each head has width $32$ and two blocks combining spatial Fourier mixing with nonperiodic temporal convolutions.
Both heads use $2{,}000$ AdamW updates with learning rate $2\times10^{-4}$, batch size four, and gradient-norm limit one.
The micro head conditions on the clean macro code.
Targets use the same standardized bounded-error exponential weighting and relative-drift guidance rule as in the strong-collision study.
The base networks and continuous decoder remain frozen.

\paragraph{Collisionless VP.}
Guidance acts on three primitive Gaussian coordinates while preserving the Generator's condition stream.
For each observed condition, we fit a degree-two normalized probabilists' Hermite expansion by closed-form ridge regression with penalty $10$, using a log-error target.
Nonconstant guidance terms are enabled only if that condition has at least two held-out examples, positive held-out Spearman correlation, and top-quartile lift above one.
Unsupported conditions retain zero guidance gradients.
Sampling uses the resulting bounded transport rule with the base model fixed.

\subsection{Perturbation}
\label{app:perturbation}

Perturbation lets the Agent explore local variations of generated problems.
We transform a generated continuous reference solution, choosing the transformation to preserve the constraints required by the PDE.
We then use MMS (Equation~\eqref{eq:manufactured_case}) to construct matching problem inputs, so the perturbed solution is the exact reference for the new problem.
Each implementation below specifies which reference quantities change and which remain fixed.

\paragraph{NS.}
The Generator in Appendix~\ref{app:ns-generator} uses cosine mode indices $0,\ldots,16$ along each spatial axis.
We retain these coefficients and add a streamfunction perturbation $\delta\psi$ using modes with nonnegative integer indices $k_x,k_y$ satisfying $17\leq\max(k_x,k_y)\leq20$.
The streamfunction representation gives the corresponding velocity change:
\[
\psi\longmapsto\psi+\delta\psi,
\qquad
\widetilde{\mathbf u}\longmapsto
\widetilde{\mathbf u}+(\partial_y\delta\psi,-\partial_x\delta\psi).
\]
Pressure is unchanged, and the added spatial coefficients are drawn independently from a Gaussian distribution.
The perturbation's time dependence is proportional to $\bigl[(t-t_0)/(t_1-t_0)\bigr]^2$ on the problem's physical time interval $[t_0,t_1]$.
We normalize the space--time coefficient vector of $\delta\psi$ to unit $\ell_2$ norm, then scale it to the amplitude in Table~\ref{tab:ns-extension-settings}.

The streamfunction construction preserves incompressibility, and the quadratic time factor preserves the initial velocity and its first time derivative.
For the numerical domain, the same geometry mask preserves the no-slip boundaries.
For PIV, we recompute the boundary velocities from the perturbed reference solution.
Appendix~\ref{app:perturbation-case} illustrates how perturbation feedback informs revisions to an NS solver.

\paragraph{Strong-collision VPFP.}
We embed an unguided or guided reference distribution into an expanded representation, increasing the density Fourier order from six to eight, the spatial order of the Hermite coefficient fields from four to six, and the velocity Hermite total degree from $14$ to $18$.
We add six new density Fourier modes and twelve kinetic coefficient directions outside the original representation.
Their time controls and first derivatives vanish at $t=0$, preserving the initial distribution.
The bounded density and normalized squared-Hermite construction retain positivity, spatial mean density one, and conditional velocity normalization.
We then recompute the electric field and source from the perturbed distribution.

Kinetic time controls sample normalized $t^2\sin(\pi k t+\varphi)$, with $k$ uniform on $\{1,\ldots,5\}$ and $\varphi$ uniform on $[-\pi,\pi]$.
For each density mode's sine--cosine coefficient pair, we use a normalized cubic path $t^2R_\varphi(1+a t,b t)$, where $R_\varphi$ is a planar rotation, $a$ is uniform on $[-0.45,0.45]$, and $b$ has a random sign and magnitude uniform on $[0.15,0.75]$.
This pair stays nonzero for $t>0$.
The density and kinetic direction scales are $0.025/\sqrt{6}$ and $0.055/\sqrt{12}$, respectively.
We try at most four directions, starting each at amplitude one and halving its amplitude at most nine times.
We accept a perturbation only when it satisfies the checks in Table~\ref{tab:vpfp-perturbation-settings}.

The checks use $256$ uniformly sampled space--time probes and $512$ phase-space probes, with velocity coordinates sampled from $\mathcal N(0,T_{\mathrm v}I)$.
Macroscopic displacement is the joint relative $\ell_2$ change in $(\rho,\mathbf j,K)$, normalized by the parent fields, where $K=\frac12\int|\mathbf v|^2F\,\mathrm d\mathbf v$ is kinetic-energy density.
Source displacement is normalized by the larger of the parent and child source norms.
Coefficient RMS quantities are square roots of the mean squared Euclidean coefficient norm over probes; the time-derivative change is additionally multiplied by the interval length.
Acceptance is determined on these finite probe sets.

\begin{table}[htbp]
\centering
\caption{\textbf{Strong-collision VPFP perturbation acceptance settings.} All listed checks must pass for an accepted child problem. Initial-state checks use the same spatial probes at $t=0$.}
\label{tab:vpfp-perturbation-settings}
\begin{tabular}{@{}lc@{}}
\toprule
Quantity & Accepted range or bound \\
\midrule
Joint relative macroscopic change & $[0.002,0.20]$ \\
Hermite-coefficient RMS change & $[0.005,0.35]$ \\
Relative source change & $[10^{-4},0.80]$ \\
Mean squared norm of added Hermite coefficients & $\geq10^{-5}$ \\
RMS change in coefficient time derivative & $\leq4$ \\
Initial relative macroscopic change & $\leq2\times10^{-6}$ \\
Initial Hermite-coefficient RMS change & $\leq2\times10^{-6}$ \\
\bottomrule
\end{tabular}
\end{table}

\paragraph{Weak-collision VPFP.}
We perturb the moving-frame velocity and Cholesky controls coherently, retaining the complete microscopic code, density, and positive floor.
A randomly chosen spatial mode of order at most two is multiplied by a normalized $C^1$ temporal envelope whose initial value and derivative vanish.
Velocity, log-diagonal, and off-diagonal amplitudes are $0.03\sqrt{T_{\mathrm v}}$, $0.02$, and $0.02\sqrt{T_{\mathrm v}}$, respectively.
These perturbations stay within the native spatial representation and preserve the initial distribution.
The continuous decoder recomputes the source from the perturbed distribution; the density and its Poisson field are unchanged.

\paragraph{Collisionless VP.}
We apply a smooth orientation-preserving velocity map $T$ within each of the $256$ sensor bins.
The map fixes bin boundaries and is the identity outside the sensor window.
The transformed reference
$\widetilde F_{\mathrm{child}}(t,x,v)=\widetilde F(t,x,T(v))T'(v)$
preserves nonnegativity, each bin's mass, the density, and the Poisson field.
Higher moments are recomputed from the parent's full-line moments plus quadrature corrections within the window.
We apply the same velocity map to the initial distribution to obtain the child initial condition at $t=0$.
We derive the child source from the same transformed reference by the chain rule.

\begin{table}[htbp]
\centering
\caption{\textbf{NS guidance and perturbation settings.} The numerical and PIV studies share these settings. Here $[t_0,t_1]$ is the physical time interval.}
\label{tab:ns-extension-settings}
\begin{tabular}{@{}lc@{}}
\toprule
Parameter & Value \\
\midrule
\multicolumn{2}{@{}l}{\emph{Error predictor $h$}} \\
Residual blocks & $6$ \\
Hidden width & $512$ \\
Embedding dimension for flow time $s$ & $128$ \\
Training steps & $5{,}000$ \\
Batch size & $256$ \\
Optimizer & AdamW \\
Learning rate & $2\times10^{-4}$ \\
Weight decay & $10^{-4}$ \\
Gradient clipping threshold ($\ell_2$) & $1$ \\
\midrule
\multicolumn{2}{@{}l}{\emph{Guided sampling}} \\
Target guidance-to-flow norm ratio & $0.1$ \\
Maximum guidance-to-flow norm ratio & $0.25$ \\
\midrule
\multicolumn{2}{@{}l}{\emph{Streamfunction perturbation}} \\
Spatial cosine mode indices & $17\leq\max(k_x,k_y)\leq20$ \\
Physical time factor & $\bigl[(t-t_0)/(t_1-t_0)\bigr]^2$ \\
Norm of added coefficients ($\ell_2$) & $2.862\times10^{-6}$ \\
\bottomrule
\end{tabular}
\end{table}

\FloatBarrier

\section{Experimental details}
\label{app:protocol}

\subsection{Agent settings and prompts}
\label{app:agent-settings}
We use the multi-turn loop in Figure~\ref{fig:agent-loop}, with the coding model, number of rounds, and time budget set for each study as in Table~\ref{tab:development-settings}.
Both agents within a run use the same coding model.
At the start of each round, they receive the same current program and cumulative memory log.
One agent explores alternative approaches, while the other refines the current approach; both are free to change methods in response to numerical feedback.
The selected program, both agents' research notes, and the evaluation results pass to the next round.
Program selection is described in Appendix~\ref{app:selection-evaluation}.
Sol denotes \texttt{gpt-5.6-sol} with maximum reasoning effort, and Astra denotes \texttt{gpt-6-astra} with ultra reasoning effort.

\begin{table}[htbp]
\centering
\caption{\textbf{Configurations of the shared multi-turn development setting.} Time budgets are per agent per round and include tool calls and waiting. Each kinetic task has separate Sol and Astra runs.}
\label{tab:development-settings}
\begin{tabular}{@{}lccc@{}}
\toprule
Study & Coding model & Rounds & Time budget (h) \\
\midrule
Numerical and PIV-derived NS & Sol & 5 & 4 \\
Kinetic tasks & Sol & 6 & 3 \\
Kinetic tasks & Astra & 6 & 3 \\
NS ablations & Sol & 3 & 2 \\
\bottomrule
\end{tabular}
\end{table}

The prompt templates below separate the shared task instructions from the two roles.
Bracketed fields specify the study's equation and task, public inputs, required outputs, evaluation criterion, and development budget.

\begin{quote}
\begin{minipage}{\linewidth}
\noindent\textbf{Shared task prompt (template).}
Develop a program for [PDE and task] that maps [public inputs] to [required outputs].
Improve its performance under [evaluation criterion] within [development budget].
Use the current program and memory log as context.
Choose your methods and experiments using the problems and tools available in this environment.
Use numerical feedback to assess your changes.
Before the round ends, submit your strongest program and leave concise research notes describing the methods tried, observations, and useful next steps.
\end{minipage}
\end{quote}

\begin{quote}
\begin{minipage}{\linewidth}
\noindent\textbf{Exploration role.}
Focus primarily on exploring new approaches to the task. Pursue the ideas you judge most promising for producing a stronger program. Use the current program and available evidence as context, and follow useful findings wherever they lead. You have full freedom to choose methods, representations, and experiments, including combining new ideas with improvements to existing approaches.
\end{minipage}
\end{quote}

\begin{quote}
\begin{minipage}{\linewidth}
\noindent\textbf{Refinement role.}
Focus primarily on improving the current approach. Use the current program and available evidence to identify where further work could produce a stronger program. You decide what to retain, change, or replace, and follow useful findings wherever they lead. You have full freedom to choose methods, representations, and experiments, including a different approach when it offers a better path forward.
\end{minipage}
\end{quote}

\paragraph{Numerical NS initial program and resources.}
The initial program holds the queried initial velocity fixed throughout the time interval.
The run uses seed $101$ and allocations with four NVIDIA L40S GPUs, $64$ CPU cores, and $160$\,GiB of host memory.

\paragraph{Kinetic initial programs.}
Each kinetic task has separate Sol and Astra runs, both starting from the same task-specific initial program.
For VPFP, the initial program samples the public initial density on a $32\times32$ grid, holds it fixed, computes its periodic Poisson field, and uses a static Maxwellian closure: $\mathbf j=\mathbf q=0$ and $\mathbf P=\rho T_{\mathrm v}I$.
For VP, the initial program computes raw velocity moments through third order using $128$-point Gauss--Legendre quadrature on the declared velocity window, solves periodic Poisson for the initial electric field, and holds these fields fixed.
The agents are free to replace these approximations.

\paragraph{Kinetic development settings.}
Each run uses seed $101$, giving one development trajectory per task and coding model.
The kinetic runs use four NVIDIA H100 GPUs, with the same $64$ CPU cores and $160$\,GiB of host memory as the numerical NS study.
Each agent may acquire at most $32$ new development problems per round, including four initial unguided problems and all perturbed children.
Each agent may nominate at most one primary candidate for selection.
We use $128$ hidden evaluation points per development problem and $256$ per selection or test problem.
The base Generators remain fixed throughout development.

\subsection{Program selection and final evaluation}
\label{app:selection-evaluation}
\begin{samepage}
\paragraph{Program selection.}
At the end of each round in the setting of Appendix~\ref{app:agent-settings}, we compare the current program and the agents' nominated candidates on the same selection problems, using the same evaluation points and execution budget per problem.
We retain the program with the lowest mean error among those that complete the evaluation with finite errors.
If neither candidate improves on the current program, we retain it for the next round.

\paragraph{Final evaluation.}
After the prescribed last development round, we evaluate the retained program on a separate test set.
The final round is fixed by the development configuration in Table~\ref{tab:development-settings}; test results do not determine which round is reported as final.
Within each study, all programs use the same test problems and evaluation points, with the task-specific error measure defined in Section~\ref{sec:protocol}.
For Figures~\ref{fig:development-rounds} and~\ref{fig:kinetic-development}, we apply this test evaluation to the program retained after each round.
\end{samepage}

\paragraph{Program execution budget.}
The per-problem execution limit is $900$\,s for the NS and kinetic studies in Section~\ref{sec:main-comparison} and $300$\,s for the comparisons in Sections~\ref{sec:pde-agent-workflows} and~\ref{sec:environment-analysis}.

\paragraph{NS evaluation.}
For generated NS problems, evaluation times are sampled uniformly over the problem's time interval and spatial locations uniformly over the fluid domain.
The NS studies use $96$ evaluation points per selection problem and $128$ per test problem, shared by the programs being compared.
On generated NS sets, selection and testing average the mean errors of unguided, guided, and perturbed problems with equal weights.
On the interpolated test set, we average all per-problem errors equally.
Evaluation against discrete solution data uses the available fluid locations and, for PIV, the measurement times.

\paragraph{Numerical solver comparisons.}
\label{app:numerical-baselines}
\label{app:comsol}
Table~\ref{tab:ns-solver-comparison} reports the comparison on the main numerical NS test set.
The baselines solve the same test problems as GAV-Full, with the same geometry, viscosity, boundary conditions, and continuous input queries, and are scored at the same points.
Our FEniCSx solver is adapted from a public DOLFINx implementation.\footnote{\url{https://github.com/jorgensd/dolfinx_ipcs}}

\paragraph{Kinetic evaluation.}
For VPFP, the evaluation points are sampled uniformly over the space--time domain.
For VP, we first sample one noninitial native time and then sample spatial locations uniformly over the periodic domain; public source access is restricted to the sampled time horizon.
Programs receive the scoring query after preparation.
For VP, the time horizon is public during preparation, while the spatial evaluation coordinates remain hidden.
The same coordinates are used for all output fields and programs in each comparison.

Selection uses fresh unguided, guided, and perturbed unguided problems each round, with guided problems targeting the entering program. We use a mixture of $16+16+16=48$ problems.
Selection ranks programs by the mean score in Equation~\eqref{eq:kinetic-error} alone.
Additional unguided evaluations provide guidance feedback for the selected program, separately from that round's selection.

The fixed test set contains four groups of 16 problems: unguided, guided, perturbed unguided, and perturbed guided.
Its guided problems use the initial guidance model.
Within each task, the test problems and queries are shared across Sol, Astra, and all six retained rounds; they do not determine program selection.
Strong-collision Astra retains its round-3 program in round 4.
The 36 round records therefore correspond to 35 distinct programs and $35\times64=2{,}240$ program--problem scores.

\paragraph{Kinetic normalization.}
The scales in Equation~\eqref{eq:kinetic-error} are frozen from independent reference-only calibration data.
Calibration uses $32$ VP problems and $64$ problems from each VPFP regime, with $128$ points per problem.
Each VPFP regime has eight strata formed by four collision-rate bins crossed with two temperature bins.
For each field, calibration averages its mean-square norm first within each stratum, then equally across strata within a task, and finally equally across the three kinetic tasks; the square root gives the fixed scale.
In the dimensionless units used here, $(a_\rho,a_{\mathbf E},a_{\mathbf j},a_{\mathbf P},a_{\mathbf q})$ is approximately
$(1.00413,\,0.0698186,\,0.478835,\,2.93076,\,0.990077)$.
The same scales and weights are used in all six runs.
Pressure errors use the full tensor norm, including both off-diagonal entries in two dimensions.
For VP, predicted and reference raw moments are centered using their own density and mean velocity before comparison.

Table~\ref{tab:problem-sets} summarizes the problem sets used for initial development, selection, and testing.

\begin{table}[htbp]
\centering
\caption{Problem sets for initial development, program selection, and evaluation. Counts are per dataset. Unguided problems are sampled directly from the learned Generator. NS perturbations use unguided parents; kinetic test sets include both unguided and guided parents. Selection sets in Section~\ref{sec:main-comparison} are refreshed each round; test sets are fixed. Numerical NS uses $192$ selection problems in round one ($64$ per source) and $32$ in subsequent rounds. Selection sets for Sections~\ref{sec:pde-agent-workflows} and~\ref{sec:environment-analysis} are fixed. The numerical NS test set is shared with COMSOL, FEniCSx, and Nektar++.}
\label{tab:problem-sets}

\begin{tabular}{@{}>{\raggedright\arraybackslash}p{0.27\linewidth}>{\raggedright\arraybackslash}p{0.17\linewidth}r>{\raggedright\arraybackslash}p{0.39\linewidth}@{}}
\toprule
Study & Use & Size & Composition \\
\midrule
Numerical NS (Section~\ref{sec:main-comparison}) & Selection, R2--R5 & 32 & 11 unguided, 11 guided, 10 perturbed \\
 & Test & 32 & 11 unguided, 11 guided, 10 perturbed \\
\midrule
PIV-derived NS (Section~\ref{sec:main-comparison}) & Selection per round & 192 & 64 unguided, 64 guided, 64 perturbed \\
 & Test & 768 & 256 unguided, 256 guided, 256 perturbed \\
\midrule
Existing PDE agents (Section~\ref{sec:pde-agent-workflows}) & Initial development & 64 & Generated NS problems shared across environments \\
\midrule
NS comparisons (Sections~\ref{sec:pde-agent-workflows} and~\ref{sec:environment-analysis}) & Selection & 24 & 8 unguided, 8 guided, 8 perturbed \\
 & Test & 96 & 32 unguided, 32 guided, 32 perturbed \\
\midrule
Interpolated NS (Appendix~\ref{app:interpolation-test}) & Test & 96 & Interpolated numerical solutions \\
\midrule
Kinetic studies (Section~\ref{sec:main-comparison}) & Selection & 48 & 16 unguided, 16 guided, 16 perturbed unguided \\
 & Test & 64 & 16 each: unguided, guided, perturbed unguided, perturbed guided \\
\midrule
Independent VPFP (Appendix~\ref{app:kinetic-independent}) & Test per regime & 8 & Interpolated numerical reference distributions \\
Independent VP (Appendix~\ref{app:kinetic-independent}) & Test & 6 & 12 unused ARTEMIS profiles from two observation windows \\
\bottomrule
\end{tabular}

\end{table}
\FloatBarrier

\subsection{Adapting existing PDE agents}
\label{app:pde-agent-adaptations}
GAV provides the experimentation environment for the three methods in Section~\ref{sec:pde-agent-workflows} to develop NS solvers.
The agents use this environment to request new problems, query initial velocity and forcing, and obtain numerical feedback during their existing development stages.
Our NS reimplementations preserve each method's procedure for proposing, revising, and selecting programs, with the same workflow settings across its three environment variants.
All use \texttt{gpt-5.6-sol} with maximum reasoning effort.
Each method completes its prescribed stages without a fixed wall-clock development limit; per-problem execution limits are given in Appendix~\ref{app:selection-evaluation}.

Original, GAV, and GAV-Full start from the same fixed collection of 64 NS problems sampled from the learned Generator in Appendix~\ref{app:ns-generator}.
Each problem provides initial velocity on a $64\times64$ cell-centered grid, together with forcing and reference velocities on that grid at 32 uniformly spaced times in $[0,0.12]$.
Original uses this collection throughout development; the GAV variants additionally let agents request new problems and query initial velocity and forcing at chosen coordinates.
GAV-Full also provides guided generation and perturbation.
Only the Verifier has access to the reference solutions of newly generated problems.
It evaluates submitted programs against these solutions and returns numerical error feedback to the agents.

\paragraph{CodePDE.}
CodePDE~\citep{li2025codepde} independently generates solver candidates, repairs execution failures, and refines programs using successful implementations and their measured errors.
Our NS implementation follows the authors' \href{https://github.com/LithiumDA/CodePDE/tree/26938c5832ddd0a957426dff6f702b54649c14af}{released refinement code}.
We use three initial candidate branches followed by three refinement branches, each receiving the same fixed pool of valid initial candidates and their evaluation results.
In Original, generation, debugging, and refinement use the fixed data collection and candidate evaluation results.
With GAV, these same stages can request further problems, query inputs at chosen coordinates, and evaluate revised programs before submission.

\paragraph{AutoNumerics.}
AutoNumerics~\citep{du2026autonumerics} uses specialized agents to formulate a PDE, propose and select numerical plans, and implement and repair the chosen plans.
Our implementation follows the authors' \href{https://github.com/Daviddjddu/Autonumerics/tree/adc25f7c32e278d63b9ff8a1f6d882da4ec2f560}{released planning and execution workflow}.
It proposes ten numerical plans, selects up to five, and tests each implementation on coarse and fine grids.
An agent selects the final solver from these execution reports, prioritizing accuracy and stability, and a final reasoning stage explains the choice.
In Original, coding and repair use the fixed sampled inputs.
With GAV, these stages can acquire further problems and query inputs at the coordinates required by the chosen spatial and temporal discretization.

\paragraph{PDE-SHARP.}
PDE-SHARP~\citep{fazliani2025pdesharp} begins with mathematical analysis and independent solver candidates.
Judge agents nominate candidates, receive numerical feedback, and revise or combine their approaches.
Our NS implementation follows the paper's algorithm and appendix prompts, with four initial candidates, three judges, and two rounds of refinement.
In Original, candidate generation and refinement use the fixed problem collection and shared evaluation feedback.
With GAV, both stages can request additional problems and conduct numerical experiments before nominating a program, allowing the judges to use these observations when deciding what to retain or change.

The common selection set in Table~\ref{tab:problem-sets} supplies numerical feedback for each method's selection procedure.
CodePDE and PDE-SHARP retain the valid candidate with the lowest mean error; AutoNumerics uses its final selection agent.
We then evaluate all selected programs on the common test set.

\paragraph{Agent active time.}
We record total agent active time across all development stages.
Table~\ref{tab:pde-agent-development-time} reports these times for the runs in Section~\ref{sec:pde-agent-workflows}.
For all three methods, both GAV and GAV-Full achieve lower test error (Figure~\ref{fig:pde-agent-gains}) with less total agent active time than Original.

\begin{table}[htbp]
\centering
\caption{Total agent active time (hours).}
\label{tab:pde-agent-development-time}

\begin{tabular}{@{}lrrr@{}}
\toprule
Method & Original & GAV & GAV-Full \\
\midrule
CodePDE & 11.77 & 5.59 & 8.52 \\
AutoNumerics & 13.22 & 9.71 & 8.61 \\
PDE-SHARP & 20.80 & 13.44 & 12.36 \\
\bottomrule
\end{tabular}

\end{table}
\FloatBarrier

\subsection{Ablation study details}
\label{app:environment-comparisons}
We describe the implementations of the ablation conditions in Section~\ref{sec:environment-analysis}.
In every condition, agents can request newly generated problems throughout development.
The studies use the same development strategy as the main experiments (Appendix~\ref{app:agent-settings}), with three development rounds and a two-hour budget per agent per round.
GAV is the reference condition: its learned Generator samples continuous reference solutions (Appendix~\ref{app:ns-generator}), MMS constructs their matching problem inputs, and agents can query these inputs at arbitrary coordinates.
Program selection follows Appendix~\ref{app:selection-evaluation}; the problem sets are listed in Table~\ref{tab:problem-sets}.

\paragraph{Without learning from data.}
To examine the role of learning from the target data, we replace GAV's learned Generator with a Gaussian-process prior that is not fitted to the data.
We construct continuous reference solutions using the spatial cosine and temporal Chebyshev bases in Appendix~\ref{app:ns-generator}, drawing independent Gaussian coefficients for the streamfunction and pressure.
We apply the same boundary mask to the streamfunction, enforcing zero velocity on the outer walls and obstacle surfaces.
MMS then constructs each problem's initial velocity and forcing from the sampled reference solution.
Agents can query these inputs at chosen coordinates.

\paragraph{Without MMS construction.}
To test the role of MMS, we use an \emph{input-first} procedure while retaining GAV's learned problem distribution.
We obtain each problem's initial velocity and forcing from the same learned Generator and retain continuous input queries.
We then solve each problem with the incompressible Navier--Stokes solver in Nektar++~\citep{moxey2020nektar} to obtain the Verifier's reference solution for development feedback.
This takes additional computation time and yields an approximate reference subject to spatial and temporal discretization error.

\paragraph{Without continuous input queries.}
This condition tests the value of querying problem inputs beyond a fixed grid.
Problem generation and program development follow the same procedure as in GAV.
For each problem, the initial velocity and forcing available to agents are restricted to values sampled on a fixed $64\times64$ grid.
The restriction applies during development; program selection and final evaluation use the common continuous input interface.

\clearpage

\section{Additional results and case studies}
\label{app:additional-results}

\subsection{Evaluation beyond learned Generators}
\label{app:additional-evaluations}

To check whether program accuracy transfers beyond the learned Generator's problem distribution, we construct additional references directly from numerical solutions or unused observations and use MMS to obtain matching problem inputs.
These problems are used only for evaluation; the programs are frozen before this evaluation and compared on the same problems and scoring points within each task.

\subsubsection{Numerical NS problems}
\label{app:interpolation-test}
To check for overfitting to the learned Generator, we evaluate the final numerical NS programs from the main experiments on a separate set of 96 interpolated problems.

We start from 192 numerical solutions unused during agent development and pair similar solutions within the same flow scenario.
We fit continuous streamfunctions and pressure fields to the sampled solution data so that the derivatives required by MMS can be evaluated.
We interpolate each pair by taking a weighted combination of these fields to obtain a new reference solution.
MMS provides the corresponding forcing.
This construction uses the geometry and continuous representation of the main experiments, with references obtained directly from the numerical data.
This local interpolation keeps the test solutions close to the original data, providing an independent check of accuracy in the same flow regimes.

The final GAV-Full program from Section~\ref{sec:main-comparison} completes all $96$ problems and achieves a mean relative velocity error of $0.1256\%$.
We use the same $128$ fixed scoring points per problem for all programs and average the per-problem errors equally over the $96$ problems.
Both GAV variants improve all three existing PDE agents compared with Original (Table~\ref{tab:pde-agent-additional}).

\begin{table}[htbp]
\centering
\caption{Mean relative velocity error (\%) of the final numerical NS programs from the main experiments on 96 interpolated test problems. Lower is better.}
\label{tab:pde-agent-additional}

\begin{tabular}{@{}llr@{}}
\toprule
Agent setting & Environment & Mean error (\%) \\
\midrule
Multi-turn setting & GAV-Full & 0.1256 \\
\midrule
CodePDE & Original & 3.1231 \\
 & GAV & 0.1890 \\
 & GAV-Full & 1.0555 \\
\midrule
AutoNumerics & Original & 3.4970 \\
 & GAV & 0.3086 \\
 & GAV-Full & 0.9567 \\
\midrule
PDE-SHARP & Original & 3.5203 \\
 & GAV & 0.2744 \\
 & GAV-Full & 0.0738 \\
\bottomrule
\end{tabular}

\end{table}

Table~\ref{tab:environment-interpolation} gives the corresponding results for the environment comparison in Section~\ref{sec:environment-analysis}.
The ordering of the four conditions is unchanged: GAV has the lowest error, while the condition without continuous input queries has an error $13.5$ times that of GAV.
These results indicate that GAV helps agents develop more accurate programs, with gains that transfer to problems constructed independently of the learned Generator.

\begin{table}[htbp]
\centering
\caption{Mean relative velocity error (\%) for the environment comparison on 96 interpolated NS test problems. Lower is better.}
\label{tab:environment-interpolation}

\begin{tabular}{@{}lr@{}}
\toprule
Environment & Mean error (\%) \\
\midrule
GAV & \textbf{0.440} \\
w/o learning from data & 1.561 \\
w/o MMS construction & 0.864 \\
w/o continuous input queries & 5.915 \\
\bottomrule
\end{tabular}

\end{table}
\FloatBarrier

\subsubsection{Kinetic problems}
\label{app:kinetic-independent}
The kinetic evaluation uses eight weak-collision VPFP problems, eight strong-collision VPFP problems, and six ARTEMIS-informed VP problems constructed without the learned Generators.
We evaluate the initial program S0 and the programs retained after each of the six development rounds, using the same problems and hidden scoring coordinates for both coding models within each task.
The initial approximations are those described in Appendix~\ref{app:agent-settings}: static initial density and Maxwellian closure for VPFP, and static initial velocity moments and Poisson field for VP.
In particular, the VPFP initial program uses a $32\times32$ spatial grid, while the VP initial program uses the public 32-point spatial panel and 128-point Gauss--Legendre velocity quadrature.
All scores use Equation~\eqref{eq:kinetic-error} with the frozen calibration from Appendix~\ref{app:selection-evaluation}.

\paragraph{VPFP reference construction.}
Each VPFP reference starts from a separately simulated coarse trajectory, with 65 snapshots over $t\in[0,1]$ on the periodic domain $[0,2\pi)^2$.
The strong-collision set uses a $5\times5$ spatial grid, a $9\times9$ velocity grid, and collision rates in $[10,30]$; the weak-collision set uses $7\times7$, $17\times17$, and $[0.02,0.08]$, respectively.
Both sets use temperature parameters in $[0.8,1.2]$ and velocity nodes in $[-6,6]$ along each axis.
We interpolate the square root of the sampled distribution using cardinal Gaussian radial basis functions in velocity, a full Fourier representation in space, and cubic splines in time.
The Gaussian width is $0.8$ times the velocity-node spacing.
Squaring the interpolant gives a nonnegative distribution; division by its spatially averaged mass fixes the mean density to one.
Velocity moments and the periodic Poisson field are evaluated from this representation, and the VPFP operator applied to it supplies the manufactured source, including the derivative of the mass normalization.
The reference is the normalized continuous interpolant, with Gaussian velocity tails on $\mathbb R^2$.
For each problem, all programs are scored at the same 256 uniformly sampled space--time points.

\paragraph{VP reference construction from unused ARTEMIS profiles.}
The six VP problems use twelve observed velocity profiles, paired within two observation periods: 2011-10-23, 02:00--03:00 UTC, and 2016-07-29, 06:00--07:00 UTC.
Each period supplies three pairs; the selected profiles occupy distinct observation blocks and are excluded from all audited historical Generator training sets.
We apply the ARTEMIS velocity-histogram reduction and normalization to the new profiles.
For each histogram, we convolve the uniform density in each velocity bin with a Gaussian whose standard deviation is one quarter of that bin's width, and sum with the normalized observed bin masses.
This produces a positive, normalized profile on $\mathbb R$, with smoothing redistributing mass across the original velocity bins.

For each pair of smoothed profiles $g_a$ and $g_b$, we manufacture
\begin{align*}
 F(t,x,v)&=\rho(t,x)\bigl[w(t,x)g_a(v)+(1-w(t,x))g_b(v)\bigr],\\
 \rho(t,x)&=1+0.12\cos(x-2\pi t+\phi)+0.04\cos(2x+\pi t+\psi),\\
 w(t,x)&=0.5+0.3\sin(kx-2\pi t+\chi).
\end{align*}
For problems indexed by $i=0,\ldots,5$, the fixed phases are $\phi=2\pi i/6$, $\psi=\pi(i+1)/4$, and $\chi=\pi(i+1)/7$, and the mixing modes are $k=(1,2,4,8,2,4)$.
Thus $\rho\geq0.84$ and $0.2\leq w\leq0.8$ on the periodic domain $x\in[0,2\pi)$.
The zero-mean field satisfies $\partial_xE=\rho-1$, and the manufactured source is $\partial_tF+v\partial_xF+E\partial_vF$.
Derivatives and velocity moments through third order are analytic.
The public velocity-query window remains $[-55/3,55/3]$, while reference moments integrate the full velocity line.
Each program receives only the initial distribution and source queries through the common endpoint $t=0.875$, and is evaluated at the same 256 hidden spatial locations at that endpoint.
The observations determine the velocity profiles; the spatial and temporal variation is manufactured.

\paragraph{Transfer across reference constructions.}
Table~\ref{tab:kinetic-independent-endpoints} and Figure~\ref{fig:kinetic-independent-development} show that every final program has lower mean error than S0 on its independent set.
The R1-to-R6 improvement also transfers in five of the six runs: $31.7\%$ for VP Sol, $58.7\%$ and $42.7\%$ for weak-collision Sol and Astra, and $77.5\%$ and $85.2\%$ for strong-collision Sol and Astra.
For VP Astra, the independent mean increases from $1.522\times10^{-6}$ at R1 to $8.662\times10^{-6}$ at R6, a factor of $5.69$, while remaining below S0.
Table~\ref{tab:kinetic-independent-all-rounds} reports all retained rounds, including intermediate increases and plateaus; R6 is the final program specified by the development protocol.

These small sets test transfer to additional manufactured reference families.
Their shapes are constrained by the coarse VPFP trajectories and the prescribed ARTEMIS spatial mixing modes; their difficulty is not matched to the generated test sets.
The accuracy results therefore concern these forced reference families.
With the same score definition and calibration, weak-collision Sol has a final mean of $0.03437$ here versus $0.1352$ on the generated test set.
Strong-collision Astra achieves a mean of $2.176\times10^{-9}$ on its interpolated-reference set.
Table~\ref{tab:kinetic-independent-macros} reports the five component errors of the final programs.
Heat flux remains the dominant contribution for weak-collision Sol, accounting for $81.3\%$ of its total score.

\begin{table}[t]
\centering
\caption{Mean dimensionless transport error on the independent fixed banks (six ARTEMIS cases and eight cases for each VPFP regime). The signed change is $100(\mathrm{error}_{R6}/\mathrm{error}_{R1}-1)$; positive values indicate higher error. R6 is the prespecified final round.}
\label{tab:kinetic-independent-endpoints}
\begin{tabular}{llrrrr}
\toprule
Family & Model & S0 & R1 & R6 & Change (\%) \\
\midrule
ARTEMIS VP & Sol & 0.1044 & 0.007703 & 0.005259 & -31.73 \\
ARTEMIS VP & Astra & 0.1044 & $1.522\times 10^{-6}$ & $8.662\times 10^{-6}$ & +469.09 \\
Weak VPFP & Sol & 0.5652 & 0.08331 & 0.03437 & -58.74 \\
Weak VPFP & Astra & 0.5652 & 0.3857 & 0.2211 & -42.68 \\
Strong VPFP & Sol & 0.2054 & 0.001883 & $4.233\times 10^{-4}$ & -77.52 \\
Strong VPFP & Astra & 0.2054 & $1.472\times 10^{-8}$ & $2.176\times 10^{-9}$ & -85.22 \\
\bottomrule
\end{tabular}
\end{table}
\begin{table}[t]
\centering
\caption{Mean dimensionless transport error on the independent MMS banks: six ARTEMIS cases and eight cases for each VPFP regime. The same cases and evaluation queries are used for both models in every round. S0 is the initial program; R1--R6 are the retained development programs. All scheduled cases are included, and R6 is the prespecified final round.}
\label{tab:kinetic-independent-all-rounds}
\begin{tabular}{lrrrrrr}
\toprule
 & \multicolumn{2}{c}{ARTEMIS VP} & \multicolumn{2}{c}{Weak VPFP} & \multicolumn{2}{c}{Strong VPFP} \\
\cmidrule(lr){2-3}\cmidrule(lr){4-5}\cmidrule(lr){6-7}
Round & Sol & Astra & Sol & Astra & Sol & Astra \\
\midrule
S0 & 0.1044 & 0.1044 & 0.5652 & 0.5652 & 0.2054 & 0.2054 \\
R1 & 0.007703 & $1.522\times 10^{-6}$ & 0.08331 & 0.3857 & 0.001883 & $1.472\times 10^{-8}$ \\
R2 & 0.007637 & $1.725\times 10^{-6}$ & 0.03311 & 0.2143 & 0.001846 & $2.346\times 10^{-9}$ \\
R3 & 0.005257 & $8.662\times 10^{-6}$ & 0.03297 & 0.2882 & 0.001845 & $2.165\times 10^{-9}$ \\
R4 & 0.005259 & $8.662\times 10^{-6}$ & 0.03458 & 0.2882 & 0.001845 & $2.165\times 10^{-9}$ \\
R5 & 0.005259 & $8.662\times 10^{-6}$ & 0.03437 & 0.2882 & 0.001838 & $2.174\times 10^{-9}$ \\
R6 & 0.005259 & $8.662\times 10^{-6}$ & 0.03437 & 0.2211 & $4.233\times 10^{-4}$ & $2.176\times 10^{-9}$ \\
\bottomrule
\end{tabular}
\end{table}
\begin{table}[htbp]
\centering
\caption{\textbf{Final macroscopic errors on independent kinetic problems.} Mean per-problem normalized RMS error for each field group in Equation~\eqref{eq:kinetic-error}, evaluated for R6 on six ARTEMIS-informed VP problems or eight problems per VPFP regime. These are dimensionless component errors, before weighting: density and electric field each have weight $1/8$, and particle flux, pressure, and heat flux each have weight $1/4$. Vector errors use the Euclidean norm and VPFP pressure uses the full tensor Frobenius norm. Lower is better.}
\label{tab:kinetic-independent-macros}
\begin{tabular}{@{}llrrrrr@{}}
\toprule
Task & Model & $e_\rho$ & $e_{\mathbf E}$ & $e_{\mathbf j}$ & $e_{\mathbf P}$ & $e_{\mathbf q}$ \\
\midrule
ARTEMIS VP & Sol & 0.002574 & 0.01625 & 0.008603 & 0.002552 & $4.708\times10^{-4}$ \\
ARTEMIS VP & Astra & $3.062\times10^{-6}$ & $2.765\times10^{-5}$ & $9.302\times10^{-6}$ & $1.937\times10^{-6}$ & $8.055\times10^{-6}$ \\
Weak VPFP & Sol & 0.001524 & 0.008831 & 0.009259 & 0.01128 & 0.1118 \\
Weak VPFP & Astra & 0.041 & 0.4014 & 0.2016 & 0.09325 & 0.3682 \\
Strong VPFP & Sol & $2.574\times10^{-5}$ & $1.435\times10^{-4}$ & $1.896\times10^{-4}$ & $2.160\times10^{-4}$ & 0.001203 \\
Strong VPFP & Astra & $1.220\times10^{-10}$ & $1.092\times10^{-9}$ & $4.464\times10^{-10}$ & $5.534\times10^{-10}$ & $7.099\times10^{-9}$ \\
\bottomrule
\end{tabular}
\end{table}

\begin{figure}[htbp]
\centering
\includegraphics[width=\linewidth]{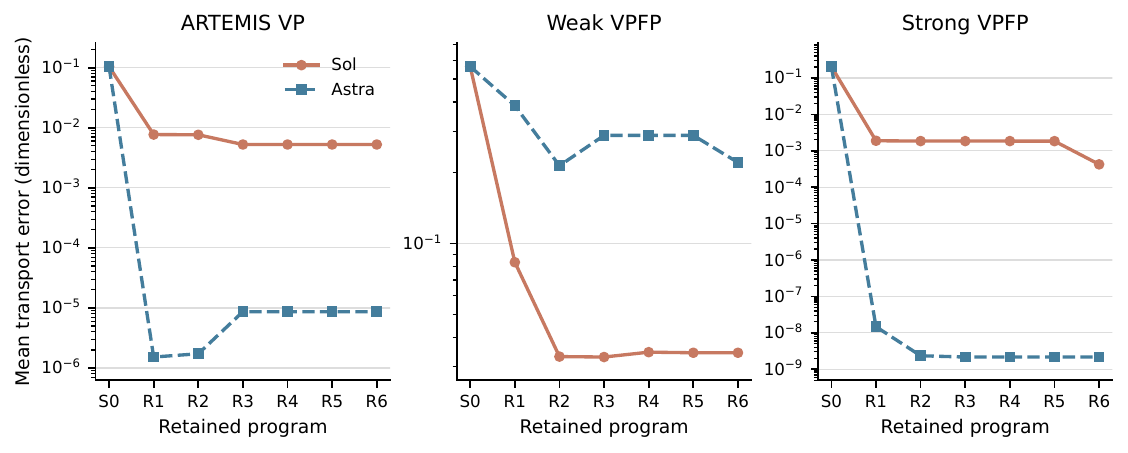}
\caption{\textbf{Kinetic programs on independently constructed problems.} Mean dimensionless transport score from Equation~\eqref{eq:kinetic-error} for the initial program S0 and the programs retained after rounds 1--6. The VP set contains six MMS problems constructed from twelve unused ARTEMIS profiles; each VPFP regime has eight interpolated-reference MMS problems. Programs within each task share the same problems and hidden queries, using the frozen scoring calibration. VP is evaluated at $t=0.875$; VPFP uses space--time queries. All vertical axes are logarithmic, and lower is better. The independent sets are used for evaluation only; R6 remains the final program.}
\label{fig:kinetic-independent-development}
\end{figure}

\FloatBarrier

\subsection{Problem difficulty under guided generation}
\label{app:guided-case}

In a separate five-round numerical NS run, we examine whether guided generation supplies harder problems for the current program.
Each guidance update uses $64$ unguided feedback problems.
At each update, we train the guidance model using error feedback from the current round's program.
Keeping that program fixed, we compare its mean error on the current round's problems with that on new problems generated using the updated guidance.

\begin{figure}[htbp]
\centering
\includegraphics[width=0.75\linewidth]{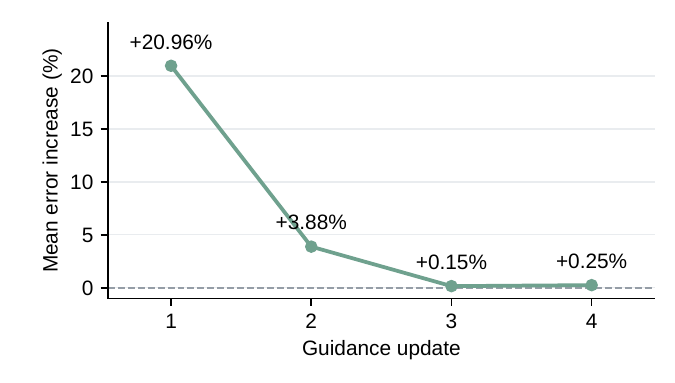}
\caption{\textbf{Percentage increase in mean error after guidance updates.} Each comparison fixes the program and measures the joint relative $\ell_2$ velocity error on the old and new problem sets.}
\label{fig:guided-difficulty}
\end{figure}

The mean error is higher on the new problems at all four updates (Figure~\ref{fig:guided-difficulty}).
These observations show that guided generation can supply more challenging problems for the current program as development proceeds.
\FloatBarrier

\subsection{Program development case studies}
\label{app:development-case-studies}
\label{app:kinetic-program-evolution}

These case studies examine how the Agents use new problems, input queries, and numerical feedback from GAV to decide which revisions to keep.
Figures~\ref{fig:program-evolution} and~\ref{fig:kinetic-evolution-artemis}--\ref{fig:kinetic-evolution-strong} show the Explorer's and Refiner's final submissions, with blue marking the selected programs and arrows identifying their recorded starting programs.
Development comparisons come from archived Agent logs; Figures~\ref{fig:development-rounds} and~\ref{fig:kinetic-development} report the separate fixed-test results.
The three kinetic cases follow the Astra runs, with a refreshed 48-problem selection bank in each of six rounds.

\subsubsection{Numerical NS}
\label{app:program-evolution}
We trace the numerical NS experiment with GAV-Full in Section~\ref{sec:main-comparison}.
Figure~\ref{fig:program-evolution} shows its submitted and retained programs over five rounds.

\begin{figure}[!htbp]
\centering
\includegraphics[width=0.96\linewidth]{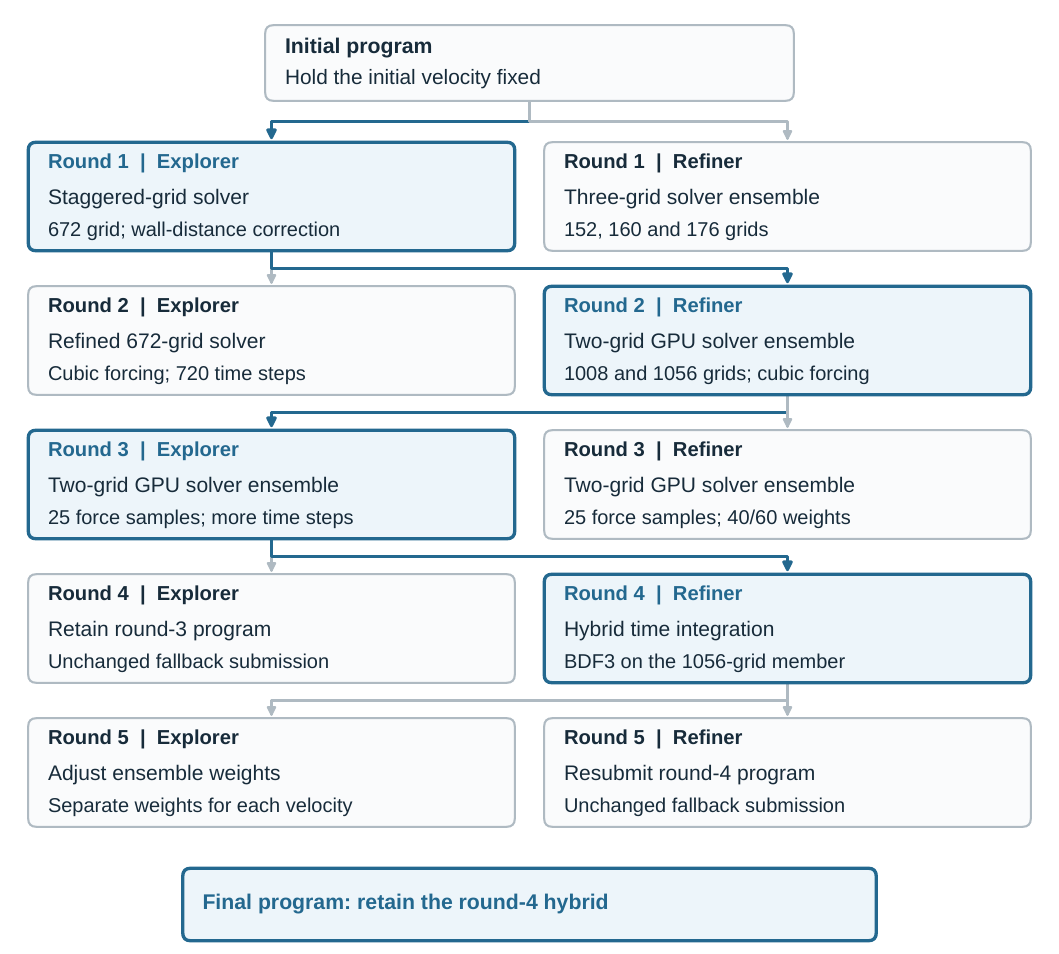}
\caption{\textbf{Numerical NS program evolution over five rounds.} Each round shows the Explorer's and Refiner's final submissions. Blue marks selected submissions; arrows connect each submission to its starting program. Round 5 retains the incumbent from round 4, shown at the bottom. Grid sizes are given per spatial axis.}
\label{fig:program-evolution}
\end{figure}

\paragraph{Choosing spatial resolutions through feedback.}
In round 2, the Refiner compares staggered-grid solvers at different resolutions using the same problems and evaluation points from GAV.
Finer grids do not consistently reduce error: the representation of polygonal obstacles changes with the grid.
The Agent queries initial velocity and forcing on each grid and tests combinations of their predictions.
The selected program averages two GPU solvers on $1008\times1008$ and $1056\times1056$ grids.

\paragraph{Using input queries to reduce computation.}
In round 3, the Explorer probes the forcing at selected times and finds that a degree-$24$ polynomial represents its temporal variation accurately.
It replaces local cubic interpolation from $61$ forcing snapshots with a Chebyshev reconstruction from $25$ queried snapshots.
Comparisons on shared problems show nearly unchanged error with less computation.
The Agent uses this saving to test more time steps, reconstructing forcing in small blocks to limit memory use.
The resulting program is retained after selection.

\paragraph{Testing changes on perturbed problems.}
\label{app:perturbation-case}
In round 4, the Refiner tests a third-order backward differentiation formula (BDF3) for time integration.
Applying it to both grids slightly worsens the mean error on a shared pair of development problems.
Changing only the $1056\times1056$ member improves both velocity-component errors on two additional perturbed problems.
The submitted program retains the previous time integrator on the other grid and is selected for the next round.

\paragraph{Retaining a program after further tests.}
In round 5, the Explorer adjusts the weights used to combine the two grids, choosing separate weights for the two velocity components.
The change improves the mean error on a small development set but does not improve on the current program on the round's selection problems.
The Refiner resubmits the current program, which is retained.
The final solver therefore remains the round-4 combination, and its round-5 test error is unchanged.

\FloatBarrier

\subsubsection{ARTEMIS-informed VP}

\paragraph{Allocating input queries.}
The first selected program uses Fourier--discontinuous-Galerkin transport and corrects density to agree with the Poisson field.
In round 3, the Explorer jointly changes velocity and temporal resolution on the same unguided problem and scoring coordinates.
It replaces 160 six-node velocity cells and 65 source times with 80 eight-node cells and 129 source times, reducing the transport score from $0.002363$ to $0.0007866$.
The Refiner wins the subsequent selection with a different program that combines projected initial data with a 129-time source panel.

\paragraph{Checking spatial source reconstruction.}
In round 5, the Explorer accepts or rejects spatial source reconstruction separately for each velocity column, querying rejected columns directly.
On one shared unguided development problem, the revised sampler reduces error from $0.003964$ to $0.001632$ and improves all five scored quantities.
The final round retains the Refiner's guarded source-time enrichment.
The common-test score stays near $0.002854$ from R3 through R6.
Their independent-reference behavior is reported separately in Appendix~\ref{app:kinetic-independent}.

\begin{figure}[p]
\centering
\includegraphics[width=\linewidth]{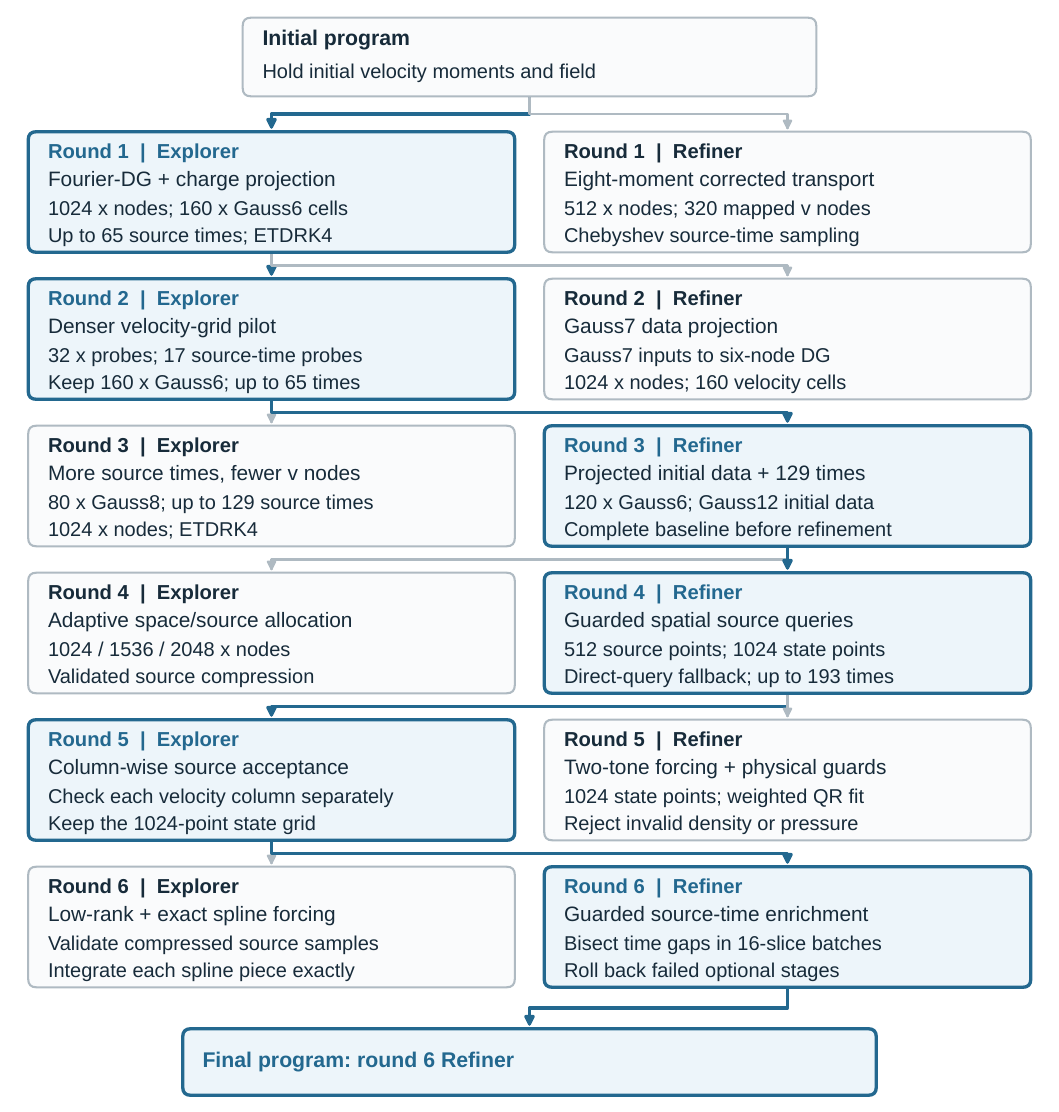}
\caption{\textbf{ARTEMIS-informed VP program evolution with Astra.} Each row shows the Explorer's and Refiner's final nominations; blue identifies the program selected on that round's 48 problems, and arrows identify recorded starting programs. DG denotes discontinuous Galerkin, Gauss$k$ denotes $k$ quadrature nodes per velocity cell, and ETDRK4 is fourth-order exponential Runge--Kutta. Source-time counts are refinement targets; an earlier complete panel may be returned. R6 is the final program; common-test errors are shown in Figure~\ref{fig:kinetic-development}.}
\label{fig:kinetic-evolution-artemis}
\end{figure}

\subsubsection{Weak-collision VPFP}

\paragraph{Revising the macroscopic response.}
Early programs combine initial moments, kinetic charge evolution, and approximate source responses.
In round 4, the Explorer replaces the local Hermite response by an Ornstein--Uhlenbeck response with cubic velocity quadrature and two density-feedback passes.
The kernel describes free collisional transport; the electric-field response remains approximate.
On one shared guided development problem, the nominated program lowers the score from $0.16335$ to $0.03634$, improving all five component errors, and is subsequently selected on the round's full bank.

\paragraph{Diagnosing grids and refining quadrature.}
Round 5 adds midpoint source samples.
A replay diagnostic reveals spatial aliasing across incompatible grids, prompting residual corrections on a common coarser grid whose size divides the original grid size.
In round 6, the Refiner adapts the initial velocity quadrature and blends response quadratures separately along each velocity axis.
On a shared unguided/guided development pair, the nominated program improves all five quantities on both problems and lowers the two-problem mean from $0.08842$ to $0.06504$.
On the common test set, the retained R4, R5, and R6 programs score $0.1109$, $0.1495$, and $0.1274$, respectively.

\begin{figure}[p]
\centering
\includegraphics[width=\linewidth]{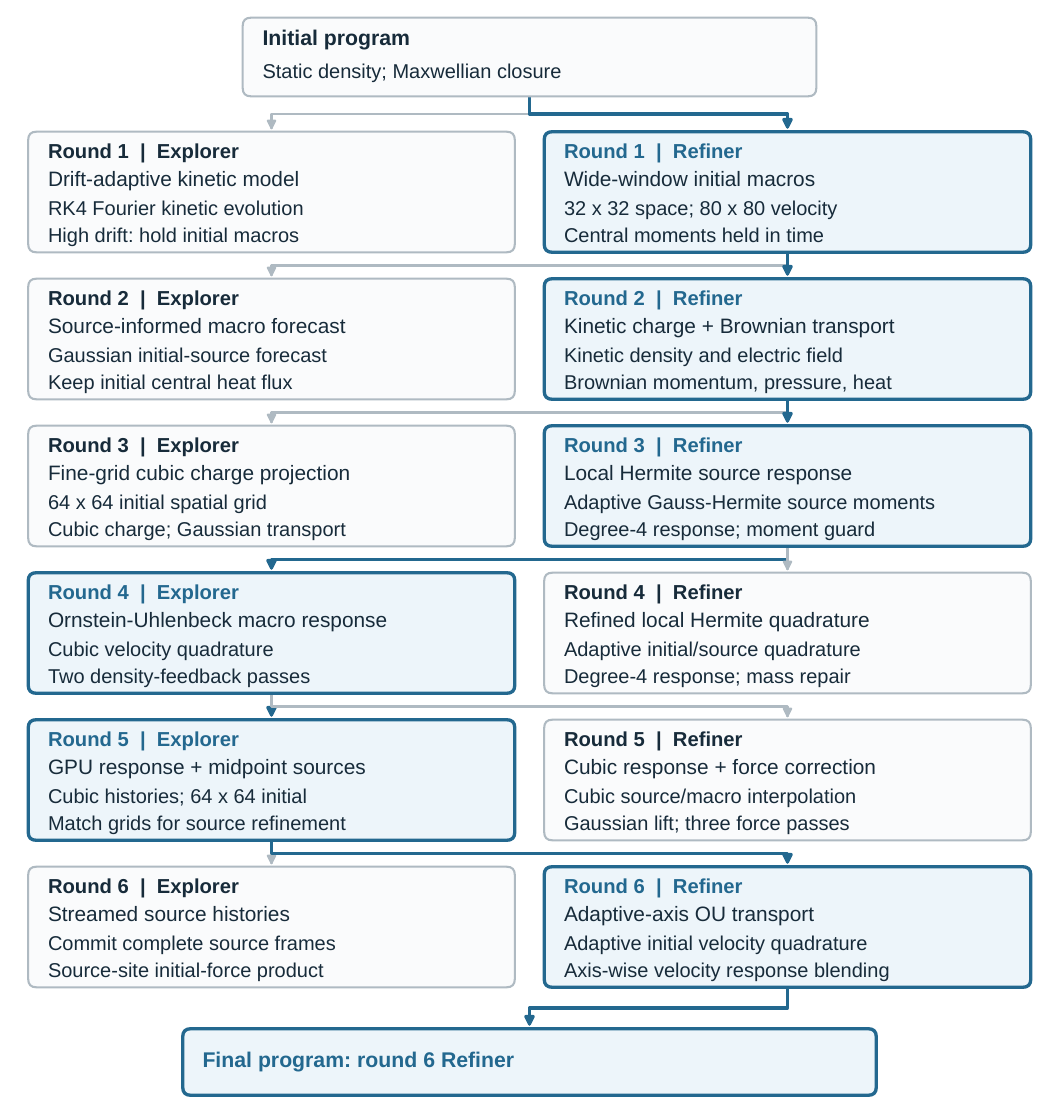}
\caption{\textbf{Weak-collision VPFP program evolution with Astra.} The two columns contain the final nominations from each round, with blue marking the program retained on the common 48-problem selection bank. Arrows show recorded starting programs. OU denotes Ornstein--Uhlenbeck and RK4 denotes fourth-order Runge--Kutta. Grid dimensions count sampling nodes. The programs predict the required macroscopic quantities using different internal approximations. Figure~\ref{fig:kinetic-development} reports their errors on the separate test bank.}
\label{fig:kinetic-evolution-weak}
\end{figure}

\subsubsection{Strong-collision VPFP}

\paragraph{Resolving temporal source structure.}
Public source probes reveal sharp changes in the time derivative and overshoot from global interpolation.
The first-round Refiner therefore fits separate cubic segments in its Fourier--Hermite transport program.
On one shared development problem, a precursor of the nominated program has $3.75$ times smaller error than a linear low-rank source alternative.
Later rounds adapt source resolution and add a low-rank source-density reconstruction checked at held-out times and spatial locations.

\paragraph{Retaining a program after further tests.}
In round 4, both new moment-reconstruction proposals pass eligibility checks but have higher error than the incumbent on the 48 selection problems.
The round-3 program is therefore retained unchanged.

\paragraph{Revising source probes.}
Round 5 selects the Explorer's separable source projection and faster interpolation assembly.
Its common-test mean rises to $3.27\times10^{-3}$, with three perturbed unguided problems contributing $99.9\%$ of the total error.
In round 6, the Refiner changes the geometry of sixteen public-source probes; captured-source diagnostics show better conditioning of the reconstructed density basis.
This program wins the new selection bank and achieves $6.72\times10^{-7}$ on the common test bank.
This test score evaluates the retained program as a whole; the contribution of the probe change was not measured separately.

\begin{figure}[p]
\centering
\includegraphics[width=\linewidth]{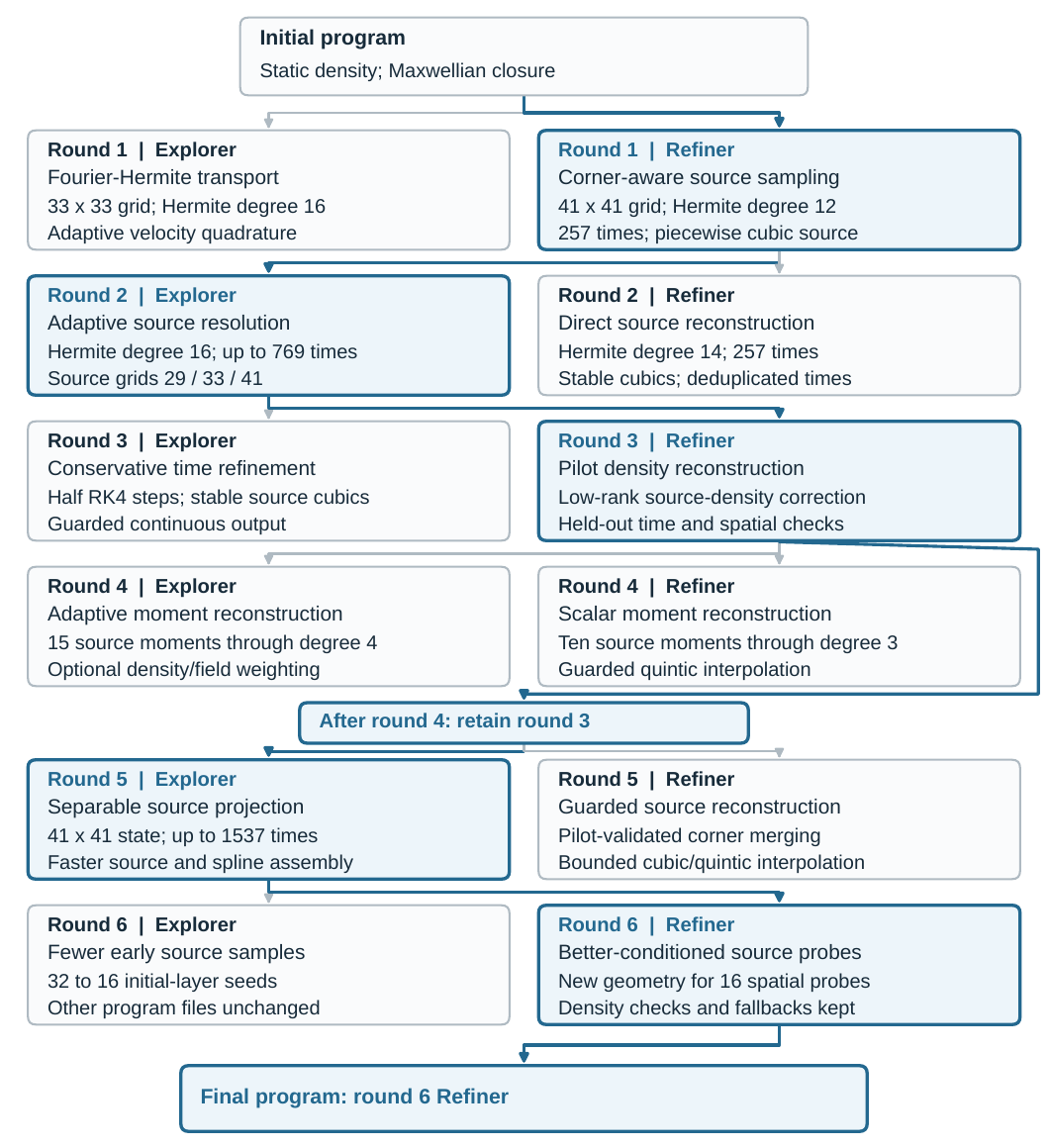}
\caption{\textbf{Strong-collision VPFP program evolution with Astra.} Final Explorer and Refiner nominations over six rounds. Blue marks the program retained on each round's 48 selection problems; arrows indicate recorded starting programs. Round 4 retains the round-3 Refiner program, shown by the bypass arrow and intervening box. Both round-5 candidates therefore start from R3. Grid sizes count nodes per spatial axis, Hermite degree specifies the velocity expansion, and RK4 is fourth-order Runge--Kutta. Test errors, including the R5 increase, are reported separately in Figure~\ref{fig:kinetic-development}.}
\label{fig:kinetic-evolution-strong}
\end{figure}
\FloatBarrier

\subsection{Predicted fields and errors}
\label{app:solution-visualizations}

For the PIV-derived NS and VPFP comparisons below, we use the first unguided problem in each task's test bank, fixed before examining its visualization errors.
The frozen programs within each task share the same problem and manufactured reference.
These diagnostics illustrate spatial and temporal behavior; Figures~\ref{fig:development-rounds} and~\ref{fig:kinetic-development} report performance over the full generated test banks.

\subsubsection{PIV-derived NS}
We visualize the programs retained after each development round on the fixed PIV example.
The continuous manufactured references and saved predictions are shared throughout each comparison; plotting adds no interpolation or smoothing.

Relative $\ell_2$ curves reduce over the spatial grid at each time; full-window values reduce over all saved space--time samples.
These diagnostics use unweighted sums without a denominator floor.
Zero error against a zero reference is assigned zero; nonzero error against a zero reference is undefined.
Positive errors use logarithmic axes.

\paragraph{Fields and diagnostics.}
The example covers $[0,1.55]$\,s at 193 times on a $96\times64$ spatial grid.
Its reference solution is sampled from the Generator fitted to the PIV measurements.
The rectangular wake region contains no solid obstacle; its outer boundary velocities follow the manufactured reference, and velocity diagnostics include these boundary points.
We compute vorticity, $\omega=\partial_xu_2-\partial_yu_1$, by applying the same finite-difference operator to reference and predicted velocities on complete interior stencils.
The vector velocity diagnostic combines both components in one norm, whereas Equation~\ref{eq:ns-velocity-error} averages their separate relative errors.

Figure~\ref{fig:piv-error-time} retains all 193 times, including brief R5 error peaks near both ends of the window.
R2 retains R1 and produces identical predictions.
R4 has a smaller full-window vector relative $\ell_2$ error than R5 on this problem ($0.00141\%$ versus $0.00172\%$).
Figures~\ref{fig:piv-reference-fields} and~\ref{fig:piv-round-errors} show the corresponding fields and round-by-round spatial errors.

\begin{figure}[htbp]
\centering
\includegraphics[width=\linewidth]{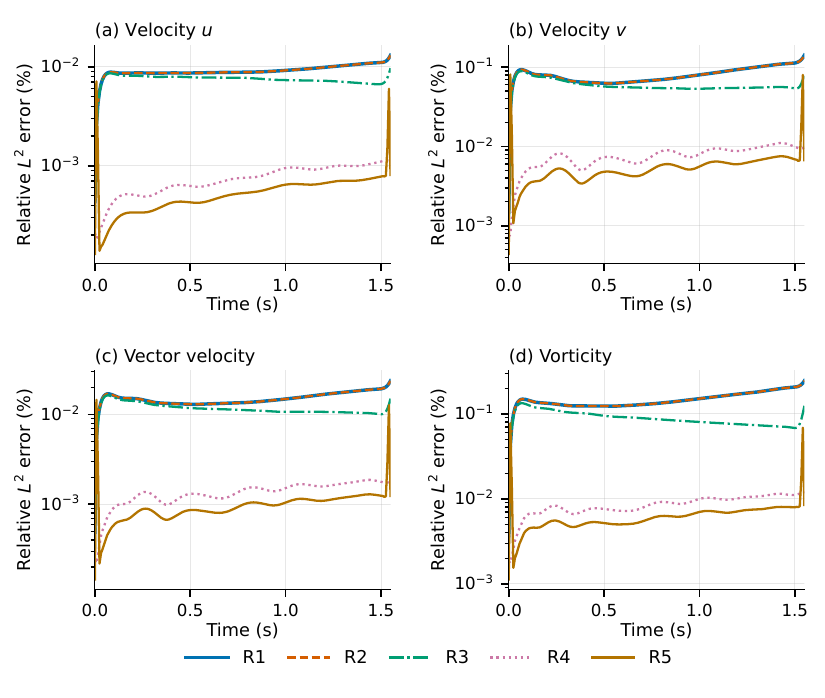}
\caption{\textbf{PIV errors over the complete time window.} Spatial relative $\ell_2$ errors (\%) for each velocity component, the joint velocity vector, and vorticity at all 193 query times. R2 coincides with R1. These are single-problem grid diagnostics.}
\label{fig:piv-error-time}
\end{figure}

\begin{figure}[p]
\centering
\includegraphics[width=0.90\linewidth]{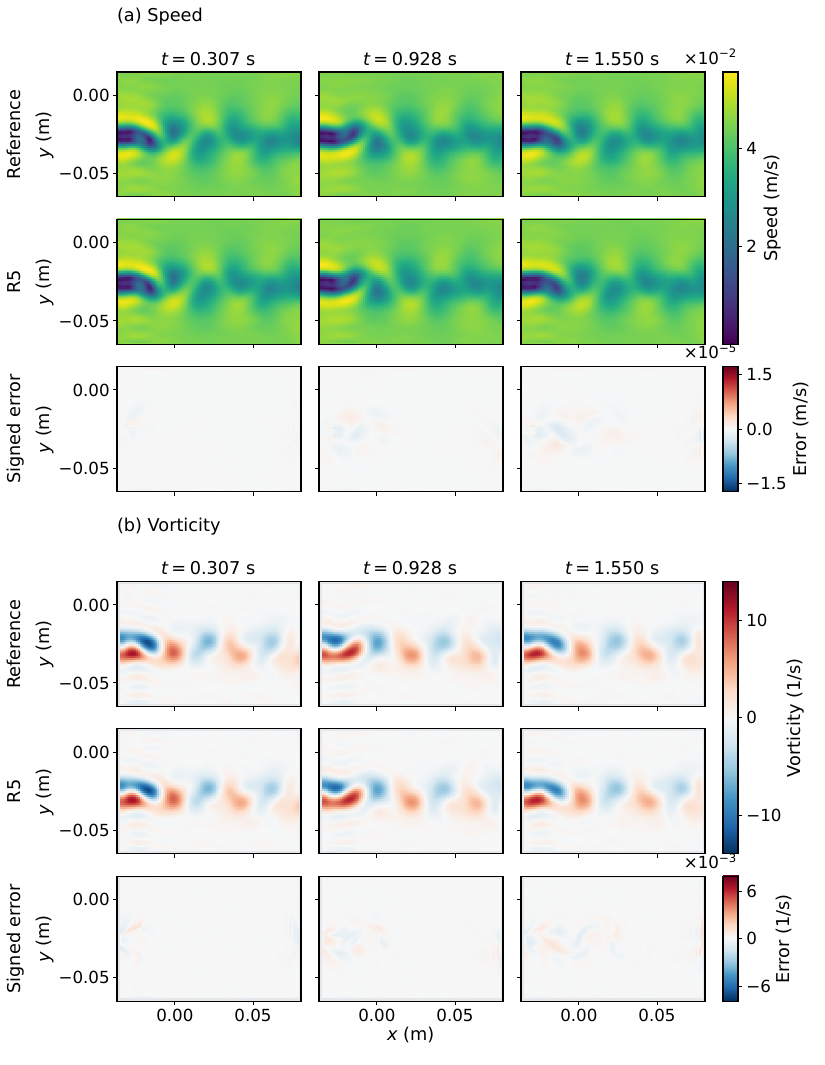}
\caption{\textbf{Final-program fields on the PIV-derived example.} Reference, R5 prediction, and signed error (prediction minus reference) for speed and vorticity at $t=0.307$, $0.928$, and $1.55$\,s. Each field uses one scale for reference and prediction and a separate symmetric error scale, fixed over the full time window. Gray vorticity pixels lie outside the common interior-stencil support.}
\label{fig:piv-reference-fields}
\end{figure}

\begin{figure}[p]
\centering
\includegraphics[width=\linewidth]{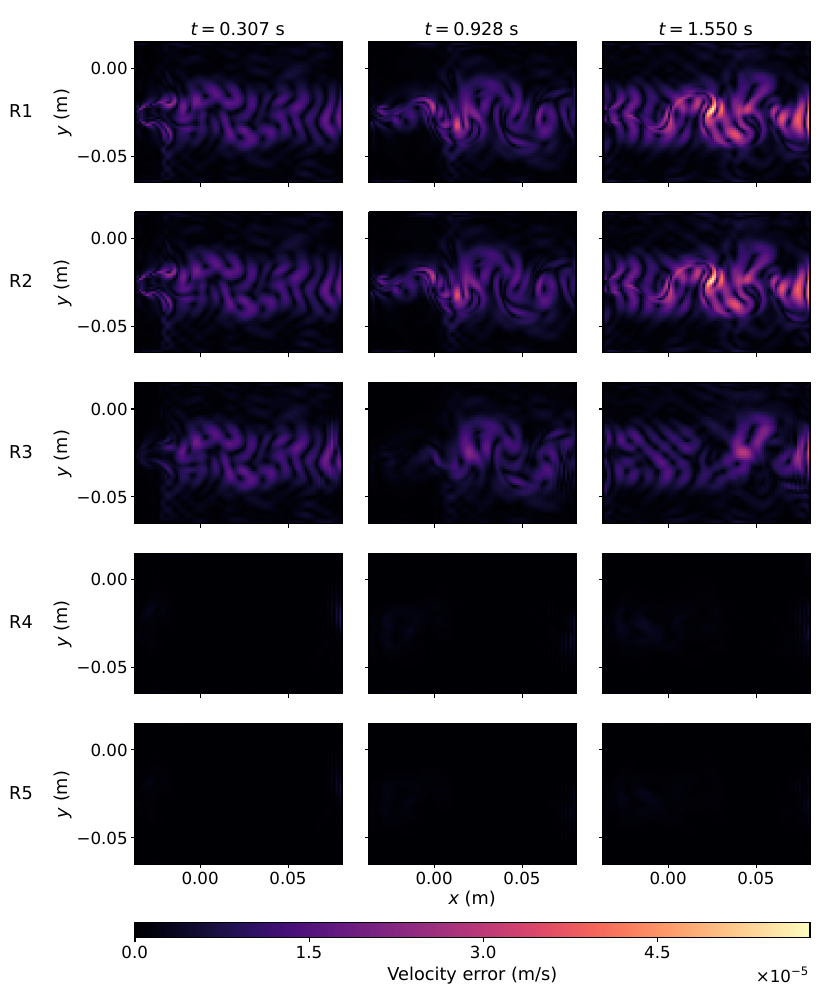}
\caption{\textbf{Velocity error across PIV development rounds.} Pointwise vector error in m/s for the retained programs R1--R5 on the same manufactured problem at three times. All panels share one linear scale determined over all five programs and all 193 times. R2 repeats R1 because its program is unchanged.}
\label{fig:piv-round-errors}
\end{figure}
\FloatBarrier

\subsubsection{VPFP}
\label{app:kinetic-visualizations}

The strong- and weak-collision examples have different references.
We show dense-field evaluations for strong-collision Sol and Astra and weak-collision Sol on generated VPFP problems.
The error panels report pointwise differences in the nondimensional model units, rather than the normalized transport score used in the aggregate tables.

Each frozen program is prepared once and queried on a $64\times64$ spatial grid at nine times in $[0,1]$.
The figures display the five times $0$, $0.25$, $0.5$, $0.75$, and $1$.
Figure~\ref{fig:kinetic-strong-reference} shows all five reference macroscopic quantities in the strong-collision example.
Figure~\ref{fig:kinetic-strong-errors} compares R1 and R6 through Sol's density error and Astra's central-pressure error, retaining one linear color scale for each comparison.
Figure~\ref{fig:kinetic-weak-sol-errors} shows the weak-collision heat-flux errors for every retained round.
The common scales expose both reductions and residual spatial structure.

\begin{figure}[p]
\centering
\includegraphics[width=\linewidth]{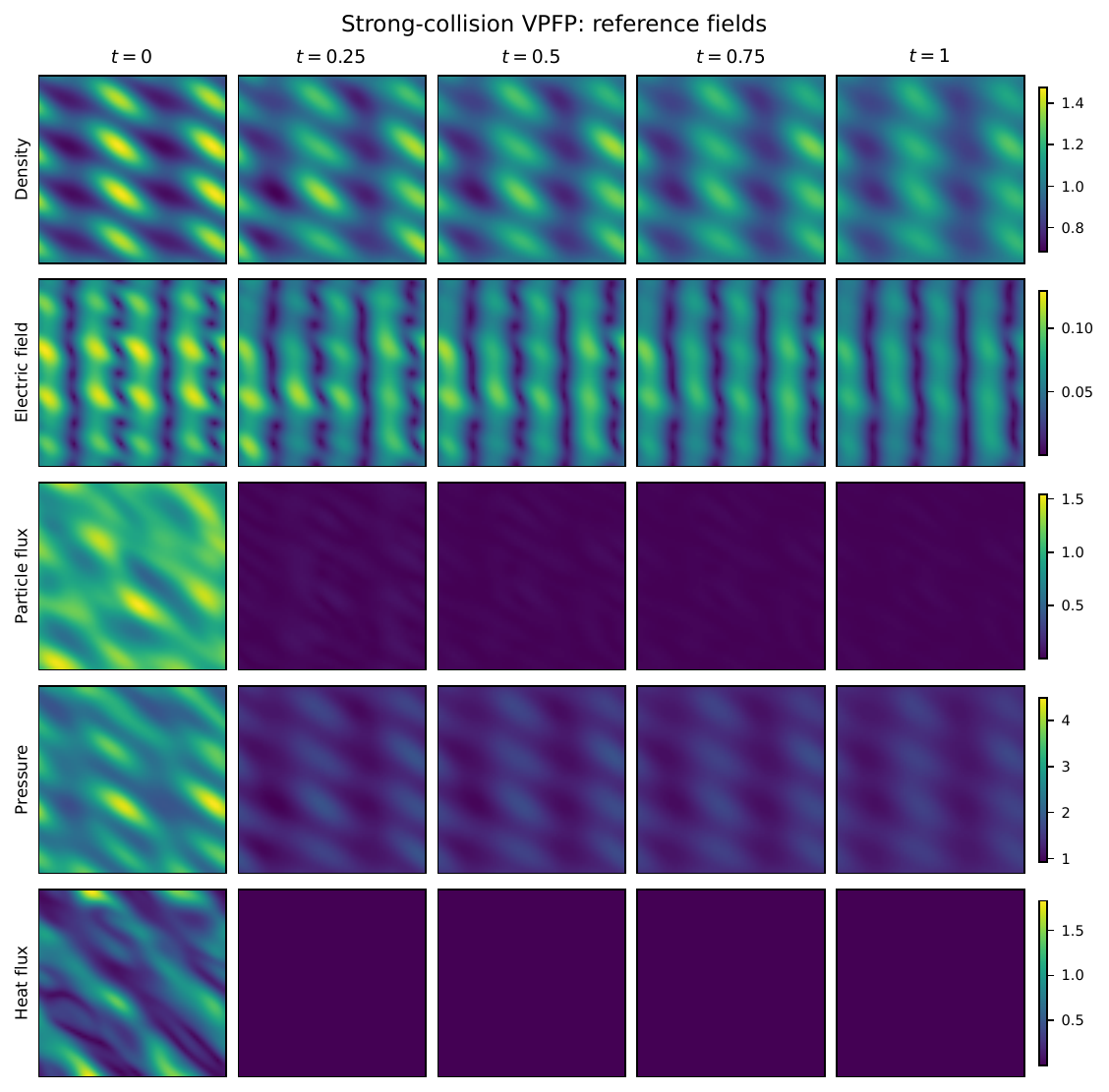}
\caption{\textbf{Reference macroscopic fields for the strong-collision VPFP example.} The first unguided test problem, sampled on a $64\times64$ grid over $[0,2\pi)^2$ at the five displayed times. Horizontal and vertical directions are $x$ and $y$. Rows show density, electric-field magnitude, particle-flux magnitude, central-pressure Frobenius norm, and central-heat-flux magnitude, in nondimensional model units. Each row uses a common linear scale across times. These are manufactured reference fields.}
\label{fig:kinetic-strong-reference}
\end{figure}

\begin{figure}[p]
\centering
\includegraphics[width=\linewidth]{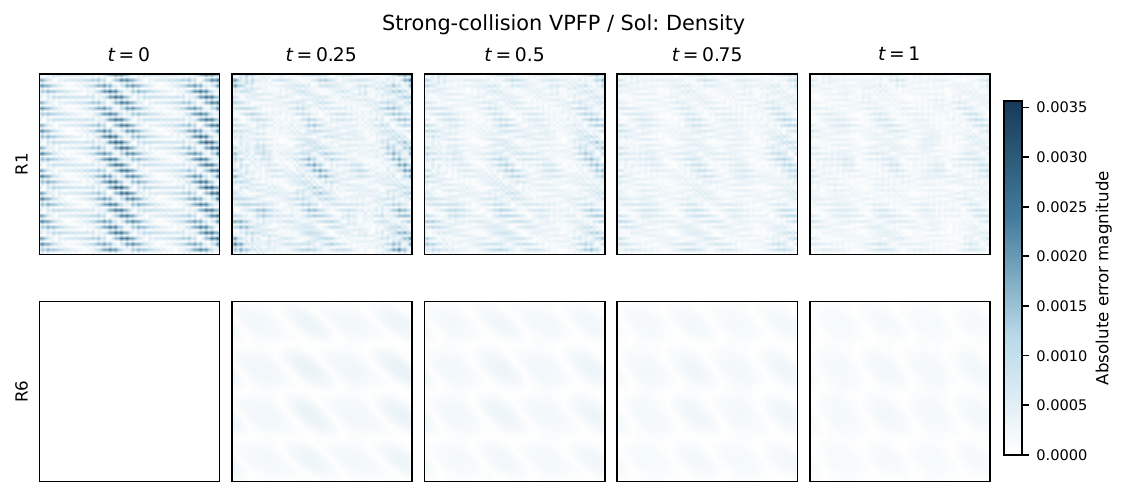}
\includegraphics[width=\linewidth]{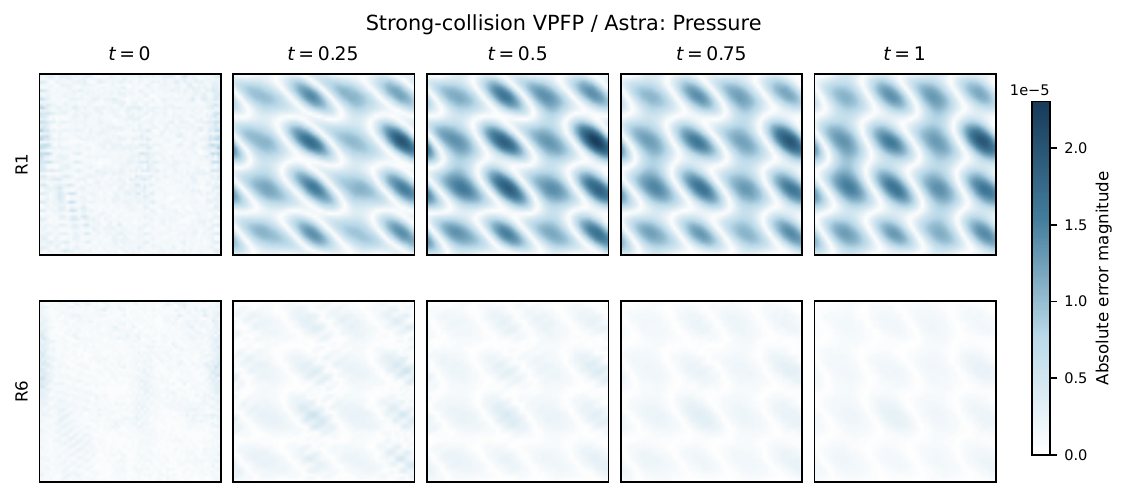}
\caption{\textbf{R1 and R6 errors on the strong-collision VPFP example.} Top: Sol density error $|\widehat\rho-\rho|$. Bottom: Astra central-pressure error $\|\widehat{\mathbf P}-\mathbf P\|_{\mathrm F}$, counting both off-diagonal tensor entries. Within each block, rows are R1 and R6 and columns are $t=0,0.25,0.5,0.75,1$. Each image covers $[0,2\pi)^2$ on a $64\times64$ grid, with $x$ horizontal and $y$ vertical. Each block has its own linear scale, fixed over both programs and all nine saved times. The errors are in nondimensional model units; the two blocks have different quantities and scales.}
\label{fig:kinetic-strong-errors}
\end{figure}

\begin{figure}[p]
\centering
\includegraphics[width=\linewidth]{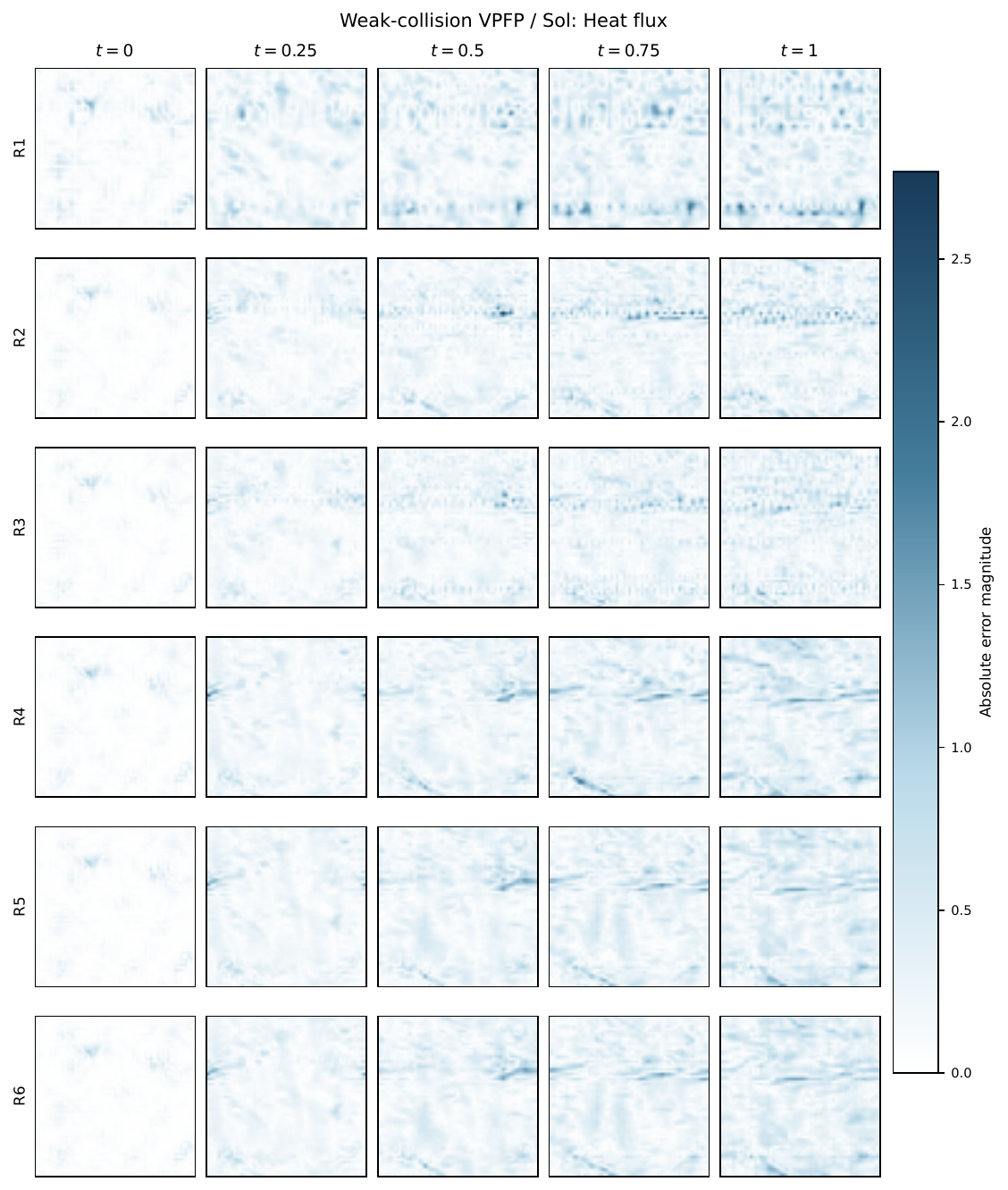}
\caption{\textbf{Weak-collision VPFP heat-flux errors across development rounds.} Sol's R1--R6 programs on the first unguided test problem. Panels show central-heat-flux error $\|\widehat{\mathbf q}-\mathbf q\|_2$ at $t=0,0.25,0.5,0.75,1$ on a $64\times64$ grid over $[0,2\pi)^2$ ($x$ horizontal, $y$ vertical). One linear color scale is shared by all rounds and nine saved times. Errors use nondimensional model units.}
\label{fig:kinetic-weak-sol-errors}
\end{figure}

\FloatBarrier

\section{Proofs}
\label{app:theory}

\subsection{\texorpdfstring{Proof of~\Cref{prop:relative-evaluation}}{Proof of Proposition \getrefnumber{prop:relative-evaluation}}}
\label{app:evaluation-theory}

\begin{proof}
Fix the learned problem distribution $\widehat P$ and a program $A$ chosen without using the test set.
All probabilities below are with respect to the $M$ independent test problems from $\widehat P$.
By \Cref{ass:evaluation}(ii), the expected loss under $\widehat P$ satisfies $\mu_{\widehat P}=\E_{c\sim\widehat P}[\ell_A(c)]\in(0,\infty)$.
Write $X(c)=\ell_A(c)/\mu_{\widehat P}-1$, so that $\E_{\widehat P}X=0$ and $X\geq-1$.
We first consider $\nu>0$ and treat $\nu=0$ at the end.
Throughout the proof, $1/\kappa=\infty$ when $\kappa=0$.

We first compare $\mu_P$ and $\mu_{\widehat P}$.
By the entropy inequality~\citep[Lemma~2.1]{baldasso2022large}, for any random variable $Z$ with $\E_P|Z|<\infty$ and $\E_{\widehat P}e^Z<\infty$, we have
\begin{equation}
 \E_P Z\leq D_{\mathrm{KL}}(P\|\widehat P)+\log\E_{\widehat P}e^Z.
 \label{eq:evaluation-entropy}
\end{equation}
We begin by verifying that $\mu_P$ is finite.
Fix $0<s<1/\kappa$.
For each integer $n\geq1$, apply~\eqref{eq:evaluation-entropy} to the bounded truncation $Z=s\min\{X+1,n\}$.
Since $\min\{X+1,n\}\leq X+1$, \Cref{ass:evaluation} gives
\[
 s\E_P\min\{X+1,n\}
 \leq \varepsilon+\log\E_{\widehat P}e^{s(X+1)}
 \leq \varepsilon+s+\frac{\nu^2s^2}{2}.
\]
Since $0\leq\min\{X+1,n\}\uparrow X+1$, monotone convergence gives $\E_P(X+1)\leq\varepsilon/s+1+\nu^2s/2<\infty$.
The bound $|X|\leq(X+1)+1$ then implies $\E_P|X|<\infty$.
Since $\ell_A=\mu_{\widehat P}(X+1)$ and $\mu_{\widehat P}<\infty$ by \Cref{ass:evaluation}(ii), we also have $\mu_P=\mu_{\widehat P}\E_P(X+1)<\infty$.

For $0<s<1/\kappa$, both $sX$ and $-sX$ are now integrable under $P$, and~\eqref{eq:normalized-tail} bounds their exponential moments under $\widehat P$.
Applying~\eqref{eq:evaluation-entropy} to these two variables gives
\[
 -\left(\varepsilon+\frac{\nu^2s^2}{2}\right)
 \leq s\E_P X
 \leq \varepsilon+\frac{\nu^2s^2}{2}.
\]
Dividing by $s>0$ and using $\E_P X=\mu_P/\mu_{\widehat P}-1$, we obtain
\[
 \left|\frac{\mu_P}{\mu_{\widehat P}}-1\right|
 =|\E_P X|
 \leq\frac{\varepsilon}{s}+\frac{\nu^2s}{2},
 \qquad 0<s<1/\kappa.
\]
If $\varepsilon=0$, letting $s\downarrow0$ gives $\mu_P=\mu_{\widehat P}$.
If $\varepsilon>0$, choose
\[
 s=\frac{\sqrt{2\varepsilon}}{\nu+\kappa\sqrt{2\varepsilon}}.
\]
This choice satisfies $0<s<1/\kappa$ and $s\leq\sqrt{2\varepsilon}/\nu$.
Substituting into the preceding bound gives
\[
 |\E_P X|
 \leq\nu\sqrt{\frac{\varepsilon}{2}}+\kappa\varepsilon
      +\nu\sqrt{\frac{\varepsilon}{2}}
 =a_\varepsilon.
\]
Thus, in both cases,
\begin{equation}
 |\mu_P-\mu_{\widehat P}|\leq a_\varepsilon\mu_{\widehat P}.
 \label{eq:evaluation-mean-shift}
\end{equation}

Next, we bound the deviation of the empirical mean loss $\widehat\mu_M$ from the expected loss $\mu_{\widehat P}$.
Let $\overline X_M=M^{-1}\sum_{j=1}^M X(\widetilde c_j)=\widehat\mu_M/\mu_{\widehat P}-1$.
For $|s|<M/\kappa$, we have $|s/M|<1/\kappa$, so independence of the test problems and~\eqref{eq:normalized-tail} give
\[
 \log\E e^{s\overline X_M}
 =\sum_{j=1}^M\log\E_{\widehat P}e^{(s/M)X}
 \leq\frac{\nu^2s^2}{2M}.
\]
Set $x=\log(2/\delta)$.
For any $0<s<M/\kappa$, applying Markov's inequality to $e^{s\overline X_M}$ yields
\[
 \mathbb{P}\left(\overline X_M>\frac{x}{s}+\frac{\nu^2s}{2M}\right)
 \leq e^{-x-\nu^2s^2/(2M)}\E e^{s\overline X_M}
 \leq e^{-x}.
\]
Replacing $s$ by $-s$ in the bound for $\log\E e^{s\overline X_M}$ gives $\log\E e^{-s\overline X_M}\leq\nu^2s^2/(2M)$.
Applying Markov's inequality to $e^{-s\overline X_M}$ therefore gives the same probability bound for $-\overline X_M$.
Choose
\[
 s=\frac{\sqrt{2Mx}}{\nu+\kappa\sqrt{2x/M}}.
\]
Then $0<s<M/\kappa$ and $s\leq\sqrt{2Mx}/\nu$, so
\[
 \frac{x}{s}+\frac{\nu^2s}{2M}
 \leq\nu\sqrt{\frac{x}{2M}}+\frac{\kappa x}{M}
      +\nu\sqrt{\frac{x}{2M}}
 =b_M.
\]
Thus the threshold in each tail bound is at most $b_M$.
Taking a union bound, we obtain
\[
 \mathbb{P}(|\overline X_M|>b_M)
 \leq \mathbb{P}(\overline X_M>b_M)+\mathbb{P}(-\overline X_M>b_M)
 \leq 2e^{-x}=\delta.
\]
Equivalently, with probability at least $1-\delta$,
\begin{equation}
 |\widehat\mu_M-\mu_{\widehat P}|\leq b_M\mu_{\widehat P}.
 \label{eq:evaluation-test-deviation}
\end{equation}

Combining~\eqref{eq:evaluation-mean-shift} and~\eqref{eq:evaluation-test-deviation} by the triangle inequality, we obtain, with probability at least $1-\delta$,
\[
 |\widehat\mu_M-\mu_P|
 \leq |\widehat\mu_M-\mu_{\widehat P}|+|\mu_{\widehat P}-\mu_P|
 \leq\mu_{\widehat P}(a_\varepsilon+b_M).
\]
Moreover, \eqref{eq:evaluation-mean-shift} and $a_\varepsilon<1$ give
\[
 \mu_P\geq(1-a_\varepsilon)\mu_{\widehat P}>0.
\]
Dividing the preceding error bound by $\mu_P$ yields
\[
 \frac{|\widehat\mu_M-\mu_P|}{\mu_P}
 \leq\frac{a_\varepsilon+b_M}{1-a_\varepsilon},
\]
as claimed.

Finally, suppose $\nu=0$.
For any $0<s<1/\kappa$, Jensen's inequality and~\eqref{eq:normalized-tail} give
\[
 1=e^{s\E_{\widehat P}X}
 \leq\E_{\widehat P}e^{sX}
 \leq e^0=1.
\]
Since the exponential function is strictly convex, equality in Jensen's inequality implies that $X$ is constant almost surely under $\widehat P$.
The identity $\E_{\widehat P}X=0$ shows that this constant is zero.
The finite KL divergence gives $P\ll\widehat P$, so $X=0$ also holds almost surely under $P$.
Consequently, $\widehat\mu_M=\mu_P=\mu_{\widehat P}>0$ with probability one, which proves the bound in this case and completes the proof.
\end{proof}

\subsection{\texorpdfstring{Proof of~\Cref{thm:guided-convergence}}{Proof of Theorem \getrefnumber{thm:guided-convergence}}}
\label{app:convergence}

The following lemma shows that the guided loss agrees with the robust objective at the current parameter and controls their difference after a parameter change.
It supplies the quadratic bound used in the convergence proof.

\begin{lemma}
\label{lem:guidance-model}
Suppose \Cref{ass:convex-guidance}(ii) holds, and let $\bar\beta\geq0$ be finite.
For $\rho>0$, let $R_t$ be the exponential tilt in~\eqref{eq:exact-guidance-tilt}, calibrated to $D_{\mathrm{KL}}(R_t\|\widehat P)=\rho$, with $0<\beta_t\leq\bar\beta$ for every $t\geq0$.
For $\rho=0$, take $R_t=\widehat P$.
Then, for every $t\geq0$ and $\theta\in\Theta$,
\begin{equation}
 0\leq F(\theta)-f_t(\theta)
 \leq\frac{\bar\beta G^2}{2}\|\theta-\theta_t\|^2.
 \label{eq:guidance-model}
\end{equation}
In particular, $f_t(\theta_t)=F(\theta_t)$.
\end{lemma}
\noindent See~\Cref{app:guidance-model-proof} for the proof.

\begin{proof}
We first verify that the minima in~\eqref{eq:stabilized-update} and in the definition of $F^\star$ are attained.
For any distribution $R$ with $D_{\mathrm{KL}}(R\|\widehat P)\leq\rho$, finite KL divergence implies $R\ll\widehat P$.
Thus the boundedness and Lipschitz conditions in~\Cref{ass:convex-guidance}(ii) also hold $R$-almost surely.
The uniform bound on the losses makes $F$ and every $f_t$ finite.
Moreover, for any $\theta,\theta'\in\Theta$,
\[
 \bigl|\E_R\ell_\theta-\E_R\ell_{\theta'}\bigr|
 \leq\E_R|\ell_\theta-\ell_{\theta'}|
 \leq G\|\theta-\theta'\|.
\]
Taking $R=R_t$ shows that each $f_t$ is $G$-Lipschitz.
The same inequality also gives $\E_R\ell_\theta\leq F(\theta')+G\|\theta-\theta'\|$.
Taking the supremum over $R$ and then interchanging $\theta$ and $\theta'$, we obtain
\[
 |F(\theta)-F(\theta')|\leq G\|\theta-\theta'\|.
\]
Consequently, $F$ and the objective in~\eqref{eq:stabilized-update} are continuous on the nonempty compact set $\Theta$, so their minima are attained.
Fix $\theta^\star\in\Theta$ such that $F(\theta^\star)=F^\star$.

We next prove that $F(\theta_{t+1})\leq F(\theta_t)$ for every $t\geq0$.
Fix $t\geq0$ and set $C=\bar\beta G^2$.
By the definition of idealized guidance, $D_{\mathrm{KL}}(R_t\|\widehat P)\leq\rho$, so $f_t(\theta)\leq F(\theta)$ for all $\theta\in\Theta$.
For any such $\theta$, we have
\begin{align}
 F(\theta_{t+1})+\frac{\tau-C}{2}\|\theta_{t+1}-\theta_t\|^2
 &\leq f_t(\theta_{t+1})+\frac\tau2\|\theta_{t+1}-\theta_t\|^2 \notag\\
 &\leq f_t(\theta)+\frac\tau2\|\theta-\theta_t\|^2 \notag\\
 &\leq F(\theta)+\frac\tau2\|\theta-\theta_t\|^2.
 \label{eq:majorization-step}
\end{align}
The first inequality follows from \Cref{lem:guidance-model} at $\theta_{t+1}$, and the second follows from the minimizing property in~\eqref{eq:stabilized-update}.
Taking $\theta=\theta_t$ in~\eqref{eq:majorization-step} and using $\tau>C$, we obtain
\[
 F(\theta_{t+1})
 \leq F(\theta_t)-\frac{\tau-C}{2}\|\theta_{t+1}-\theta_t\|^2
 \leq F(\theta_t).
\]
This proves the first claim.

It remains to prove~\eqref{eq:convex-guided-rate}.
Write $D=\operatorname{diam}(\Theta)$ and $\Delta_t=F(\theta_t)-F^\star\geq0$.
For any $\alpha\in[0,1]$, convexity of $\Theta$ gives $(1-\alpha)\theta_t+\alpha\theta^\star\in\Theta$.
Substituting this parameter into~\eqref{eq:majorization-step} and using $\tau>C$, we obtain
\begin{align*}
 F(\theta_{t+1})
 &\leq F\bigl((1-\alpha)\theta_t+\alpha\theta^\star\bigr)
       +\frac{\tau\alpha^2}{2}\|\theta^\star-\theta_t\|^2\\
 &\leq(1-\alpha)F(\theta_t)+\alpha F^\star
       +\frac{\tau\alpha^2D^2}{2},
\end{align*}
where the last inequality follows from convexity of $F$ and $\|\theta^\star-\theta_t\|\leq D$.
Subtracting $F^\star$ from both sides gives
\begin{equation}
 \Delta_{t+1}\leq(1-\alpha)\Delta_t+\frac{\tau\alpha^2D^2}{2}.
 \label{eq:convex-comparison}
\end{equation}
For each $t\geq0$, choose $\alpha=2/(t+2)\in(0,1]$.
Then~\eqref{eq:convex-comparison} becomes
\[
 \Delta_{t+1}\leq\frac{t}{t+2}\Delta_t+\frac{2\tau D^2}{(t+2)^2}.
\]
Multiplying by $(t+1)(t+2)$ and rearranging yields
\[
 (t+1)(t+2)\Delta_{t+1}-t(t+1)\Delta_t
 \leq 2\tau D^2\frac{t+1}{t+2}.
\]
Summing over $t=0,\ldots,T-1$, the intermediate terms on the left cancel and the coefficient of $\Delta_0$ is zero.
Hence
\begin{align*}
 T(T+1)\Delta_T
 &=\sum_{t=0}^{T-1}\bigl((t+1)(t+2)\Delta_{t+1}-t(t+1)\Delta_t\bigr)\\
 &\leq 2\tau D^2\sum_{t=0}^{T-1}\frac{t+1}{t+2}
 \leq 2\tau D^2T,
\end{align*}
where the last inequality uses $(t+1)/(t+2)\leq1$.
Since $T\geq1$, dividing by $T(T+1)$ and recalling the definitions of $D$ and $\Delta_T$ proves~\eqref{eq:convex-guided-rate}.
\end{proof}

\subsubsection{\texorpdfstring{Proof of~\Cref{lem:guidance-model}}{Proof of Lemma \getrefnumber{lem:guidance-model}}}
\label{app:guidance-model-proof}

\begin{proof}
If $\rho=0$, the only feasible distribution is $\widehat P=R_t$, so $F=f_t$ and the claim holds.
Suppose $\rho>0$, and fix $t\geq0$ and $\theta\in\Theta$.
Write $\beta=\beta_t$ and $d(c)=\ell_\theta(c)-\ell_{\theta_t}(c)$.
The exponential tilt in~\eqref{eq:exact-guidance-tilt} is equivalent to $\widehat P$, and its log density is bounded by the boundedness of the losses.
For any feasible distribution $R$, changing the reference measure in the KL divergence therefore gives
\begin{align*}
 D_{\mathrm{KL}}(R\|R_t)
 &=D_{\mathrm{KL}}(R\|\widehat P)
   -\beta\E_R\ell_{\theta_t}
   +\log\E_{\widehat P}e^{\beta\ell_{\theta_t}}\\
 &=D_{\mathrm{KL}}(R\|\widehat P)-\rho
   -\beta\bigl(\E_R\ell_{\theta_t}-\E_{R_t}\ell_{\theta_t}\bigr),
\end{align*}
where the second equality uses $D_{\mathrm{KL}}(R_t\|\widehat P)=\rho$.
The first equality also shows that $R_t$ uniquely maximizes the KL-regularized objective: its value at any feasible $R$ is $\beta^{-1}\log\E_{\widehat P}e^{\beta\ell_{\theta_t}}-\beta^{-1}D_{\mathrm{KL}}(R\|R_t)$.
This identity holds for every $R$ with finite $D_{\mathrm{KL}}(R\|\widehat P)$, and the second equality shows that calibration makes $R_t$ a constrained maximizer as well.
Since $D_{\mathrm{KL}}(R\|\widehat P)\leq\rho$, it follows that
\[
 \E_R\ell_\theta-\E_{R_t}\ell_\theta
 \leq \E_Rd-\E_{R_t}d-\frac1\beta D_{\mathrm{KL}}(R\|R_t)
 \leq \frac1\beta\log\E_{R_t}e^{\beta(d-\E_{R_t}d)}.
\]
For the last inequality, apply nonnegativity of relative entropy to the probability law with density $e^{\beta d}/\E_{R_t}e^{\beta d}$ relative to $R_t$.
All these quantities are finite because the losses are bounded.

To bound the centered exponential moment, set $a=G\|\theta-\theta_t\|$.
The common Lipschitz bound gives $|d|\leq a$ almost surely under $R_t$.
For $s\in\mathbb R$, define $K(s)=\log\E_{R_t}e^{s(d-\E_{R_t}d)}$.
Boundedness of $d$ permits differentiation under the expectation, and $K''(s)$ is its variance under the probability law proportional to $e^{sd}\,dR_t$.
Under any such law, $|d|\leq a$, so $K''(s)\leq a^2$.
Since $K(0)=K'(0)=0$, integration over $[0,\beta]$ gives $K(\beta)\leq\beta^2a^2/2$.
Consequently, taking the supremum over feasible $R$ yields
\[
 F(\theta)-f_t(\theta)
 \leq\frac{\beta G^2}{2}\|\theta-\theta_t\|^2
 \leq\frac{\bar\beta G^2}{2}\|\theta-\theta_t\|^2.
\]
Feasibility of $R_t$ gives the lower bound in~\eqref{eq:guidance-model}; setting $\theta=\theta_t$ gives equality of the two objectives there, completing the proof.
\end{proof}

\end{document}